\documentclass[final,12pt]{colt2022} 
 \usepackage{amsfonts}
\usepackage{txfonts}
\usepackage{algorithm}
\usepackage{algorithmic}

\usepackage{hyperref}

\def\1{\bm{1}}

\DeclareMathAlphabet{\mathsfit}{\encodingdefault}{\sfdefault}{m}{sl}
\SetMathAlphabet{\mathsfit}{bold}{\encodingdefault}{\sfdefault}{bx}{n}

\newcommand{\poly}{\mathrm{poly}}
\newcommand{\polylog}{\mathrm{polylog}}

\newcommand{\expect}{\mathbb{E}}

\newcommand{\indict}{\mathbb{I}}
\newcommand{\states}{\mathcal{S}}
\newcommand{\trans}{P}
\newcommand{\actions}{\mathcal{A}}
\newcommand{\mdp}{M}

\DeclareMathOperator*{\argmax}{arg\,max}

 \newtheorem{assumption}{Assumption}
 \title[
Horizon-Free Reinforcement Learning in Polynomial Time
]{
Horizon-Free Reinforcement Learning in Polynomial Time: \\the Power of Stationary Policies
	}
\usepackage{times}
\coltauthor{%
 \Name{Zihan Zhang} \Email{zihan-zh17@mails.tsinghua.edu.cn}\\
 \addr Tsinghua University
 \AND
 \Name{Xiangyang Ji} \Email{xyji@tsinghua.edu.cn}\\
 \addr Tsinghua University
 \AND
 \Name{Simon S. Du} \Email{ssdu@cs.washington.edu}\\
 \addr University of Washington
}

\begin{document}

\maketitle

\begin{abstract}
    This paper gives the first polynomial-time algorithm for tabular Markov Decision Processes (MDP) that enjoys a regret bound \emph{independent on the planning horizon}. 
Specifically, we consider tabular MDP with $S$ states, $A$ actions, a planning horizon $H$, total reward bounded by $1$, and the agent plays for $K$ episodes.
We design an algorithm that achieves an  $O\left(\mathrm{poly}(S,A,\log K)\sqrt{K}\right)$ regret in contrast to existing bounds which either has an additional $\mathrm{polylog}(H)$ dependency~\citep{zhang2020reinforcement} or has an exponential dependency on $S$~\citep{li2021settling}.
Our result relies on a sequence of new structural lemmas establishing the approximation power, stability, and concentration property of stationary policies, which can have applications in other problems related to Markov chains.

\end{abstract}

\section{Introduction}\label{sec:in}
Tabular Markov Decision Process (MDP) is one of the most fundamental models for reinforcement learning (RL).
The first algorithm that enjoys polynomial time and sample complexity guarantee at least dates back to 1990s~\citep{kearns1998near}.
However, despite of nearly two and half decades of research,  the sample complexity on this fundamental model remains open.
We study the canonical episodic time-homogeneous MDP with $S$ states, $A$ actions,  planning horizon $H$, and total reward upper bounded by $1$.\footnote{The upper bounded total reward is without the loss of generality. If the total reward is upper bounded by some $V_{\max} > 0$, then our regret will scale with $V_{\max}$.}

The main challenges that differentiate  RL and its special case, contextual bandits, are the unknown state-dependent transition and the long planning horizon.
In contextual bandits, the planning horizon is one and there is no state-dependent transition to learn.
Till today, it is unclear whether RL requires more samples than contextual bandits in the minimax sense.\footnote{For gap-dependent bounds, it has been shown that there is a gap between tabular MDP and contextual bandits~\citep{xu2021fine}.}
Specifically, the lower bound for the RL setting considered in this paper is  $\Omega\left(\sqrt{SAK}\right)$, which is the same for contextual bandits.

Due to these two challenges in RL, \citet{jiang2018open} conjectured a $\poly\left(H\right)$ lower bound.
Recent work refuted this conjecture by providing algorithms whose regret scales only \emph{logarithmically} with $H$~\citep{wang2020long,zhang2020reinforcement}.
Specifically, when the dependency on $H$ is allowed, state-of-the-art result shows one can have an $O\left(\sqrt{SAK}\cdot \polylog(S,A,H,K)\right)$ regret~\citep{zhang2020reinforcement}.
More recently, \citet{li2021settling} gave a surprising result showing the dependency on $H$ is not necessary.
However, their sample complexity has an exponential dependency on the number of states.
Therefore, one natural and conceptually important open question is:
\begin{center}
  \emph{Is there an algorithm whose regret (1) scales polynomially with $S$ and $A$, and (2) does not depend on $H$?}  
\end{center}

	\begin{table}[t]
		\centering
		\resizebox{1\columnwidth}{!}{%
		\renewcommand{\arraystretch}{2}
			\begin{tabular}{ |c|c|c|}
				\hline
				\textbf{Paper} & \textbf{Regret}  & \textbf{PAC Bound} \\
				\hline
					\cite{zhang2020reinforcement} & $O\left(\left(\sqrt{SAK}+S^2A\right)\polylog\left(S,A,K,  {H}\right)\right)$ &  $O\left(\left(\frac{SA}{\epsilon^2} + \frac{S^2A}{\epsilon}\right)\polylog\left(S,A,\frac{1}{\epsilon},  {H}\right)\right)$\\
				\hline 
				\cite{li2021settling} &   - &     $ \frac{\left(SA\right)^{O\left(S\right)}}{\epsilon^5}$ \\
				\hline 
			This work &    $ O\left(\left(\sqrt{S^9A^3K}\right)\polylog\left(S,A,K\right)\right)$&     $ O\left(\left(\frac{S^9A^3}{\epsilon^2} \right)\polylog\left(S,A,\frac{1}{\epsilon}\right)\right)$ \\
				\hline 
	Contextual bandits
					lower bound 
				& $\Omega\left(\sqrt{SAK}\right)$ & $\Omega\left(\frac{SA}{\epsilon^2}\right)$ \\
				\hline
			\end{tabular}
		}
		\caption{
			Comparisons of our result with prior arts.
			$S$: number of states, $A$: number of actions, $H$: planning horizon, $K$: number of episodes, $\epsilon$: target error.\label{tab:comparisons}
		}
	\end{table}

\subsection{Our Result}
Our paper answers this question positively.
\begin{theorem}\label{thm:main}
Suppose the reward at each step is non-negative and the total reward of each episode is bounded by $1$.
Given a failure probability $0 < \delta < 1$, then with probability at least $1-\delta$, the regret of our algorithm is bounded by  $O\left(\left(\sqrt{S^9A^3K}\right)\polylog\left(S,A,\log K, \log 1/\delta\right)\right)$ where $S$ is the number of episodes, $A$ is the number of actions, and $K$ is the total number of episodes.
\end{theorem}
Using a standard reduction~\citep{jin2018q}, this regret bound also implies a PAC bound of $O\left(\frac{S^9A^3 \polylog\left(S,A,1/\epsilon\right)}{\epsilon^2}\right)$ where $0 < \epsilon < 1$ is the target error.
In Table~\ref{tab:comparisons}, we compare our results with prior arts.

Several comments are in sequel. 
First, this is the first polynomial algorithm for tabular MDP whose regret has no dependence on $H$.
Therefore, we achieve an exponential improvement over \cite{li2021settling}.
Second, our dependency on $K$ (or $1/\epsilon$) is optimal up to logarithimic factors.
Third, the dependencies on $S$ and $A$ are not optimal. 
A fundamental open problem is to design an algorithm for tabular MDP whose regret bound exactly matches the lower bound of contextual bandits.

\subsection{Related Work}
We focus on papers that study episodic tabular MDP.
Other closely related settings include infinite-horizon discounted MDP, learning with a generative model, etc. 
We believe our techniques can be applied to those settings and obtain improvements, which we leave as future work.

\paragraph{Tabular MDP.}
There is a long list of sample complexity guarantees for tabular MDP~\citep{kearns2002near,brafman2002r,kakade2003sample,strehl2006pac,strehl2008analysis,kolter2009near,bartlett2009regal,jaksch2010near,szita2010model,lattimore2012pac,osband2013more,dann2015sample,azar2017minimax,dann2017unifying,osband2017posterior,agrawal2017optimistic,jin2018q,fruit2018near,talebi2018variance,dann2019policy,dong2019q,simchowitz2019non,russo2019worst,zhang2019regret,cai2019provably,zhang2020almost,yang2020q,pacchiano2020optimism,neu2020unifying,zhang2020reinforcement,li2021settling,menard2021ucb,xiong2021randomized,li2021breaking,pacchiano2020optimism}.
We note that some previous works consider the time-inhomogeneous MDP where the transition and the reward can vary on different time steps~\citep{jin2018q,zhang2020almost,li2021breaking,menard2021ucb}.
The regret for time-inhomogeneous setting will have  an $\sqrt{H}$ factor in the regret, which is necessary because the degree of freedom increases by $H$ compared with the time-homogeneous setting.
 	Transforming a regret bound tightly from the time-homogenous setting to that for time-inhomogeneous setting  is often straightforward (with an additional $\sqrt{H}$ factor), but not vice-versa, because one of the main difficulties to obtain sharp bounds in time-homogeneous MDP is how to exploit the property that the transition and reward do not vary on different time steps.

In this paper, we assume the total reward from all steps are upper bounded (cf. Assumption~\ref{asmp:total_bounded}).
Many prior work used the assumption that the reward from each step is upper bounded $1/H$, a.k.a., the uniformly bounded assumption.
The bounded total reward is strictly more general than the uniformly bounded assumption.
From a practical point of view, the bounded total reward assumption can model environments with spiky rewards, which are often considered to be a challenging problem~\citep{jiang2018open}.

\paragraph{Dependence on Horizon.}
Then main focus of this work is the dependence on the planning horizon $H$. This problem was throughly discussed in a COLT 2018 Open Problem ~\citep{jiang2018open} where it was conjectured that there would be a $\poly(H)$ regret lower bound. \citet{zanette2019tighter} partially refuted this conjecture by giving an algorithm whose regret only scales logarithmically with $H$ in the regime where $K = \poly(S,A,H)$. 
This conjecture was refuted by \citet{wang2020long} who built an $\epsilon$-net for the policy set and used it to develop a computationally inefficient algorithm which only requires $\poly\left(S,A,\log H, 1/\epsilon\right)$ to learn an $\epsilon$-optimal policy.
This result was substantially improved by \citet{zhang2020reinforcement} who gave a computationally efficient algorithm which enjoys an $O\left(\left(\sqrt{SAK}+S^2A\right)\polylog\left(S,A,K, {H}\right)\right)$ regret. 
Their technique was later adopted in several other setting to tighten the dependency on the horizon~\citep{zhang2020nearly,zhang2021variance,ren2021nearly,chen2021implicit,tarbouriech2021stochastic,chen2021improved}.
In Section~\ref{sec:tec}, We will discuss why their work has $\polylog H$ dependency and how we design new techniques to remove it.

\paragraph{Comparison with \cite{li2021settling}.}
The recent breakthrough by \citet{li2021settling} gave an $\frac{\left(SA\right)^{O\left(S\right)}}{\epsilon^5}$ sample complexity bound. Notably, this is the first result showing the sample complexity \emph{can be completely independent of $H$.} 
They have two key ideas: (1) a refined perturbation analysis in the generative model setting,\footnote{In the generative model setting, the agent can query any state-action pair. 
In this setting, they have can have polynomial sample complexity.} and (2) bounding the reaching probability based on the analysis on the entire trajectory instead of dynamic programming which is typically used in the literature.
An implication of the second idea is an approximation bound to non-stationary policies using stationary policies of discounted MDPs.
The approximation to non-stationary policies incurs the exponential dependency on $S$.
At a high level, they first uses all stationary policies (which is exponential in size) to collect enough samples to reduce the problem to the generative model setting, and uses the refined bound for to prove the final result.

We adopt their idea on the analyzing the entire trajectory. We give a refined analysis in bounding the reaching probability with an exponentially improved multiplicative constant (see discussions below Lemma~\ref{lemma:add1}).
Our algorithm framework is different from theirs: our algorithm follows a more conventional approach based on upper confidence bound (UCB), and thus we do not use their perturbation analysis for the generative model setting.
Nevertheless, the goal of the stage 1 of our algorithm, initial sample collection, is the same as their algorithm.
The main differences are (1) we divide each episode in stage into two phases, each with a different policy, whereas they used a single stationary policy; and (2) our sample collection is adaptive in that we change the exploration policy based on the collected samples whereas theirs is oblivious in the sense that they simply enumerate all stationary policies.

Stationary policy is a central object in both works. 
They established the approximation power of stationary policies to non-stationary policies.
We give an exponentially improved bound of approximation power dedicated to the visitation count (Lemma~\ref{lemma:al1}), and new results on the concentration (Lemma~\ref{lemma:stationary}) and stability (Lemma~\ref{lemma:approx}) of stationary policies.

Lastly, besides using discounted MDPs to establish approximation bounds, we also use discounted MDP to compute stationary policies to make our algorithm run in polynomial time.
All these differences are crucial in obtaining our polynomial-time horizon-free algorithm.
See Section~\ref{sec:tec} for more details.

\section{Preliminaries}
\label{sec:pre}
\paragraph{Notations.}
Throughout this paper, we use $[N]$ to denote the set $\{1, 2, \ldots, N\}$ for $N\in \mathbb{Z}_{+}$. We use $\textbf{1}_{s}$ to denote the one-hot vector whose only non-zero element is in the $s$-th coordinate.
For an event $\mathcal{E}$, we use $\indict[\mathcal{E}]$ to denote the indicator function, i.e., $\indict[\mathcal{E}] = 1$ if $\mathcal{E}$ holds and $\indict[\mathcal{E}] = 0$ otherwise.
For notational convenience, we set $\iota  = \ln(2/\delta)$ throughout the paper.  
For two $n$-dimensional vectors $x$ and $y$, we use $xy$  to denote  $x^{\top}y$, use $\mathbb{ V}(x ,y) = \sum_{i}x_i y_i^2 - (\sum_{i}x_iy_i)^2 $. In particular, when $x$ is a probability vector, i.e., $x_i\geq 0$ and $\sum_i x_i = 1$, $\mathbb{ V}(x ,y) = \sum_i x_{i}\left( y_i - (\sum_{i}x_iy_i)\right)^2 = \min_{\lambda\in \mathbb{R}}\sum_i x_{i}\left( y_i -\lambda\right)^2$.
We also use $x^2$ to denote the vector $[x_1^2,x_2^2,...,x_n^2]^{\top}$ for $x = [x_1,x_2,...,x_n]^{\top}$.
For two vectors $x,y$, $x \ge y$ denotes $x_i \ge y_i$ for all $i \in [n]$ and $x \le y$ denotes $x_i \le y_i$ for all $i \in [n]$.

\paragraph{Episodic Tabular MDP.}
We consider finite-horizon time-homogeneous Markov Decision Process (MDP) which can be described by a tuple $\mdp =\left(\states, \actions, \trans ,r, H, \mu_1\right)$.
$\states$ is the finite state space with cardinality $S$.
$\actions$ is the finite action space with cardinality $A$.
$\trans: \states \times \actions \rightarrow \Delta\left(\states\right)$ is the unknown transition operator which takes a state-action pair and returns a distribution over the states.
$r : \states \times \actions \rightarrow [0,1]$ is the reward function.
For simplicity,  we assume the reward function is known because the main difficulty is in estimating the transition function.
Prior work, e.g., \cite{jin2018q}, also made this assumption.
$H \in \mathbb{Z}_+$ is the planning horizon.
 $\mu_1 \in \Delta\left(\states\right)$ is the initial state distribution.

For notational convenience, we use $P_{s,a}$ and $P_{s,a,s'}$ to denote $P(\cdot|s,a)$ and $P(s'|s,a)$ respectively.

A policy $\pi$ chooses an action $a$ based on the current state $s \in \states$ and the time step $h \in [H]$. 
Note even though transition operator and the reward distribution do not depend on the level $h \in [H]$, the policy can choose different actions for the same state at different level $h$.
Formally, we define $\pi = \{\pi_h\}_{h = 1}^H$ where for each $h \in [H]$, $\pi_h : \states \to \actions$ maps a given state to an action.
The policy $\pi$ induces a trajectory $\{s_1,a_1,r_1,s_2,a_2,r_2,\ldots,s_{H},a_{H},r_{H} \}$,
where $s_1 \sim \mu_1$, $a_1 = \pi_1(s_1)$, $r_1 \sim R(s_1,a_1)$, $s_2 \sim \trans(\cdot|s_1,a_1)$, $a_2 = \pi_2(s_2)$, etc.
Our goal is to find a policy $\pi$ that maximizes the expected total
reward, i.e.,
$
\max_\pi \expect \left[\sum_{h=1}^{H} r_h \mid \pi\right],
$
where the expectation is over $\mu_1$,  $P$ and $R$.
We make the following normalization assumption about the reward.

 \begin{assumption}[Bounded Total Reward]\label{asmp:total_bounded}
 	The reward satisfies that $r_h\geq 0$ for all $h\in [H]$. Besides, for all policy $\pi$, $\sum_{h=1}^H r_h\leq 1$ almost surely.
 \end{assumption}
\paragraph{$Q$-function and $V$-function.}
 Given a policy $\pi$ and a level $h \in [H]$ the $Q$-function is defined as:
$Q_h^{\pi}: \states \times \actions \rightarrow R$, 
$Q_h^\pi(s,a) = \expect\left[\sum_{h' = h}^{H}r_{h'}\mid s_h =s, a_h = a, \pi\right].
$
Similarly, given a policy $\pi$, a level $h \in [H]$, the value function is defined as: $V_h^\pi : \states \rightarrow R$, 
$
V_h^\pi(s)=\expect\left[\sum_{h' = h}^{H}r_{h'}\mid s_h =s,
  \pi\right].
$
Then Bellman equation states the following identities for policy $\pi$ and $(s,a,h) \in \states \times \actions \times [H]$:
 $
Q_h^\pi(s,a) =r(s,a)+ P_{s,a}^{\top}V_{h+1}^\pi$ and $V_{h}^\pi(s) = Q_{h}^\pi(s,\pi_h(a)).
 $
 Throughout the paper, we let $V_{H+1}(s) = 0$ and $Q_{H+1}(s,a) = 0$ for simplicity.
We use $Q^*_h$ and $V^*_h$ to denote the optimal $Q$-function and $V$-function, which satisfies for any state-action pair $(s,a) \in \states \times \actions$, $Q^*_h(s,a) = \max_{\pi}Q^{\pi}_h(s,a)$ and $V^*_h(s) =\max_{\pi}V^{\pi}_h(s)$.
%

 \paragraph{Regret and PAC Bound.}
The agent interacts with the environment for $K$ episodes, and it chooses a policy $\pi^k$ at the $k$-th episode.
The total regret is defined as 
\[
\mathrm{Regret}(K) =  \sum_{k=1}^K V_1^*(s_1^k) - V_1^{\pi^k}(s_1^k).
\]
PAC-RL sample complexity is another measure which counts the total number of episodes to find an $\epsilon$-optimal policy $\pi$, i.e., $\expect_{s_1 \sim \mu_1}\left[V_1^*(s_1) - V^\pi(s_1)\right] \le \epsilon.$

A regret bound can be transformed into a PAC bound~\citep{jin2018q}.
Specifically, if an algorithm achieves a $CK^{1-\alpha}$ regret for some $\alpha \in (0,1)$ and some $C$ independent of $K$, by randomly selecting from policy $\pi^k$ used in $K$ episodes, $\pi$ will satisfy $\expect_{s_1 \sim \mu_1}\left[V_1^*(s_1) - V^\pi(s_1)\right] = O\left(CK^{-\alpha}\right)$.
Setting $CK^{-\alpha} = \epsilon$, we can obtain a PAC-RL bound, which we also use.

%
%
%
%

\paragraph{Additional Notations.} 
Let $\Pi$ denote the set of all policies and $\Pi_{\mathrm{sta}}$ denote the set of all stationary policies (a policy $\pi$ is stationary if $\pi_1 = \pi_2 \cdots = \pi_H$).
We use $\mathbb{E}_{\pi,p}[\cdot ]$ and $P_{\pi,p}[\cdot]$ to denote the expectation and probability following a  policy $\pi$ under a transition $p$. We let $W_{d}^{\pi}(r',p,\mu):=\mathbb{E}_{\pi,p}[\sum_{h=1}^d r'(s_h,a_h) | s_1 \sim \mu]$ be the value function for a reward function $r'$ and a transition model $p$ with horizon length $d$ and initial distribution $\mu$. With a slight abuse of notation, we also define  $W_{d}^{\pi}(r',p,\mu):=\mathbb{E}_{\pi,p}[\sum_{h=1}^d r'(s_h,a_h) | (s_1,a_1) \sim \mu]$ for  $\mu$ as a distribution over state-action space.   
We also use $\textbf{1}_{s}$ and $\textbf{1}_{s,a}$ to denote the reward function $r'$ such that $r'(s',a')=\mathbb{I}[s'=s]$ and $r'(s',a') = \mathbb{I}[(s',a')=(s,a)]$, respectively.
Sometimes we also abuse the notation to use $\textbf{1}_s$ and $\textbf{1}_{s,a}$ to denote a distribution with $\mathrm{Pr}(s)=1$ and $\mathrm{Pr}(s,a)=1$ respectively, With these notations,  $W_d^\pi(\textbf{1}_{s,a},p,\textbf{1}_{s})$ denotes the expected number of visits to $(s,a)$ under the policy $\pi$ in a transition $p$ with a planing horizon $d$ and the agent starts from the fixed state $s$. This is a crucial function which we will use in our proof.

\section{Technical Overview}\label{sec:tec}
Our algorithm follows the conventional UCB-based framework. Different from existing work, to avoid the dependency on $H$, we also design a new stage to \emph{explicit explore} each state-action pair.
To illustrate why we introduce this new stage, along with our other technical ideas, we discuss each major source that incurs a $\log H$ dependency in \cite{zhang2020reinforcement}, and then describe our techniques to remove the $\log H$.

\paragraph{Source 1: Higher Order Expansion.}
One key idea in \cite{zhang2020reinforcement} in bounding the regret is to use a recursive structure to relate the estimated variance to the higher moments.
They expanded for $O\left(\log H\right)$ times, which incurred a $\log H$ in their regret bound.

This source is relatively simple to remove.
We use an observation in \cite{chen2021implicit} (cf. Lemma~\ref{lemma:varb}) that bounds the variance of the product two random variables, together with several probability bounds (cf. Lemma~\ref{lemma:ratiocon},~\ref{lemma:con4},~\ref{freedman},~\ref{lemma:self-norm},~\eqref{eqn:use_lemma13}, and~\eqref{eq:lemma13_2}), to avoid the use of recursion.
We note that this analysis technique directly applies to the algorithm in \cite{zhang2020reinforcement}, and therefore can simplify their proof.

\paragraph{Source 2: Counting in Pigeonhole.}
A standard proof step in nearly all UCB-based algorithms is bounding $\sum_{k=1}^K\sum_{h=1}^H\frac{1}{\max\{N^k(s_h^k,a_h^k),1\}}$ where $N^k(s,a)$ is the number of visits to state-action pair $(s,a)$ before the $k$-th episode, and $(s_h^k,a_h^k)$ is the state-action pair of the $h$-step in the $k$-th episode.
By the pigeonhole principle, we can bound \[
\sum_{k=1}^K\sum_{h=1}^H\frac{1}{\max\{N^k(s_h^k,a_h^k),1\}} = O\left( \sum_{s \in \states, a \in \actions} \sum_{k=1}^K \min\left(\log \left(\frac{ N^{k+1}(s,a)}{ N^k(s,a)}\right),1\right)\right)
\]
Since in total we have $KH$ state-action pair visitations, we can have a straightforward bound  $\sum_{k=1}^K\min\left(\frac{\log N^{k+1}(s,a)}{\log N^k(s,a)},1\right)= O\left( \log (KH)\right)$, which is used in all prior work.
However, in the regime that $H$ is large, e.g., $H=2^{K}$, this naive bound gives linear regret.
%
This source is much more difficult to remove.
To remote it, we start with the followingg observation.

\paragraph{The Benefit of Initial Samples.}
Our first observation is that if we have enough initial samples, i.e., $N^1(s,a)$ is above a certain threshold, then we can avoid using the naive $\log(KH)$ bound.
Formally, define $U(s,a):=\max_{\pi}W_H^{\pi}(\textbf{1}_{s,a},P,\mu_{1})$ be the maximum expected visitation count of $(s,a)$ in one episode.
Then by Markov's inequality, with probability  $1-\delta$, the total count of $(s,a)$ in $K$ episodes satisfies $N^K(s,a)\leq KU(s,a)/\delta$. Conditioned on this event, if we make $N^1(s,a)$ comparable to $U(s,a)$,  for example, $N^1(s,a)\geq U(s,a)/\exp(\mathrm{poly}(S))$, then we  have that
\begin{align}
\sum_{k=1}^K \min\left(\log \left(\frac{ N^{k+1}(s,a)}{ N^k(s,a)}\right),1\right) = O\left(\log\left(\frac{N^K(s,a)}{N^1(s,a)}\right)  \right) = O(\mathrm{poly}(S)\log(K/\delta)),
\end{align}
which is independent of $H$.
Now, the problem reduces to collect enough initial samples to make $N^1(s,a)$ comparable to $U(s,a)$.
This problem of collecting initial samples is highly non-trivial and we devote the following subsection to describe our technical ideas.

\subsection{Collecting Initial Samples}\label{sec:col}




Now we focus on collecting the initial samples. For a fixed $(s^*,a^*)$, our goal is to collect samples of $(s^*,a^*)$.  For the ease of discussion, here we assume that $N^1(s,a)\geq \frac{U(s,a)}{ \exp(\mathrm{poly}(S))}$ for all $(s,a)$ except for $(s^*,a^*)$. 

For this task, we divide one epoch into two phases. In the first phase, we aim to reach the target state $s^*$. In the second phase, we aim to collect as many samples of  $(s^*,a^*)$ as possible with the agent starting from $s^*$.
We note that this two-phase procedure uses \emph{two} stationary policies, in contrast to \cite{li2021settling} who used a single stationary policy to collect samples. 
This difference is one of the key ingredients in obtaining the polynomial bound.

\subsubsection{Phase 1: Reaching the Target State $s^*$}
We first decide the length for each phase, which relies on the following lemma. The formal statement requires more notations and defer to appendix.
\begin{lemma}(Informal) \label{lemma:add1}
Let $\mathcal{O}\subset\mathcal{S}\times\mathcal{A}$. Let $X^{\pi}_{d'}(\mathcal{O},p,\mu_1)$ denote the probability of reaching $\mathcal{O}$ in $d'$ steps following $\pi$ under the  transition $p$. We have the following bound: for any $\tilde{d} \in \mathbb{Z}_+$, $\max_{\pi}X^{\pi}_{(S+2)\tilde{d}}(\mathcal{O},p,\mu_1)\leq S^2\max_{\pi}X_{(S+1)\tilde{d}}^{\pi}(\mathcal{O},p,\mu_{1})$.
\end{lemma}

Lemma~\ref{lemma:add1} establishes a bound of two reaching probabilities induced by the same policy, transition, initial distribution but slightly different planning horizons ($(S+1)\tilde{d}$ v.s. $S\tilde{d}$).
We believe this lemma will have applications in other problems.
To prove this lemma, we count the probability of all possible trajectories under two horizons and construct a mapping between the trajectories.

We note this lemma is similar in spirit to Lemma 4.2, 4.3 and 4.4 of \cite{li2021breaking}, which bound the reaching probability of a longer horizon Markov chain by that of a shorter horizon Markov chain and a multiplicative factor.
The main difference is that we are using the reaching probability of a horizon-$S\tilde{d}$ Markov chain to approximate that of a horizon-$(S+1)\tilde{d}$ Markov chain whereas they used the reaching probability of a horizon-$S\tilde{d}$ Markov chain to approximate that of a horizon-$4S\tilde{d}$ Markov chain.
This difference ($S\tilde{d}$ and $(S+1)\tilde{d}$ versus $S\tilde{d}$ and $4S\tilde{d}$) results in an exponential improvement in the multiplicative factor: from $S^{4S}$ in \cite{li2021settling} to $S^2$ in Lemma~\ref{lemma:add1}.
The proof for both results are based on counting arguments although the details are substantially different.

To use this lemma, we view $H = (S+1)\tilde{d}$, and from the bound, it is natural to use the first $\frac{HS}{S+1}$ steps in one episode to reach $s^*$ and use the remaining steps to collect $(s^*,a^*)$. \footnote{With loss of generality, we assume $\frac{H}{S+1}$ is an integer and $H \gg S$ because we are interested in the regime $H$ is large. }

To find a policy that reaches $s^*$, we 
can use $\textbf{1}_{s^*}$ as the reward function, and perform a regret-minimization algorithm.
Since we have assumed $N^1(s',a')\geq U(s',a')/\exp(\mathrm{poly}(S))$ for any $(s,a)\neq (s^*,a^*)$, running the regret minimization problem for $K_1$ episodes gives a \emph{first-order regret bound} of ${O}(\mathrm{poly}(SA)\mathrm{polylog}(K_1)\sqrt{K_1v^*\iota})$, where $v^*$ is the optimal value, i.e., the maximal probability of reaching $s$.  Now we have two cases: (1) $v^*\geq \frac{f(SA)\mathrm{polylog}(K_1)}{K_1}$ for some polynomial $f$, then the cumulative reward, i.e., the number of times of reaching $s^*$ is large enough; (2) $v^*<\frac{f(SA)\mathrm{polylog}(K_1)}{K_1} $, then $(s^*,a^*)$ could be ignored with most $O\left(\frac{Kf(SA)\mathrm{polylog}(K_1)}{K_1} \right)=O(\mathrm{poly}(SA)\mathrm{polylog}(K)\sqrt{K\iota})$ regret by choosing $K_1 = O(\sqrt{K\iota})$.
We note that the actual algorithm simultaneously explores all under-explored states by setting reward to be $1$ for all under-explored states. See Algorithm~\ref{alg:main} for details.

\subsubsection{Phase 2: Collecting Samples of $(s^*,a^*)$ Starting from $s^*$}
In this phase, we start from the state $s^*$, and we would like to collect as many samples of $(s^*,a^*)$ as possible.
Inspired by recent work~\citep{li2021settling}, we also consider using \emph{stationary} policies to collect samples.
Below we will give three key lemmas (Lemma~\ref{lemma:al1},~\ref{lemma:stationary},~\ref{lemma:approx}) to characterize the approximation power, the concentration property, and the stability of stationary policies.
We believe these lemmas will have applications in other problems.

Recall that $W^{\pi}_{d}(\textbf{1}_{s,a},P,\textbf{1}_s)$ denotes the expected number of visits to $(s,a)$, starting from $s$, following $\pi$ in a transition $P$ with planning horizon $d$.
The following lemma establishes that the power of stationary policies in collecting samples is not much worse than that of  non-stationary policies.
We prove this lemma using the discounted approximation by noting that there is an optimal stationary policy for the discounted planning.

This lemma can be compared to Corollary 4.7 of \cite{li2021settling}. Their lemma is more general because it applies general reward and arbitrary initial distribution but ours only applies to reward of the form $\textbf{1}_{s,a}$ with the starting distribution being $\textbf{1}_s$.
On the other hand, our multiplicative factor is exponentially smaller than theirs (roughly speaking, $O(S)$ vs. $S^{O(S)}$) and this improvement is crucial in obtaining our polynomial-time algorithm.

\begin{lemma}\label{lemma:al1}[Approximation Power of Stationary Policies] Let $k$ and $d$ be positive integers. We have that for any $(s,a) \in \states \times \actions$,  
\[\max_{\pi \in \Pi}W^{\pi}_{kd}(\textbf{1}_{s,a},P,\textbf{1}_s)\leq 6k\max_{\pi\in \Pi_{\mathrm{sta}}}W^{\pi}_{d}(\textbf{1}_{s,a},P,\textbf{1}_s).\]
\end{lemma}


The following lemma is a concentration bound for stationary policy, which shows the number of samples we  collect empirically is close to the expectation.
The proof is by regarding the recurrent time as i.i.d. random variables and constructing a stopping time.

\begin{lemma}\label{lemma:stationary}[Concentration Property of Stationary Policies]
 For  any $(s,a) \in \states \times \actions$ and $\pi \in \Pi_{\mathrm{sta}}$ such that $\pi(s)=a$, we have that 
$
\mathrm{Pr}\left[ N \geq \frac{1}{4}W^{\pi}_{d}(P,\textbf{1}_{s,a},\textbf{1}_{s})  \right]\geq \frac{1}{2}$
for any horizon $d$,
where $N$ is the visit count of $(s,a)$ following $\pi$ under $P$ in $d$ steps with the initial distribution as $\textbf{1}_{s}$.
\end{lemma}

Therefore, if we successfully find a stationary policy that  maximizes
$W^{\pi}_{H/(S+2)}(\textbf{1}_{s,a},P,\textbf{1}_s)$, then by Lemma~\ref{lemma:stationary}, we can collect $\Omega(\max_{\pi\in \Pi_{\mathrm{sta}}}W^{\pi}_{H/(S+2)}(\textbf{1}_{s,a},P,\textbf{1}_s))$ samples, which, by Lemma~\ref{lemma:al1}, is larger than $\Omega\left(\frac{1}{S+1}\max_{\pi\in \Pi}W^{\pi}_{H}(\textbf{1}_{s,a},P,\textbf{1}_s)\right)=\Omega(U(s,a)/S)$.



To learn a stationary policy with large enough visitation count to $(s,a)$, we consider to learn a reference model $P^{\mathrm{ref}}$ close to $P$ to help plan.
The next lemma can be viewed as a \emph{multiplicative performance difference lemma}.
This lemma establishes the stability of stationary policies in the relative sense.
Importantly, the multiplciative factor is completely independent of $H$.
The proof is based on a local perturbation analysis.
In each time, we perturb one $(s,a)$ and aggregate the perturbation error in the end.

\begin{lemma}[Multiplicative Performance Difference Lemma for Stationary Policies]\label{lemma:approx}
Let the initial distribution $\mu_1$ be fixed.
	For two transition model $P'$ and $P''$ such that $e^{-\epsilon}P''_{s,a,s'}\leq P'_{s,a,s'}\leq e^{\epsilon}P''_{s,a,s'}$, we have that 
	\begin{align}
	   e^{-4S\epsilon}    W^{\pi}_{d}(P',r,\mu_1) \leq W^{\pi}_{d}(P'',r,\mu_1)\leq  e^{4S\epsilon}W^{\pi}_{d}(P',r,\mu_1)
	\end{align}
	for any stationary policy $\pi$, horizon $d\geq 1$ and non-negative reward $r$. 
\end{lemma}
 By viewing $W_{d}^{\pi}(P',r,\mu_1)$ as a function of $P'$, Lemma~\ref{lemma:approx} shows that $\log(W(P',r,\mu_1))$ is $O(S)$-Lipschtiz continuous in $\log(P')$. It is crucial that the Lipschtiz constant is independent of $H$, which allows us to choose $\epsilon=O(1/S)$ in Lemma~\ref{lemma:approx}.

Now our goal is to find a transition model $P^{\mathrm{ref}}$ such that $e^{-\epsilon}P^{\mathrm{ref}}_{s,a,s'}\leq P_{s,a,s'}\leq e^{\epsilon}P^{\mathrm{ref}}_{s,a,s'}$ for any $(s,a,s')$. By concentration inequalities for the multinomial distribution, to learn such a transition model, we need to sample from $(s,a)$ until $(s,a,s')$ is visited more than $\frac{C\iota}{\epsilon^2}$ times for each $(s,a,s')$.

\paragraph{Clipped MDP and Explicit Exploration.}
The main difficulty is to deal with the case that $P_{s,a,s'}$ is small. For example if $P_{s,a,s'}\leq \frac{1}{KH}$, we can hardly collect enough samples of $(s,a,s')$. 
To address this problem,  we simply ignore such $(s,a,s')$ tuples since the probability of visiting them is also very small. 
More precisely, we maintain a set $(\mathcal{K})^{C}$ for such tuples and construct a \emph{clipped MDP}, where we redirect all $(s,a,s') \in (\mathcal{K})^{C}$ tuples to the virtual ending state, denoted as $z$. 
We note that we will update $(\mathcal{K})^{C}$ throughout the training process because after we collect new samples, we can assert that certain $P_{s,a,s'}$ is large and we can move $(s,a,s')$ out of $(\mathcal{K})^{C}$.



In addition, we conduct \emph{explicit exploration}. Roughly speaking, for each $(s,a)$, if we have the chance to visit $(s,a)$, then we design a policy to visit $(s,a)$ as much as possible to judge whether $z$ can be reached by $(s,a)$.
More precisely, in the beginning of the  first sub-phase, we test that if there exists  some $(s,a)$ such that  the maximal possible expected count of $(s,a)$ under the clipped transition model exceeds the current visitation count of $(s,a)$ by a $\frac{1}{\mathrm{poly}(SA)}$ ratio. 
Then we have two cases: (1)
 There exists such a $(s,a)$. In this case we conduct exploration to collect samples of $(s,a)$. Our target is the maximal possible expected count under the clipped transition model, which is smaller than that under the original transition model. Therefore, this task is easier  and could be completed by naive planning. (2) Otherwise, we can ensure that the probability of visiting $z$ is bounded by an universal constant. Then we can plan to visit the target pair $(s^*,a^*)$ by ignoring $z$.

\paragraph{Using Discounted MDP for Efficient Planning with Stationary Policies.}
Our final major technical idea is for the computational purpose. Given a finite-horizon MDP, finding the best stationary policy that maximizes the reward may not be computationally efficient.
Recall that all we need is a multiplicative approximation.
Therefore, we use discounted MDP to approximate the finite-horizon MDP.
See Lemma~\ref{lemma:eff} for the guarantees.
We note that the idea of using discounted MDP was also used in \cite{li2021settling}, although they did not use it for computational reasons.

\section{Main Algorithm}\label{sec:ma}
\begin{algorithm}[!t]
\caption{Main Algorithm \label{alg:main} }
\begin{algorithmic}[1]
\STATE{\textbf{Input:} state space $S$, action space $A$, reward $r$, horizon $H$, confidence parameter $\delta$;}
\STATE{\textbf{Initialization}:  $N(s,a,s')\leftarrow 0, \forall s,a,s'$, $\bar{N}(s,a) \leftarrow 0, \forall (s,a)$, $d\leftarrow \frac{(S+1)H}{S+2}$; $\mathcal{O}^1\leftarrow \mathcal{S}\times\mathcal{A}$; $\mathcal{P}^1\leftarrow (\Delta^{S})^{SA}$ ; $K_1\leftarrow C_1\sqrt{S^9A^3K\iota}$, $n_1\leftarrow C_2S^7A^3\iota$; $d'=H-d$;  $m(s,a)\leftarrow 0$; $N_0\leftarrow 256S^2\log(1/\delta)$;
}
\STATE{{//} \emph{Stage 1: Collecting initial samples} }
\FOR{$k=1,2,\ldots, K_1$}
\STATE{$\mathcal{P}^k \leftarrow \mathtt{Confidence Set}(\{N(s,a,s')\}_{s,a,s'})$; \label{line:O_planning}}
\STATE{$(\pi^k,\tilde{P}^k)\leftarrow \max_{\pi,p\in\mathcal{P}^k}X^{\pi}_{d}(\mathcal{O}^k,p,\mu_1)$ 
}
\FOR{$h=1,2,\ldots,d$}
\STATE{Observes $s_h^k$, takes action $\pi^k_h(s_h^k)$, receives $r_h^k$ and transits to $s_{h+1}^k$;}
\STATE{$N(s_h^k,a_h^k,s_{h+1}^k)\leftarrow N(s_h^k,a_h^k,s_{h+1}^k)+1$;}
\IF{$\exists a, (s_{h+1}^k,a)\in \mathcal{O}^k$}
\STATE{$(s_1^*,a_1^*)\leftarrow (s_{h+1}^k,a)$;}
\STATE{$\{N(s,a,s')\}_{s,a,s'}\leftarrow \{n(s,a,s')\}_{s,a,s'}$;}
\STATE{$\mathcal{K}^k \leftarrow \{(s,a,s'): n(s,a,s')\geq N_0\}$, $\mathcal{K}^k(s,a)\leftarrow \{ s': (s,a,s')\in \mathcal{K}^k\}$; \label{line:known_set}}
\STATE{$n(s,a)\leftarrow \max\{\sum_{s':(s,a,s')\in \mathcal{K}^k} n(s,a,s'),1\}~\forall (s,a)$; \label{line:pref_1}}
\STATE{$P^{\mathrm{ref}}_{s,a,s'}\leftarrow \frac{n(s,a,s')}{n(s,a)}$, $P^{\mathrm{ref}}_{s,a,z}\leftarrow 0$, $\forall (s,a,s')\in \mathcal{K}^k$; \label{line:pref_2} }
\STATE{$P^{\mathrm{ref}}_{s,a,s'}\leftarrow 0$, $P^{\mathrm{ref}}_{s,a,z}=1$, $\forall (s,a,s')$ such that $ \mathcal{K}^{k}(s,a)= \emptyset$; \label{line:pref_3}}
\STATE{$(\mathrm{Trigger}, \{n(s,a,s')\}_{s,a,s'} )\leftarrow Algorithm~\ref{alg:model1}$ with inputs\\  \quad \quad \quad\quad \quad \quad \quad \quad \quad \quad \quad \quad   $( (s_1^*,a_1^*),P^{\mathrm{ref}}, \{n(s,a,s')\}_{(s,a,s')},\mathcal{K}^k  ,d')$;}

\IF{$\mathrm{Trigger}=\mathrm{FALSE}$}
\STATE{$\{n(s,a,s')\}_{(s,a,s')}\leftarrow  Algorithm~\ref{alg:model2}$ with inputs \\ \quad \quad \quad\quad \quad \quad \quad \quad \quad \quad \quad \quad  $((s_1^*,a^*_1), P^{\mathrm{ref}}, \{n(s,a,s')\}_{s,a,s'} ,d')$}
\STATE{$m(s^*_1,a^*_1)\leftarrow m(s^*_1,a^*_1)+1$;}
\IF{$m(s^*_1,a^*_1)\geq 400\log(1/\delta)$}
\STATE{$\mathcal{O}^{k+1}\leftarrow \mathcal{O}^{k}/(s_1^*,a_1^*)$;}
\ENDIF
\ENDIF

\STATE{$\{n(s,a,s')\}_{s,a,s'}\leftarrow \{N(s,a,s')\}_{s,a,s'}$;}

\STATE{\textbf{break};}
\ENDIF
\ENDFOR

\STATE{ If there are remaining steps, run a random policy and update $\{N(s,a,s')\}_{s,a,s'}$;}


\ENDFOR

\STATE{{//} \emph{Stage 2: Regret Minimization with Initial Samples} }

\STATE{Run Algorithm~\ref{alg:rmis} with inputs $\{ N_{s,a,s'}\}_{s,a,s'}$.}

\end{algorithmic}
\end{algorithm}

Now we present our main algorithm. 
There are two stages in Algorithm~\ref{alg:main}. In the first stage, for each episode, we let the agent explore in its first $\frac{HS}{S+1}$ steps to reach new state-action pairs, and collect the initial samples using the remaining $\frac{H}{S+1}$ steps.
The number of this stage is bounded by $O(\mathrm{poly}(S,A,\log(K))\sqrt{K})$, and incurring at most $O(\mathrm{poly}(S,A,\log K)\sqrt{K})$ regret. In the second stage, we play optimistic value iteration to learning the MDP with initial samples. 
Below we give two important notions used in Algorithm~\ref{alg:main}.

In stage 1, the algorithm maintains an omitted set denoted as $\mathcal{O}^k \subset \states \times \actions$ for the $k$-th episode.
If a state-action pair $(s,a)$ is \emph{not} in $\mathcal{O}^k$, we know we have have collected enough samples for $(s,a)$.
We note that in this end we may not have $\mathcal{O}^k = \emptyset$ because there can be states that are hard to reach using any policy and we cam simply ignore them.
To explore, we plan \emph{optimistically} according to a confidence set of the transition matrix, constructed by the collected samples.

\paragraph{Confidence set.}
Given $\{N(s,a,s')\}_{s,a,s'}$, we define $N(s,a)=\max\{\sum_{s'}N(s,a,s'),1\}$, and  $\mathcal{P}=\mathtt{ConfidenceSet}(\{N(s,a,s')\}_{s,a,s'})$ by setting $\mathcal{P}=\otimes_{h,s,a}\mathcal{P}_{h,s,a}$ where 
\begin{align}
    \mathcal{P}_{h,s,a}= \left\{ p\in \Delta^{S}:  |p_{s'}-\frac{N(s,a,s')}{N(s,a)}| \leq \sqrt{4\frac{N(s,a,s')\iota}{N^2(s,a)}}+\frac{5\iota}{N(s,a)}\right\}.\nonumber
\end{align}
We note that $\mathcal{P}_{h,s,a}$ does not depend on $h$. We add $h$ in the subscript only for the writing purpose when we use $\mathcal{P}_{h,s,a}$.

For each $k\in [K]$, we use $N^{k}(s,a,s')$ to denote the value of $N(s,a,s')$ before the $k$-th episode. Define $N^k(s,a)=\max\{ \sum_{s'}N^k(s,a,s'),1\}$ and $\hat{P}^k_{s,a,s'}=\frac{N^k(s,a,s')}{N^k(s,a)}$. Define $\mathcal{G}$ be the event where
\begin{align}
    |P_{s,a,s'}-\hat{P}^k_{s,a,s'}|\leq \min \left\{   \sqrt{2\frac{P_{s,a,s'}\iota}{N^k(s,a)}}+\frac{\iota}{3N(s,a)}, \sqrt{4\frac{\hat{P}^k_{s,a,s'}\iota}{N^k(s,a)}}+\frac{5\iota}{N(s,a)}\right\}
\end{align}
holds for any $k,s,a,s'$.
By Bennets's inequality and Bernstein's inequality, we have that $\mathbb{P}[\mathcal{G}]\geq 1-2S^2AK\delta$. In the analysis below, we assume $\mathcal{G}$ holds.

Now we describe Stage 1.
We divide each episode into two phases. The first phase has length $d= \frac{SH}{H+1}$ and the second phase $d' = H-d$.
In Line~\ref{line:O_planning}, we plan and try to arrive at a state-action pairs that we have not collected enough samples, a.k.a., maximize the reaching probability of $\mathcal{O}^k$.
In the episode $k$ and during phase 1, $h=1,\ldots,d$, whenever we meet a state $s_{h+1}^k$ such that there exists $a$ that $(s_{h+1}^k,a) \in \mathcal{O}^k$, we stop phase 1 because we have reached one state-action pair that we have not collected enough samples of.

\begin{algorithm}[!t]
\caption{$\mathtt{Explicit~Exploration}$ \label{alg:model1}
}
\begin{algorithmic}[1]
\STATE{\textbf{Input}: starting state-action pair $(s_1,a_1)$, reference model $P^{\mathrm{ref}}$, sample count $\{n(s,a,s')\}_{s,a,s'}$, known set $\mathcal{K}$, horizons $d'$, $d_2 = d'/(20S\log(S))$, $d_1= d'-d_2$}
\STATE{\textbf{Initialization}: discounted factor $\gamma = 1-1/d_2$, $N_0 \leftarrow 256S^2\log(1/\delta)$;}
\STATE{Trigger = FALSE;}
\FOR{$(s,a)\in \mathcal{S}\times \mathcal{A}$}
\IF{$\exists s' \in \mathcal{S}$ such that $(s,a,s')\notin \mathcal{K}$ }
\STATE{$\pi^k_1\leftarrow \arg\max_{\pi\in \Pi_{\mathrm{sta}},\pi_1(s_{1})=a_{1} } X^{\pi}_{\gamma}(\{s\},P^{\mathrm{ref}},\textbf{1}_{s_1})$;}
\STATE{$u^k(s)\leftarrow X^{\pi_1^k}_{\gamma}(\{s\},P^{\mathrm{ref}},\textbf{1}_{s_1}) $;}
\STATE{$\pi^k_2\leftarrow \arg\max_{\pi\in \Pi_{\mathrm{sta}} } W^{\pi}_{\gamma}(\textbf{1}_{s,a},P^{\mathrm{ref}},\textbf{1}_{s})$;}
\STATE{$v^k(s,a)\leftarrow W^{\pi_2^k}_{\gamma}( \textbf{1}_{s,a}, P^{\mathrm{ref}}, \textbf{1}_{s})  $;}
\IF{$u^k(s)\geq \frac{1}{1200S}$ and $n(s,a)\leq 810SAN_0u^k(s)v^k(s,a)$ \label{line:check}}
\STATE{$\mathrm{Trigger}\leftarrow \mathrm{TRUE}$;}\label{line:exp3}
\STATE{Run $\pi_1^k$ for $d_1$ steps. Stop if $(s,a)$ is reached or some \emph{unknown} state-action-state tuple is visited;}
\IF{$(s,a)$ is reached}
\STATE{Play $\pi_2^k$ for $d_2$ steps, then play random policies till the end;}
\ELSE
\STATE{Play random policies till the end;}
\ENDIF
\STATE{Let $\{s_i,a_i,s_{i+1}\}_{i=1}^{d'}$ denote the data collected in the length $d'$-trajectory;}
\FOR{$i=1,2,\ldots,d'$}
\STATE{$n(s_i,a_i,s_{i+1})\leftarrow n(s_1,a_i,s_{i+1})+1$;}
\ENDFOR
\STATE{\textbf{Break};}

\ENDIF
\ENDIF
\STATE{\textbf{Break};}

\ENDFOR
\STATE{\textbf{Return}: $\mathrm{Trigger}$, $\{n(s,a,s') \}_{(s,a,s')}$;}
\end{algorithmic}
\end{algorithm}
In phase 2, we denote $(s_1^*,a_1^*) = (s_{h+1}^k,a)$ and try to collect as many $(s_1^*,a_1^*)$ as possible.
Instead of using the confidence set of the transition matrix to do planning optimistically, we split state-action-state triples as known set ($\mathcal{K}^k$) and unknown set $\left(\mathcal{K}^k\right)^c$ (cf. Line~\ref{line:known_set}), and then we compute a clipped reference transition model to plan  defined below (also see Line~\ref{line:pref_1} - Line~\ref{line:pref_3} in Algorithm~\ref{alg:main}).

\paragraph{Clipped Reference Transition Model.}
Given  $\mathcal{K}^C\subset \mathcal{S}\times\mathcal{A}\times \mathcal{S}$ and a transition model $p$, we define $p':=\mathrm{clip}(p,\mathcal{K}^C)$ be the transition model such that $p'_{s,a,s'}=p_{s,a,s'},\forall (s,a,s')\notin \mathcal{K}^C$, $p'_{s,a,s'}=0,\forall (s,a,s')\in \mathcal{K}^C$, $p'_{s,a,z}=\sum_{s':(s,a,s')\in \mathcal{K}^C}p_{s,a,s'}$, $p'_{z,a}=\textbf{1}_{z'},\forall a$ and $p'_{z',a}=\textbf{1}_{z'}, \forall a$. In words, we redirect the $(s,a,s')$ triples in $\mathcal{K}^C$ to a virtual state $z$, which transits to a virtual absorbed state $z'$ with probability $1$. 
The reason why we need an additional $z'$ instead of just $z$ is make the total reward bounded by $1$.
As a result, we have the following identity by definition:
\begin{align}
    X_{\tilde{d}}^{\pi}(\mathcal{K}^C,p,\mu_1)=W_{\tilde{d}}^{\pi}(\textbf{1}_{z},\mathrm{Clip}(p,\mathcal{K}^C),\mu_1),~~\forall \tilde{d} \in \mathbb{Z}_+. \label{eq:vtr}
\end{align}
In a similar way, we define $\mathrm{clip}(p,\mathcal{K}^C)$ for $\mathcal{K}^C\subset \mathcal{S}\times \mathcal{A}$ and $\mathcal{K}^C\subset \mathcal{S}$. 

In our context, $P^{\mathrm{ref}}=\mathrm{Clip}(\hat{P},\left(\mathcal{K}^k\right)^C)$ where $\hat{P}$ is the empirical model.
This clipping operation is crucial to enable us to  use Lemma~\ref{lemma:con1}.

\paragraph{Explicit Exploration.} Given the a starting state-action pair ($s_1^*,a_1^*$), a reference model and the known set, we apply Algorithm~\ref{alg:model1}.
In Algorithm~\ref{alg:model1}, we first try to explicit explore the unknown set in order to make our reference model estimation more accurate. To do so, for every state-action-state triple $(s,a,s')$ not in the known set,  we compute two \emph{stationary policies}, $\pi_1$ and $\pi_2$ where $\pi_1$ tries to reach $s$ from $(s_1^*,a_1^*)$ and $\pi_2$ tries to collect as many $(s,a)$ as possible \emph{starting from $s$}.
The stationary policies are computed by using a discounted MDP to approximate a finite-horizon MDP.
The purpose is that we can compute the stationary policies in polynomial time.

Besides the policies, we also obtain estimates $u^k(s)$ and $v^k(s,a)$ on how many samples we can expect to collect. 
In Line~\ref{line:check} of Algorithm~\ref{alg:model1}, we check whether our estimation is large, and we have not collect enough samples. If this is the case, we execute $\pi_1$ and $\pi_2$.
Otherwise, we either have collected enough samples or $s$ is hard to reach.

We iterate all state-action pairs, and if for all pairs we have either collected enough samples or identified that this triple is hard to reach (which we an ignore), we are confident the reference model is good enough for our purpose (Trigger = FALSE in this case).
In this case, we use the reference model to collect as many $(s_1^*,a_1^*)$ as possible (cf. Algorithm~\ref{alg:model2}).
Again, for computational efficiency purpose, we use a stationary policy computed from a discounted MDP that approximates the finite-horizon MDP.

\begin{algorithm}[!t]
\caption{$\mathtt{Sample~Collection~with~a~Reference~Model}$ \label{alg:model2} }
\begin{algorithmic}
\STATE{\textbf{Input}: initial state-action pair $(s_1,a_1)$, reference model $P^{\mathrm{ref}}$ , visit count $\{n(s,a,s')\}_{s,a,s'}$, horizon $d'$.}
\STATE{\textbf{Initialization}: discounted factor $\gamma = 1-1/d_2$ where $d_2 = d'/(20S\log(S))$.}
\STATE{$\pi\leftarrow \arg\max_{\pi\in \Pi_{\mathrm{sta}}}W^{\pi}_{\gamma}(\textbf{1}_{s_1,a_1}, P^{\mathrm{ref}},\textbf{1}_{s_1})$}
\STATE{Run $\pi$ and collect $d'$ samples $\{s_i,a_i,s_{i+1}\}_{i=1}^{d'}$;}
\FOR{$i=1,2,\ldots,d'$}
\STATE{$n(s_i,a_i,s_{i+1})\leftarrow n(s_i,a_i,s_{i+1})+1$;}
\ENDFOR
\STATE{\textbf{Return:} $ \{n(s,a,s')\}_{(s,a,s')}$;}
\end{algorithmic}
\end{algorithm}

\begin{algorithm}[!t]
\caption{$\mathtt{Regret\, Minimization \,with \,Initial\, Samples}$ ($\mathtt{RMIS}$) \label{alg:rmis}}
\begin{algorithmic}[1]
\STATE{\textbf{Input}: $\{ N(s,a,s')\}_{(s,a,s')}$;}
\STATE{$N(s,a)\leftarrow \max\{\sum_{s'}N(s,a,s'),1\}$;}
\FOR{$k=1,2,\ldots,K-K_1$}
\STATE{$\mathcal{P}^k \leftarrow \mathtt{ConfidenceSet}(\{N(s,a,s')\}_{s,a,s'})$}
\STATE{$V^k_{H+1}(s)\leftarrow 0,\forall s$;}
\STATE{$V^k_h(s)\leftarrow \min\{ \max_{a,p\in\mathcal{P}^k_{s,a}} (r(s,a)+pV_{h+1}^k),1\}, \forall (h,s)\in [H]\times\mathcal{S}$;}
\STATE{$\pi_h^k(s)\leftarrow \arg\max_{a}\max_{p\in\mathcal{P}^k_{s,a}}(r(s,a)+pV_{h+1}^k), \forall (h,s)\in [H]\times\mathcal{S}$;}
\FOR{$h=1,2,\ldots,H$}
\STATE{Observes $s_h^k$, takes action $\pi^k_h(s_h^k)$, receives reward $r_h^k$ and transits to $s_{h+1}^k$;}
\STATE{$N(s_h^k,a_h^k,s_{h+1}^k)\leftarrow N(s_h^k,a_h^k,s_{h+1}^k)+1$;}
\ENDFOR
\ENDFOR
\end{algorithmic}
\end{algorithm}

\paragraph{Stage 2: Regret Minimization with Initial Samples}
After collecting initial samples, in Stage 2, we perform standard optimistic model-based planning using dynamic programming.
See Algorithm~\ref{alg:rmis} for details.


\section{Regret Analysis}\label{sec:ra}

Setting $K_1 = C_1\sqrt{S^9A^3K\iota}$ with some constant $C_1$, we have two key lemmas below.

\begin{lemma}\label{lemma:stage1}
    Let $\mathcal{O}^{K_1+1}$ be defined in Algorithm~\ref{alg:main}. With probability $1-10SAK\delta$, we have that
    \begin{align}
        \max_{\pi}\mathbb{P}_{\pi}[\exists h\in [H], (s_h,a_h)\in \mathcal{O}^{K_1}]\leq  O\left(\frac{S^9A^3\iota+S^3A\iota^2}{K_1}\right).
    \end{align}
\end{lemma}
Lemma~\ref{lemma:stage1} states that, we can collect enough initial samples for most state-action pairs. And the probability of visiting the remaining state-action pairs (those in the omitted set) is comparably small.
See Appendix~\ref{app:stage1} for details.

\begin{lemma}\label{lemma:sr}
With  probability $1-10S^3A^2K\delta$, it holds that
\begin{align}
N^{\tilde{k}+1}(\tilde{s},\tilde{a})\geq 2\max_{\pi\in \Pi_{\mathrm{sta}} }W_{d_2}^{\pi}(\textbf{1}_{\tilde{s},\tilde{a}},P,\textbf{1}_{\tilde{s}})\log(1/\delta) \nonumber
\end{align}
 for any $(\tilde{s},\tilde{a})\in \mathcal{O}^{\tilde{k}+1}/\mathcal{O}^{\tilde{k}} $ and any $1\leq \tilde{k}\leq K_1$.
\end{lemma}

Lemma~\ref{lemma:sr} states that the initial number of state-action pairs not in $\mathcal{O}^{K_1+1}$ is large enough. The proof of Lemma~\ref{lemma:sr} is given in Appendix~\ref{app:pfsr}

\begin{lemma}\label{lemma:stage2}
    With probability $1-10SAK\delta$, the regret in the second stage is bounded by 
    $$O\left(\mathrm{polylog}(SAK)\left(\sqrt{S^2AK\iota^2}+\frac{S^9A^3K\iota+S^3A\iota^2}{K_1}\right)\right).$$
\end{lemma}

Lemma~\ref{lemma:stage2} is based on classical regret analysis for finite horizon-MDP.
The second term comes from the error from stage 1.
In the proof we also need refined analysis to remove the extra $\log(H)$ factors. 
See Appendix~\ref{app:stage2} for details.

By Lemma~\ref{lemma:stage2}, and noting that the regret in the first stage is bounded by $K_1$,  we have that the total regret is upper bounded by $O(\mathrm{\polylog}(SAK)\sqrt{(S^9A^3+S^3A\iota)K\iota})$, and we finish  the proof.

\section{Conclusion}\label{sec:conclusion}
In this paper, we presented the first polynomial-time algorithm for tabular MDP whose regret is completely independent of the horizon.
Our result crucially relies a series of structural lemmas of stationary policies, which we believe will be useful in other setting. A fundamental open problem is whether we can design an algorithm with $O\left(\sqrt{SAK}\right)$ regret. A positive answer would have a surprising implication that tabular MDP is as easy as contextual bandits in the minimax sense.
We also believe designing an algorithm with $O\left(\left(\sqrt{SAK}+\poly (S,A)\right)\polylog(S,A,K)\right)$ regret is a meaningful intermediate result.

\section*{Acknowledgements}
The authors thank Ruosong Wang for insightful discussions. Zihan Zhang and Xiangyang Ji are supported by Beijing Municipal Science and Technology Commission grant Z201100005820005. Simon S. Du acknowledges funding from NSF Award’s IIS-2110170 and DMS- 2134106.

\bibliography{ref}

\begin{thebibliography}{46}
\providecommand{\natexlab}[1]{#1}
\providecommand{\url}[1]{\texttt{#1}}
\expandafter\ifx\csname urlstyle\endcsname\relax
  \providecommand{\doi}[1]{doi: #1}\else
  \providecommand{\doi}{doi: \begingroup \urlstyle{rm}\Url}\fi

\bibitem[Agrawal and Jia(2017)]{agrawal2017optimistic}
Shipra Agrawal and Randy Jia.
\newblock Optimistic posterior sampling for reinforcement learning: worst-case
  regret bounds.
\newblock In \emph{Advances in Neural Information Processing Systems}, pages
  1184--1194, 2017.

\bibitem[Azar et~al.(2017)Azar, Osband, and Munos]{azar2017minimax}
Mohammad~Gheshlaghi Azar, Ian Osband, and R{\'e}mi Munos.
\newblock Minimax regret bounds for reinforcement learning.
\newblock In \emph{Proceedings of the 34th International Conference on Machine
  Learning}, pages 263--272, 2017.

\bibitem[Bartlett and Tewari(2009)]{bartlett2009regal}
Peter~L Bartlett and Ambuj Tewari.
\newblock Regal: a regularization based algorithm for reinforcement learning in
  weakly communicating mdps.
\newblock In \emph{Proceedings of the 25th Conference on Uncertainty in
  Artificial Intelligence (UAI 2009))}, 2009.

\bibitem[Brafman and Tennenholtz(2003)]{brafman2002r}
Ronen~I. Brafman and Moshe Tennenholtz.
\newblock R-max - a general polynomial time algorithm for near-optimal
  reinforcement learning.
\newblock \emph{J. Mach. Learn. Res.}, 3\penalty0 (Oct):\penalty0 213--231,
  March 2003.
\newblock ISSN 1532-4435.

\bibitem[Cai et~al.(2019)Cai, Yang, Jin, and Wang]{cai2019provably}
Qi~Cai, Zhuoran Yang, Chi Jin, and Zhaoran Wang.
\newblock Provably efficient exploration in policy optimization.
\newblock \emph{arXiv preprint arXiv:1912.05830}, 2019.

\bibitem[Chen et~al.(2021{\natexlab{a}})Chen, Jafarnia-Jahromi, Jain, and
  Luo]{chen2021implicit}
Liyu Chen, Mehdi Jafarnia-Jahromi, Rahul Jain, and Haipeng Luo.
\newblock Implicit finite-horizon approximation and efficient optimal
  algorithms for stochastic shortest path.
\newblock \emph{Advances in Neural Information Processing Systems}, 34,
  2021{\natexlab{a}}.

\bibitem[Chen et~al.(2021{\natexlab{b}})Chen, Jain, and Luo]{chen2021improved}
Liyu Chen, Rahul Jain, and Haipeng Luo.
\newblock Improved no-regret algorithms for stochastic shortest path with
  linear mdp.
\newblock \emph{arXiv preprint arXiv:2112.09859}, 2021{\natexlab{b}}.

\bibitem[Dann and Brunskill(2015)]{dann2015sample}
Christoph Dann and Emma Brunskill.
\newblock Sample complexity of episodic fixed-horizon reinforcement learning.
\newblock In \emph{Advances in Neural Information Processing Systems}, pages
  2818--2826, 2015.

\bibitem[Dann et~al.(2017)Dann, Lattimore, and Brunskill]{dann2017unifying}
Christoph Dann, Tor Lattimore, and Emma Brunskill.
\newblock Unifying {PAC} and regret: Uniform {PAC} bounds for episodic
  reinforcement learning.
\newblock In \emph{Proceedings of the 31st International Conference on Neural
  Information Processing Systems}, NIPS’17, page 5717–5727, Red Hook, NY,
  USA, 2017. Curran Associates Inc.
\newblock ISBN 9781510860964.

\bibitem[Dann et~al.(2019)Dann, Li, Wei, and Brunskill]{dann2019policy}
Christoph Dann, Lihong Li, Wei Wei, and Emma Brunskill.
\newblock Policy certificates: Towards accountable reinforcement learning.
\newblock In \emph{Proceedings of the 36th International Conference on Machine
  Learning}, pages 1507--1516, 2019.

\bibitem[Dong et~al.(2019)Dong, Wang, Chen, and Wang]{dong2019q}
Kefan Dong, Yuanhao Wang, Xiaoyu Chen, and Liwei Wang.
\newblock Q-learning with ucb exploration is sample efficient for
  infinite-horizon mdp.
\newblock \emph{arXiv preprint arXiv:1901.09311}, 2019.

\bibitem[Freedman(1975)]{freedman1975tail}
David~A Freedman.
\newblock On tail probabilities for martingales.
\newblock \emph{the Annals of Probability}, 3\penalty0 (1):\penalty0 100--118,
  1975.

\bibitem[Fruit et~al.(2018)Fruit, Pirotta, and Lazaric]{fruit2018near}
Ronan Fruit, Matteo Pirotta, and Alessandro Lazaric.
\newblock Near optimal exploration-exploitation in non-communicating markov
  decision processes.
\newblock In \emph{Advances in Neural Information Processing Systems}, pages
  2994--3004, 2018.

\bibitem[Jaksch et~al.(2010)Jaksch, Ortner, and Auer]{jaksch2010near}
Thomas Jaksch, Ronald Ortner, and Peter Auer.
\newblock Near-optimal regret bounds for reinforcement learning.
\newblock \emph{Journal of Machine Learning Research}, 11\penalty0
  (Apr):\penalty0 1563--1600, 2010.

\bibitem[Jiang and Agarwal(2018)]{jiang2018open}
Nan Jiang and Alekh Agarwal.
\newblock Open problem: The dependence of sample complexity lower bounds on
  planning horizon.
\newblock In \emph{Conference On Learning Theory}, pages 3395--3398, 2018.

\bibitem[Jin et~al.(2018)Jin, Allen-Zhu, Bubeck, and Jordan]{jin2018q}
Chi Jin, Zeyuan Allen-Zhu, Sebastien Bubeck, and Michael~I Jordan.
\newblock Is {Q}-learning provably efficient?
\newblock In \emph{Advances in Neural Information Processing Systems}, pages
  4863--4873, 2018.

\bibitem[Kakade(2003)]{kakade2003sample}
Sham~M Kakade.
\newblock \emph{On the sample complexity of reinforcement learning}.
\newblock PhD thesis, University of London London, England, 2003.

\bibitem[Kearns and Singh(2002)]{kearns2002near}
Michael Kearns and Satinder Singh.
\newblock Near-optimal reinforcement learning in polynomial time.
\newblock \emph{Machine learning}, 49\penalty0 (2-3):\penalty0 209--232, 2002.

\bibitem[Kearns and Singh(1998)]{kearns1998near}
Michael~J Kearns and Satinder~P Singh.
\newblock Near-optimal reinforcement learning in polynominal time.
\newblock In \emph{Proceedings of the Fifteenth International Conference on
  Machine Learning}, page 260–268, 1998.

\bibitem[Kolter and Ng(2009)]{kolter2009near}
J~Zico Kolter and Andrew~Y Ng.
\newblock Near-bayesian exploration in polynomial time.
\newblock In \emph{Proceedings of the 26th annual international conference on
  machine learning}, pages 513--520, 2009.

\bibitem[Lattimore and Hutter(2012)]{lattimore2012pac}
Tor Lattimore and Marcus Hutter.
\newblock Pac bounds for discounted mdps.
\newblock In \emph{International Conference on Algorithmic Learning Theory},
  pages 320--334. Springer, 2012.

\bibitem[Li et~al.(2021{\natexlab{a}})Li, Shi, Chen, Gu, and
  Chi]{li2021breaking}
Gen Li, Laixi Shi, Yuxin Chen, Yuantao Gu, and Yuejie Chi.
\newblock Breaking the sample complexity barrier to regret-optimal model-free
  reinforcement learning.
\newblock \emph{Advances in Neural Information Processing Systems}, 34,
  2021{\natexlab{a}}.

\bibitem[Li et~al.(2021{\natexlab{b}})Li, Wang, and Yang]{li2021settling}
Yuanzhi Li, Ruosong Wang, and Lin~F Yang.
\newblock Settling the horizon-dependence of sample complexity in reinforcement
  learning.
\newblock In \emph{IEEE Symposium on Foundations of Computer Science},
  2021{\natexlab{b}}.

\bibitem[M{\'e}nard et~al.(2021)M{\'e}nard, Domingues, Shang, and
  Valko]{menard2021ucb}
Pierre M{\'e}nard, Omar~Darwiche Domingues, Xuedong Shang, and Michal Valko.
\newblock Ucb momentum q-learning: Correcting the bias without forgetting.
\newblock In \emph{International Conference on Machine Learning}, pages
  7609--7618. PMLR, 2021.

\bibitem[Neu and Pike-Burke(2020)]{neu2020unifying}
Gergely Neu and Ciara Pike-Burke.
\newblock A unifying view of optimism in episodic reinforcement learning.
\newblock \emph{arXiv preprint arXiv:2007.01891}, 2020.

\bibitem[Osband and Van~Roy(2017)]{osband2017posterior}
Ian Osband and Benjamin Van~Roy.
\newblock Why is posterior sampling better than optimism for reinforcement
  learning?
\newblock In \emph{Proceedings of the 34th International Conference on Machine
  Learning-Volume 70}. JMLR. org, 2017.

\bibitem[Osband et~al.(2013)Osband, Russo, and Van~Roy]{osband2013more}
Ian Osband, Daniel Russo, and Benjamin Van~Roy.
\newblock (more) efficient reinforcement learning via posterior sampling.
\newblock In \emph{Advances in Neural Information Processing Systems}, pages
  3003--3011, 2013.

\bibitem[Pacchiano et~al.(2020)Pacchiano, Ball, Parker-Holder, Choromanski, and
  Roberts]{pacchiano2020optimism}
Aldo Pacchiano, Philip Ball, Jack Parker-Holder, Krzysztof Choromanski, and
  Stephen Roberts.
\newblock On optimism in model-based reinforcement learning.
\newblock \emph{arXiv preprint arXiv:2006.11911}, 2020.

\bibitem[Ren et~al.(2021)Ren, Li, Dai, Du, and Sanghavi]{ren2021nearly}
Tongzheng Ren, Jialian Li, Bo~Dai, Simon~S Du, and Sujay Sanghavi.
\newblock Nearly horizon-free offline reinforcement learning.
\newblock \emph{Advances in neural information processing systems}, 34, 2021.

\bibitem[Russo(2019)]{russo2019worst}
Daniel Russo.
\newblock Worst-case regret bounds for exploration via randomized value
  functions.
\newblock In \emph{Advances in Neural Information Processing Systems}, pages
  14433--14443, 2019.

\bibitem[Simchowitz and Jamieson(2019)]{simchowitz2019non}
Max Simchowitz and Kevin~G Jamieson.
\newblock Non-asymptotic gap-dependent regret bounds for tabular {MDPs}.
\newblock In \emph{Advances in Neural Information Processing Systems}, pages
  1153--1162, 2019.

\bibitem[Strehl and Littman(2008)]{strehl2008analysis}
Alexander~L Strehl and Michael~L Littman.
\newblock An analysis of model-based interval estimation for markov decision
  processes.
\newblock \emph{Journal of Computer and System Sciences}, 74\penalty0
  (8):\penalty0 1309--1331, 2008.

\bibitem[Strehl et~al.(2006)Strehl, Li, Wiewiora, Langford, and
  Littman]{strehl2006pac}
Alexander~L Strehl, Lihong Li, Eric Wiewiora, John Langford, and Michael~L
  Littman.
\newblock {PAC model-free reinforcement learning}.
\newblock In \emph{Proceedings of the 23rd international conference on Machine
  learning}, pages 881--888. ACM, 2006.

\bibitem[Szita and Szepesv{\'a}ri(2010)]{szita2010model}
Istv{\'a}n Szita and Csaba Szepesv{\'a}ri.
\newblock Model-based reinforcement learning with nearly tight exploration
  complexity bounds.
\newblock In \emph{ICML}, 2010.

\bibitem[Talebi and Maillard(2018)]{talebi2018variance}
Mohammad~Sadegh Talebi and Odalric-Ambrym Maillard.
\newblock Variance-aware regret bounds for undiscounted reinforcement learning
  in mdps.
\newblock \emph{arXiv preprint arXiv:1803.01626}, 2018.

\bibitem[Tarbouriech et~al.(2021)Tarbouriech, Zhou, Du, Pirotta, Valko, and
  Lazaric]{tarbouriech2021stochastic}
Jean Tarbouriech, Runlong Zhou, Simon~S Du, Matteo Pirotta, Michal Valko, and
  Alessandro Lazaric.
\newblock Stochastic shortest path: Minimax, parameter-free and towards
  horizon-free regret.
\newblock \emph{Advances in Neural Information Processing Systems}, 34, 2021.

\bibitem[Wang et~al.(2020)Wang, Du, Yang, and Kakade]{wang2020long}
Ruosong Wang, Simon~S Du, Lin~F Yang, and Sham~M Kakade.
\newblock Is long horizon reinforcement learning more difficult than short
  horizon reinforcement learning?
\newblock In \emph{Advances in Neural Information Processing Systems}, 2020.

\bibitem[Xiong et~al.(2021)Xiong, Shen, and Du]{xiong2021randomized}
Zhihan Xiong, Ruoqi Shen, and Simon~S Du.
\newblock Randomized exploration is near-optimal for tabular mdp.
\newblock \emph{arXiv preprint arXiv:2102.09703}, 2021.

\bibitem[Xu et~al.(2021)Xu, Ma, and Du]{xu2021fine}
Haike Xu, Tengyu Ma, and Simon Du.
\newblock Fine-grained gap-dependent bounds for tabular mdps via adaptive
  multi-step bootstrap.
\newblock In \emph{Conference on Learning Theory}, pages 4438--4472. PMLR,
  2021.

\bibitem[Yang et~al.(2020)Yang, Yang, and Du]{yang2020q}
Kunhe Yang, Lin~F Yang, and Simon~S Du.
\newblock {$Q$}-learning with logarithmic regret.
\newblock \emph{arXiv preprint arXiv:2006.09118}, 2020.

\bibitem[Zanette and Brunskill(2019)]{zanette2019tighter}
Andrea Zanette and Emma Brunskill.
\newblock Tighter problem-dependent regret bounds in reinforcement learning
  without domain knowledge using value function bounds.
\newblock In \emph{International Conference on Machine Learning}, pages
  7304--7312, 2019.

\bibitem[Zhang and Ji(2019)]{zhang2019regret}
Zihan Zhang and Xiangyang Ji.
\newblock Regret minimization for reinforcement learning by evaluating the
  optimal bias function.
\newblock In \emph{Advances in Neural Information Processing Systems}, pages
  2823--2832, 2019.

\bibitem[Zhang et~al.(2020)Zhang, Zhou, and Ji]{zhang2020almost}
Zihan Zhang, Yuan Zhou, and Xiangyang Ji.
\newblock Almost optimal model-free reinforcement learning via
  reference-advantage decomposition.
\newblock In \emph{Advances in Neural Information Processing Systems}, 2020.

\bibitem[Zhang et~al.(2021{\natexlab{a}})Zhang, Du, and Ji]{zhang2020nearly}
Zihan Zhang, Simon~S Du, and Xiangyang Ji.
\newblock Nearly minimax optimal reward-free reinforcement learning.
\newblock \emph{International Conference on Machine Learning},
  2021{\natexlab{a}}.

\bibitem[Zhang et~al.(2021{\natexlab{b}})Zhang, Ji, and
  Du]{zhang2020reinforcement}
Zihan Zhang, Xiangyang Ji, and Simon Du.
\newblock Is reinforcement learning more difficult than bandits? a near-optimal
  algorithm escaping the curse of horizon.
\newblock In \emph{Conference on Learning Theory}, pages 4528--4531. PMLR,
  2021{\natexlab{b}}.

\bibitem[Zhang et~al.(2021{\natexlab{c}})Zhang, Yang, Ji, and
  Du]{zhang2021variance}
Zihan Zhang, Jiaqi Yang, Xiangyang Ji, and Simon~S Du.
\newblock Variance-aware confidence set: Variance-dependent bound for linear
  bandits and horizon-free bound for linear mixture mdp.
\newblock In \emph{Advances in Neural Information Processing Systems},
  2021{\natexlab{c}}.

\end{thebibliography}

\newpage

\appendix

\section{Technical Lemmas}

\begin{lemma}[Lemma 30 in \cite{chen2021implicit}]
\label{lemma:varb}
$\mathrm{Var}(XY)\leq 2\mathbb{E}^2[Y]\mathrm{Var}(X)+2\sup X^2\mathrm{Var}(Y)$.
\end{lemma}

\begin{lemma}\label{lemma:ratiocon}
Let $X_1,X_2,\ldots$ be a sequence of random variables taking value in $[0,l]$. Define $\mathcal{F}_k =\sigma(X_1,X_2,\ldots,X_{k-1})$ and $Y_k = \mathbb{E}[X_k|\mathcal{F}_k]$ for $k\geq 1$. For any $\delta>0$, we have that
\begin{align}
& \mathbb{P}\left[ \exists n, \sum_{k=1}^n X_k \geq  3\sum_{k=1}^n Y_k+ l\log(1/\delta)\right]\leq \delta\nonumber
\\  & \mathbb{P}\left[  \exists n,  \sum_{k=1}^n Y_k \geq 3\sum_{k=1}^n X_k + l\log(1/\delta)  \right]    \leq \delta .\nonumber 
\end{align}
\end{lemma}

\begin{proof} Let $t\in [0,1/l]$ be fixed.
Consider to bound $Z_k:=\mathbb{E}[\exp(t\sum_{k'=1}^k(X_{k'}-3Y_{k'})  )]$. By definition, we have that
\begin{align}
    \mathbb{E}[Z_k|\mathcal{F}_k]  & =\exp(t\sum_{k'=1}^{k-1}(X_{k'}-3Y_{k'})) \mathbb{E}\left[\exp(X_{k}-3Y_{k})\mid \mathcal{F}_k\right] \nonumber 
    \\ & \leq \exp(t\sum_{k'=1}^{k-1}(X_{k'}-3Y_{k'}))\exp(-3Y_{k})\cdot \mathbb{E}[1+tX_k+2t^2X^2_{k} \mid \mathcal{F}_k]\nonumber 
    \\ & \leq \exp(t\sum_{k'=1}^{k-1}(X_{k'}-3Y_{k'}))\exp(-3Y_{k})\cdot \mathbb{E}[1+3tX_k \mid \mathcal{F}_k]\nonumber 
    \\ & = \exp(t\sum_{k'=1}^{k-1}(X_{k'}-3Y_{k'}))\exp(-3Y_{k})\cdot (1+3tY_k)\nonumber 
    \\ & \leq \exp(t\sum_{k'=1}^{k-1}(X_{k'}-3Y_{k'}))\nonumber
    \\ & =Z_{k-1},\nonumber 
\end{align}
where the second line is by the fact that $e^x\leq 1+x+2x^2$ for $x\in [0,1]$. 
Define $Z_0 = 1$
Then $\{Z_{k}\}_{k\geq 0}$ is a super-martingale with respect to $\{\mathcal{F}_{k}\}_{k\geq 1}$. Let $\tau$ be the smallest $n$ such that $\sum_{k=1}^n X_k - 3\sum_{k=1}^n Y_k >l\log(1/\delta)$.
It is easy to verify that $Z_{\min\{\tau,n\}}\leq \exp(tl\log(1/\delta)+tl)<\infty$. Choose $t=1/l$.
By the optimal stopping time theorem, we have that for any $N$:
\begin{align}
 & \mathbb{P}\left[ \exists n\leq N, \sum_{k=1}^n X_k \geq  3\sum_{k=1}^n Y_k + l\log(1/\delta)\right]\nonumber 
 \\ & =\mathbb{P}\left[\tau \leq N\right]\nonumber 
 \\ & \leq \mathbb{P}\left[ Z_{\min\{\tau,N\}}\geq \exp(tl\log(1/\delta))  \right]\nonumber 
 \\ & \leq \frac{\mathbb{E}[Z_{\min\{\tau,N\}}]}{\exp(tl\log(1/\delta))}\nonumber
 \\ & \leq \frac{Z_0}{\exp(tl\log(1/\delta))}\nonumber
 \\ & \leq \delta. \nonumber
\end{align}

Letting $N\to \infty$, we have that
\begin{align}
& \mathbb{P}\left[ \exists n, \sum_{k=1}^n X_k \geq  3\sum_{k=1}^n Y_k+ l\log(1/\delta)\right]\leq \delta.\nonumber
\end{align}
Considering $W_k = \mathbb{E}[\exp(t\sum_{k'=1}^k (Y_k/3-X_k))]$, using similar arguments  and choosing $t=1/(3l)$, we have that
\begin{align}
 & \mathbb{P}\left[  \exists n,  \sum_{k=1}^n Y_k \geq 3\sum_{k=1}^n X_k + l\log(1/\delta)  \right]    \leq \delta .\nonumber 
\end{align}
The proof is completed.
\end{proof}

\begin{lemma}\label{lemma:con4}
Let $v\in [0,1]^{S}$ be fixed vectir. Let $X_1,X_{2},\ldots, X_{t}$ be i.i.d. multinomial distribution with parameter $p\in \Delta^{S}$. Define $\widehat{\mathrm{Var}} = \frac{1}{t} \left( \sum_{i=1}^t v^2_{X_i} - \frac{1}{t}\left(\sum_{i=1}^t v_{X_i}\right)^2 \right)$ be the empirical variance of $\{v_{X_i}\}_{i=1}^t$ and $\mathrm{Var}=\mathbb{V}(p,v)$ be the true variance of $v_{X_1}$. With probability $1-3\delta$, it holds that
\begin{align}
\frac{1}{3}\mathrm{Var} - \frac{7\log(1/\delta)}{3t} \leq \widehat{\mathrm{Var}}\leq 3\mathrm{Var}+ \frac{\log(1/\delta)}{t}.\label{eq:ee1}
\end{align}
\end{lemma}
\begin{proof}
Without loss of generality, we assume $\mathbb{E}[v_{X_1}]=0$. By Lemma~\ref{lemma:ratiocon}, with probability $1-\delta$, it holds that
\begin{align}
    t\widehat{\mathrm{Var}} \leq \sum_{i=1}^t v^2_{X_i}\leq 3t\mathbb{E}[v^2_{X_i}] + \log(1/\delta).
\end{align}
Dividing both side with $t$, we prove the right hand side of \eqref{eq:ee1}. For the other side, by Hoeffding's inequality, we have that $|\sum_{i=1}^t v_{X_i}|\leq \sqrt{2t\log(1/\delta)}$ holds with probability $1-\delta$. Using Lemma~\ref{lemma:ratiocon} again, with probability $1-\delta$ it holds that
\begin{align}
    \sum_{i=1}^tv^2_{X_i} \geq \frac{t}{3}\mathbb{E}[v^2_{X_i}] - \frac{1}{3}\log(1/\delta).\nonumber
\end{align}
As a result, with probability $1-2\delta$ it holds  that (by the definition of $\widehat{Var} $)
\begin{align}
    t\widehat{Var} & \geq \frac{t}{3}\mathbb{E}[v^2_{X_i}] - \frac{1}{3}\log(1/\delta) - \frac{1}{t}\cdot  2t\log(1/\delta) \nonumber 
    \\ & = \frac{t}{3}\mathbb{E}[v^2_{X_i}]-\frac{7}{3}\log(1/\delta).\nonumber
\end{align}
The proof is completed by dividing both side by $t$.

\end{proof}

\begin{lemma}[Freedman's Inequality, Theorem 1.6 of \cite{freedman1975tail}]\label{freedman}
	Let $(M_{n})_{n\geq 0}$ be a  martingale such that $M_{0}=0$ and $|M_{n}-M_{n-1}|\leq c$. Let $\mathrm{Var}_{n}=\sum_{k=1}^{n}\mathbb{E}[(M_{k}-M_{k-1})^{2}|\mathcal{F}_{k-1}]$ for $n\geq 0$,
	where $\mathcal{F}_{k}=\sigma(M_0,M_{1},M_{2},\dots,M_{k})$. Then, for any positive $x$ and for any positive $y$,
	\begin{equation}\label{Bernstein2}
	\mathbb{P}\left[ \exists n:  M_{n}\geq x ~\text{and}~\mathrm{Var}_{n}\leq y \right]  \leq \exp\left(-\frac{x^{2}}{2(y+cx)} \right).
	\end{equation}
\end{lemma}

\begin{lemma}[Bennet's Inequality]\label{lemma:bennet}
Let $Z,Z_1,...,Z_n$  be i.i.d. random variables with values in $[0,1]$ and let $\delta>0$. Define $\mathbb{V}Z = \mathbb{E}\left[(Z-\mathbb{E}Z)^2 \right]$. Then we have
\begin{align}
\mathbb{P}\left[ \left|\mathbb{E}\left[Z\right]-\frac{1}{n}\sum_{i=1}^n Z_i  \right| > \sqrt{\frac{  2\mathbb{V}Z \log(2/\delta)}{n}} +\frac{\log(2/\delta)}{n} \right]]\leq \delta.\nonumber
\end{align}
\end{lemma}

\begin{lemma}\label{lemma:self-norm}
	Let $(M_{n})_{n\geq 0}$ be a  martingale  such that $M_0 = 0$  and $|M_{n}-M_{n-1}|\leq c$ for some $c>0$ and any $n\geq 1$. Let $\mathrm{Var}_{n}=\sum_{k=1}^{n}\mathbb{E}[(M_{k}-M_{k-1})^{2}|\mathcal{F}_{k-1}]$ for $n\geq 0$,
	where $\mathcal{F}_{k}=\sigma(M_{1},M_{2},...,M_{k})$. Let $C>0$ be a constant. For  any $p>0$, we have that
	\begin{equation}\label{self-bernstein}
\mathbb{P}\left[\exists n,i, |M_{n}|\geq 10\cdot 2^ici\log(1/p), |\mathrm{Var}_n|\leq 2^{2i}ic^2\log(1/p)   \right] \leq p.
	\end{equation}

\end{lemma}
\begin{proof}
By using Lemma~\ref{freedman} with  $x = 10\cdot 2^i ci\log(1/p)$, $y = 2^{2i}ic^2\log(1/p)$, we have that 
\begin{align}
\mathbb{P}\left[\exists n, |M_{n}|\geq 10\cdot 2^ici\log(1/p), |\mathrm{Var}_n|\leq 2^{2i}ic^2\log(1/p)   \right] \leq \frac{p}{e^i}\leq \frac{p}{2^i}.
\end{align}
Taking sum over $i$ we finish the proof.
\end{proof}

\begin{lemma}\label{lemma:conn}
Let $X_1,X_2,\ldots$ be i.i.d. random variables with multinomial distribution. For each $y\in \mathrm{supp}(X_1)$, we define $N_{t}(y)=\sum_{i=1}^t \mathbb{I}[X_i=y]$. Let $t(y)$ be the first time $N_{t}(y)=N:=\frac{\log(2/\delta)}{\epsilon^2}$ (assume that $\frac{\log(2/\delta)}{\epsilon^2}$ is an integer)  . It then holds that
\begin{align}
    \mathbb{P}\left[ e^{-8\epsilon}\frac{N}{t(y)} \leq \mathbb{P}[X_1=y] \leq e^{8\epsilon}\frac{N}{t(y)} \right ]\geq 1-\delta.\nonumber
\end{align}
\end{lemma}
\begin{proof}
Let $p=\mathbb{P}[X_1=y] $. 
	Note that $N_{t(y)}(y)=N$, it suffices to prove that
	\begin{align}
	\mathbb{P}[  e^{-8\epsilon} N \leq  t(y)p \leq e^{8\epsilon} N     ]\geq 1-\delta.\nonumber
	\end{align}
	Let $X'_i = \mathbb{I}[X_i=y]$.
	Note that
	\begin{align}
	&	\mathrm{Pr}\left[\exists n\geq 1,   \sum_{i=1}^n X_i' \geq e^{\epsilon}n \mathbb{E}[X_i']+ \frac{4}{\epsilon^2}\log(2/\delta) \right]\leq \delta/2;\nonumber
	\\ & \mathrm{Pr}\left[\exists n\geq 1,   \sum_{i=1}^n X_i' \leq e^{-\epsilon} n \mathbb{E}[X_i']-\frac{4}{\epsilon^2}\log(2/\delta) \right]\leq \delta/2,\nonumber
	\end{align}
	we have that
	\begin{align}
	\mathrm{Pr}\left[\forall n \geq 1,  e^{-\epsilon} n \mathbb{E}[X_i']-\frac{4}{\epsilon}\log(2/\delta)\leq    \sum_{i=1}^n X_i'  \leq   e^{\epsilon}n \mathbb{E}[X_i']+ \frac{4}{\epsilon}\log(2/\delta)  \right]\geq 1-\delta.
	\end{align}
	Then with probability $1-\delta$, it holds that
	\begin{align}
	e^{-\epsilon} t(y)p -\frac{4}{\epsilon}\log(2/\delta)\leq    N  \leq   e^{\epsilon}t_{y} p+ \frac{4}{\epsilon}\log(2/\delta) .
	\end{align} 
	The proof is finished by noting that $N+\frac{4}{\epsilon}\log(2/\delta)\leq  N(1+4\epsilon)\leq e^{4\epsilon}N$ and $N-\frac{4}{\epsilon}\log(2/\delta) = N(1-4\epsilon)\geq e^{-5\epsilon}$.    
\end{proof}

\section{Collection of Notations}

\begin{itemize}
\item $N^k(s,a,s')$: the value of $N(s,a,s')$ before in the $k$-th episode;
\item $N^k(s,a)=\max\{\sum_{s'}N^k(s,a,s'),1\}$;
    \item $\mathcal{K}^k:=\{(s,a,s'):   N^k(s,a,s')\geq N_0:=256S^2\log(1/\delta)\}$: \emph{known} state-action-state triples at the beginning of the $k$-th pair;
    \item $\mathcal{K}^{k}(s,a):=\{s': (s,a,s')\in \mathcal{K}^k \}$;
    \item $\mathcal{U}^k:=\{(s,a): \mathcal{K}^{k}(s,a)=\emptyset\}$: the \emph{unknown} state-action pairs;
    \item $z$: an additional state, which transits to $z'$ with probability $1$ for any action;
    \item $z'$: an absorbed state, i.e., $P_{z,a}=\textbf{1}_{z}, \forall a$;
    \item $P$: the true transition model;
    \item $\bar{P}^k$: the clipped transition model with respect to $(\mathcal{K}^k)^{C}$, i.e., the set of \emph{unknown} state-action-state triples
    \begin{align}
        &\bar{P}^{k}_{s,a,s'}= P_{s,a,s'},\forall (s,a,s')\in \mathcal{K}^k; \nonumber
        \\ & \bar{P}^k_{s,a,s'}= 0,\forall (s,a,s')\notin \mathcal{K}^k;\nonumber
        \\ & \bar{P}^k_{s,a,z} = \sum_{s'\notin \mathcal{K}^k(s,a)} P_{s,a,s'};\nonumber
    \end{align}
    \item $\bar{P}^{\mathrm{cut},k}:$ the transition model which ignore the probabilities transiting to $z$;
    \begin{align}
     &   \bar{P}_{s,a,s'}^{\mathrm{cut},k} = \frac{P_{s,a,s'}}{\sum_{s''\in \mathcal{K}^k(s,a)}P_{s,a,s''}},\forall (s,a,s')\in \mathcal{K}^{k};\nonumber
        \\ & \bar{P}_{s,a,s'}^{\mathrm{cut},k} = 0 ,\forall (s,a,s')\notin \mathcal{K}^k;\nonumber
        \\ & \bar{P}_{s,a,z}^{\mathrm{cut},k} = 1, \forall (s,a)\in \mathcal{U}^k;\nonumber
    \end{align}
    \item $P^{\mathrm{ref},k}:$ the cut-off reference model
        \begin{align}
     &   P_{s,a,s'}^{\mathrm{ref},k} = \frac{N^k(s,a,s')}{\sum_{s''\in \mathcal{K}^k(s,a)}N^k_{s,a,s''}},\forall (s,a,s')\in \mathcal{K}^{k};\nonumber
        \\ & P_{s,a,s'}^{\mathrm{ref},k} = 0 ,\forall (s,a,s')\notin \mathcal{K}^k;\nonumber
        \\ & P_{s,a,z}^{\mathrm{ref},k} = 1, \forall (s,a)\in \mathcal{U}^k;\nonumber
    \end{align}
    \item  $\mathbb{E}_{p,\pi}[\cdot]$: the expectation(probability) following $\pi$ under transition $p$;
    \item $\mathbb{P}_{p,\pi}[\cdot]:$ the probability following $\pi$ under transition $p$;
    \item $W^{\pi}_{d}(r,p,\mu_{1}):=\mathbb{E}_{p,\pi} \left[ \sum_{h=1}^H r_h |s_1\sim \mu_1\right]$: the general value function;
    \item $W_{\gamma}^{\pi}(r,p,\mu_1):=\mathbb{E}_{p,\pi}[\sum_{i\geq 1}\gamma^{i-1}r_i|s_1\sim \mu_1]$;
    \item $X^{\pi}_{d}(\mathcal{O},p,\mu_1):$ the probability of reaching $\mathcal{O}$ in $d$ steps with $(p,\pi)$ as transition-policy pair and $\mu_1$ as initial distribution;
    \item $X^{\pi}_{\gamma}(\mathcal{O},p,\mu_1):=\sum_{i\geq 1}\gamma^{i-1}\mathbb{P}_{p,\pi}[(s_i,a_i,s_{i+1})\in \mathcal{O}, (s_{i'},a_{i'},s_{i'+1})\notin \mathcal{O}, \forall 1\leq i'\leq i-1 | s_1\sim \mu_1]$,
    \item $\mathrm{cut}(p)$: the cutting function for a transition model $p$. $p' = \mathrm{Cut}(p)$ is defined by 
    \begin{align}
             &   p'_{s,a,s'}= \frac{p_{s,a,s'}}{\sum_{s''\neq z}p_{s,a,s''}} \forall (s,a) \quad \mathrm{s.t.} p_{s,a,z}<1; \nonumber
        \\ & p'_{s,a,z}= 0, \forall (s,a) \quad \mathrm{s.t.} p_{s,a,z}<1;\nonumber
        \\ & p'_{s,a,z}= 1, \forall (s,a) \quad \mathrm{s.t.} p_{s,a,z}=1.\nonumber
    \end{align}
    Note that $\bar{P}^{\mathrm{cut},k} = \mathrm{cut}(\bar{P}^k)$.
    \item $\textbf{1}_{s}$: the vector which is $1$ at $s$  and 0 otherwise.
    \item $\textbf{1}_{s,a}$: the vector which is $1$ at $(s,a)$ and 0 otherwise.
    \item $\textbf{1}_{s,a,s'}$: the vector which is $1$ at $(s,a,s')$ and 0 otherwise.
    \item $\mathcal{J}:=\{ k\in [K_1]| \exists h\in [d],a , (s_{h+1}^k,a)\in \mathcal{O}^k\}$; 
    \item $(s_1^{*,k},a_{1}^{*,k}):$ the value of $(s_1^*,a_1^*)$ in the $k$-th episode for $k\in \mathcal{J}$;
    \item $\Pi_{\mathrm{sta}}$: the set of stationary policies;
    \item $\Pi:$ the set of all possible policies.
\end{itemize}

\section{Structural Lemmas for Stationary Policies}

\begin{lemma}[Restatement of Lemma~\ref{lemma:al1}] Let $k$ and $d$ be positive integers. We have that for any $(s,a) \in \states \times \actions$,  
\[\max_{\pi \in \Pi}W^{\pi}_{kd}(\textbf{1}_{s,a},P,\textbf{1}_s)\leq 6k\max_{\pi\in \Pi_{\mathrm{sta}}}W^{\pi}_{d}(\textbf{1}_{s,a},P,\textbf{1}_s).\]
\end{lemma}
\begin{proof}[Proof of Lemma~\ref{lemma:al1}]
Let $\pi$ be fixed. Let $\tilde{\mu}_{i}$ be the distribution of $s_{di+1}$  following $\pi$. We have that
\begin{align}
W^{\pi}_{kd}(\textbf{1}_{s,a},P,\textbf{1}_s)=\sum_{i=0}^{k-1}W^{\pi^{(i)}}_{d}(\textbf{1}_{s,a},P,\tilde{\mu}_i) & \leq \sum_{i=0}^{k-1}W^{\pi^{(i)}}_{d}(\textbf{1}_{s,a},P,\textbf{1}_{s})\leq  6k\max_{\pi\in \Pi_{\mathrm{sta}}}W^{\pi}_{d}(\textbf{1}_{s,a},P,\textbf{1}_s),\nonumber
\end{align}
where $\pi^{(i)}$ is defined as $\pi^{(i)}_{h'}(a|s)=\pi_{id+h'}(a|s)$.
The first inequality uses the fact that to get a reward $\mathbf{1}_{s,a}$, the best initial state is $s$.
The second inequality uses Lemma~\ref{lemma:al2}.
The proof is finished by taking maximization over $\pi$.
\end{proof}

\begin{lemma}\label{lemma:al2}
For any horizon $d$ and state-action pair $(s,a)$, it holds that
\begin{align}
\max_{\pi\in \Pi_{\mathrm{sta}}}W^{\pi}_{d}(\textbf{1}_{s,a},P,\textbf{1}_{s})\geq \frac{1}{6}\max_{\pi}W^{\pi}_{d}(\textbf{1}_{s,a},P,\textbf{1}_{s}).\nonumber
\end{align}
\end{lemma}
\begin{proof}
Let $\pi^*$ be the optimal stationary policy with respect to reward $\textbf{1}_{s,a}$ under transition $P$ and discounted factor $\gamma = 1-\frac{1}{d}$. 
Let $\tilde{\mu}_i$ denote the distribution of $s_{id+1}$ under $P$ following $\pi$ with initial distribution $\textbf{1}_{s}$.

Then we have that for any policy   $\pi'$,
\begin{align}
  W^{\pi^*}_{d}(\textbf{1}_{s,a},P,\textbf{1}_{s})  & \geq  \frac{1}{2}\sum_{i=0}^{\infty}\gamma^{di}W^{\pi^*}_{d}(\textbf{1}_{s,a},P,\textbf{1}_{s}) \nonumber 
\\ & \geq \frac{1}{2}\sum_{i=0}^{\infty}\gamma^{di}W^{\pi^*}_{d}(\textbf{1}_{s,a},P,\tilde{\mu}_i)\nonumber 
\\ & \geq \frac{1}{2}\sum_{i=1}^{\infty} \gamma^{i-1}\mathbb{P}_{\pi^*}[(s_i,a_i)=(s,a)]\nonumber 
\\ & \geq  \frac{1}{2}\sum_{i=1}^{\infty} \gamma^{i-1}\mathbb{P}_{\pi'}[(s_i,a_i)=(s,a)]\nonumber 
\\ & \geq \frac{1}{6}\sum_{i=1}^{d}\mathbb{P}_{\pi'}[(s_i,a_i)=(s,a)]\nonumber 
\\ & = \frac{1}{6}W^{\pi'}_{d}(\textbf{1}_{s,a},P,\textbf{1}_{s}).\nonumber
\end{align}
The proof is completed by taking maximization over $\pi'$.

\end{proof}

Recall that $W^{\pi}_{\gamma}(r,P,\mu_1)$ denotes the discounted accumulative reward with reward $r$, transition $P$, policy $\pi$, initial distribution $\mu_1$ and discounted factor $\gamma$.

\begin{lemma}\label{lemma:eff}
Let $d_1,d_2$ be positive integers such that $d_1\geq 10S\log(S)d_2$. Let $\gamma = 1-1/d_2$. Let $(s,a)$, a non-negative reward $r$ and a stationary policy $\pi$ be fixed. Then we have that
\begin{align}
   & \frac{1}{3} W^{\pi}_{d_2}(\textbf{1}_{s,a},p,\textbf{1}_{s}) \leq  W^{\pi}_{\gamma}(\textbf{1}_{s,a},p,\textbf{1}_{s}) \leq 3  W^{\pi}_{d_2}(\textbf{1}_{s,a},p,\textbf{1}_{s})\label{eq:ddl10}
   \\ & W^{\pi}_{\gamma}(r,p,\textbf{1}_{s})\geq \frac{1}{3}W^{\pi}_{d_2}(r,p,\textbf{1}_{s})\label{eq:ddl11}
   \\ &  W^{\pi}_{\gamma}(r,p,\textbf{1}_{s})\leq  10W^{\pi}_{d_1}(r,p,\textbf{1}_{s}).\label{eq:ddl12}
   \end{align}
\end{lemma}
\begin{proof}
    The left side of \eqref{eq:ddl10} holds because $\gamma^{d_2}=(1-1/d_2)^{d_2}\geq 1/3$. As for the right side, letting $\mu_i$ denote the distribution of $s_{id_2+1}$ following $\pi$ starting from $s$, we have that
    \begin{align}
      &  W^{\pi}_{\gamma}(\textbf{1}_{s,a},p,\textbf{1}_{s}) \leq \sum_{i=0}^{\infty}\gamma^{d_2i} W_{d_2}^{\pi}(\textbf{1}_{s,a},p,\mu_i)\leq W^{\pi}_{d_2}(\textbf{1}_{s,a},p,\textbf{1}_{s})\sum_{i=0}^{\infty}(1-d_2)^{d_{2}i} \leq  3W^{\pi}_{d_2}(\textbf{1}_{s,a},p,\textbf{1}_{s}).
    \end{align}
The first and second inequalities are by ignoring the terms with index larger than $d_2$.
    \eqref{eq:ddl11} holds by the fact $\gamma^{d_2}=(1-1/d_2)^{d_2}\geq 1/3$. 

To prove \eqref{eq:ddl12}, letting $\mu_{i}$ denote the distribution of $s_{d_1i+1}$ following $\pi$ starting from $s$, we have that
    \begin{align}
       &  W^{\pi}_{\gamma}(r,p,\textbf{1}_{s}) \leq \sum_{i=0}^{\infty}\gamma^{d_1i}W^{\pi}_{d_1}(r,p,\mu_i)\leq \sum_{i=0}^{\infty}e^{-10S\log(S)i}W^{\pi}_{d_1}(r,p,\mu_i).\nonumber
    \end{align}
Next, we have
    \begin{align}
    \sum_{i=2^k}^{2^{k+1}-1}W_{d_1}^{\pi}(r,p,\mu_i) & \leq  \sum_{i=0}^{2^{k+1}-1}W_{d_1}^{\pi}(r,p,\mu_i)\nonumber 
    \\ & =W_{2^{k+1}d_1}^{\pi}(r,p,\textbf{1}_{s})\nonumber 
    \\ &\leq  \exp(5(k+1)S\log(S))W_{d_1}^{\pi}(r,p,\textbf{1}_{s}).\nonumber
    \end{align}
The second inequality we used Lemma~\ref{lemma:li30} for $(k+1)$ times.
    Therefore, we obtain
    \begin{align}
 &      \sum_{i=0}^{\infty}e^{-10S\log(S)i}W^{\pi}_{d_1}(r,p,\mu_i)\nonumber ^\leq \sum_{k=0}^{\infty}e^{-10S\log(S)2^k}\sum_{i=2^k}^{2^{k+1}-1}W^{\pi}_{d_1}(r,p,\mu_i)\nonumber 
 \\ & \leq  \sum_{k=0}^{\infty}e^{-S\log(S)(10\cdot 2^k-5(k+1))} W^{\pi}_{\gamma}(r,p,\textbf{1}_{s})
 \leq 10 W^{\pi}_{\gamma}(r,p,\textbf{1}_{s}).
    \end{align}
    The proof is completed.
\end{proof}

\begin{lemma}\label{lemma:li30}[Lemma 4.6 in \cite{li2021settling}]
    Suppose $d\geq S\geq 5$. Then $W_{2d}^{\pi}(r,p,\mu)\leq 4S^{4S}W^{\pi}_{d}(r,p,\mu)\leq \exp(5S\log(S))W^{\pi}_{d}(r,p,\mu)$ for any proper $r,p,\mu$ and stationary policy $\pi$.
\end{lemma}

\begin{lemma}\label{lemma:li4}
	Let $X_1, X_2,\ldots,X_n$ be i.i.d. positive random variables. Define $\tau_H:=\min\{ i|\sum_{j=1}^i X_j \geq H\}$. Then we have that 
	\begin{align}
	\mathrm{Pr}\left[ \tau_{H}\geq  \frac{1}{2}\mathbb{E}[\tau_H] -1\right]\geq \frac{1}{2}.
	\end{align}
\end{lemma}	
\begin{proof} The proof comes from the analysis in Corollary 4.9 in \cite{li2021breaking}.  
	Clearly $\mathbb{E}[\tau_H]\geq 1$. Let $\tau' =\left\lceil \tau_H  \right \rceil -1$,  it suffices to prove that
	\begin{align}
	\mathrm{Pr}\left[   \sum_{j=1}^{\tau'/2}X_j < H \right]\geq  \frac{1}{2}.
	\end{align}
	
	Define $X_1' =\min\{X_1,H\}$.
	
	By the stopping time theorem, we have that $\tau'\mathbb{E}[X_1']\leq H$, it then holds that $\frac{\tau'}{2}\mathbb{E}[X_1']\leq \frac{H}{2}$. By Markov's inequality we have that
	\begin{align}
	\mathrm{Pr}\left[  \sum_{j=1}^{\tau'/2}X_j'<H  \right]  \geq \frac{1}{2}.
	\end{align}
	Noting that  $\sum_{j=1}^{\tau'/2}X_j'<H$ implies $\sum_{j=1}^{\tau'/2}X_j<H$, we finish the proof.
	
\end{proof}

By Lemma~\ref{lemma:li4}, we  further have that
\begin{lemma}[Restatement of Lemma~\ref{lemma:stationary}]\label{lemma:li5}
 For  any $(s,a) \in \states \times \actions$ and $\pi \in \Pi_{\mathrm{sta}}$ such that $\pi(s)=a$, we have that 
$
\mathrm{Pr}\left[ N \geq \frac{1}{4}W^{\pi}_{d}(P,\textbf{1}_{s,a},\textbf{1}_{s})  \right]\geq \frac{1}{2}$
for any horizon $d$,
where $N$ is the visit count of $(s,a)$ following $\pi$ under $P$ in $d$ steps with the initial distribution as $\textbf{1}_{s}$.

\end{lemma}

\section{Proof of Lemma~\ref{lemma:stage1}}
\label{app:stage1}
\begin{proof}[Proof of Lemma~\ref{lemma:stage1}]
Recall the definition of $\pi^k,\tilde{P}^k$ and $\mathcal{O}^k$ in Algorithm~\ref{alg:main}. Recall that $d=\frac{SH}{S+1}$. 
We use the following two lemmas below.

\begin{lemma}\label{lemma:add0}
With probability $1-\delta$, we have that
\begin{align}
\max_{\pi}W^{\pi}_{d}(\textbf{1}_{z},\mathrm{Clip}(P,\mathcal{O}^{K_1+1}),\mu_1)\leq O\left(\frac{S^7A^3\iota}{K_1}\mathrm{polylog}(SAK) \right).\nonumber
\end{align}
\end{lemma}

\begin{lemma}[Formal statement of Lemma~\ref{lemma:add1}\label{lemma:add1_formal}]
For any $\mathcal{O}\subset \mathcal{S}\times \mathcal{A}$ and $\tilde{d}\geq 1$, we have that
\begin{align}
\max_{\pi}X^{\pi}_{(S+2)\tilde{d}}(\mathcal{O},P,\mu_1)\leq S^2\max_{\pi}X^{\pi}_{(S+1)\tilde{d}}(\mathcal{O},P,\mu_1).\nonumber
\end{align}
\end{lemma}


Given these two lemmas and setting  $\tilde{d}=\frac{H}{S+2}$, we can have that
\begin{align}
\max_{\pi}W_{H}^{\pi}(\textbf{1}_z ,\mathrm{Clip}(P,\mathcal{O}^{K_1+1}),\mu_1) 
 & =\max_{\pi}\mathbb{P}_{\pi}\left[\exists h\in [(S+1)\tilde{d}], (s_h,a_h)\in \mathcal{O}^{K_1+1} \right]\nonumber  \\& =\max_{\pi}X^{\pi}_{(S+1)\tilde{d}}(\mathcal{O}^{K_1+1},P,\mu_1)\nonumber 
\\ & \leq S^2\max_{\pi}X^{\pi}_{S\tilde{d}}(\mathcal{O}^{K_1+1},P,\mu_1)\nonumber 
\\ & =  S^2\max_{\pi}X^{\pi}_{d}(\mathcal{O}^{K_1+1},P,\mu_1)\nonumber 
\\ & \leq O\left(\frac{S^9A^3\iota}{K_1}\mathrm{polylog}(SAK) \right)\nonumber
\end{align}
This completes the proof of Lemma~\ref{lemma:stage1}.
\end{proof}

Below we prove these two lemmas.
\begin{proof}[Proof of Lemma~\ref{lemma:add1_formal}]
Inspired by the analysis in \cite{li2021settling}, we regard each $\tilde{d}$ steps as one \emph{big step}, which reducing the problem to a special case where $\tilde{d}=1$. Then we construct a mapping from the set of trajectories of length $(S+1)$ with final state as $z$ to the set of trajectories of length $S$ with final state as $z$, which bounds the probability of the former trajectories using the probability of latter trajectories.

 Define $\overline{P}_{s,a}=P_{s,a}$ for any $(s,a)\notin \mathcal{O}$, $\overline{P}_{s,a}=\textbf{1}_{z}$ for $(s,a)\in \mathcal{O}$ and $\overline{P}_{z,a}=\textbf{1}_{z}$ for any $a$. In words, $\overline{P}$ is a copy of $P$, except for redirecting $(s,a)\in \mathcal{O}$ to a absorbed state $z$.

Let $\mathcal{P}(1)=\{ p| p_{s}\in \mathrm{Conv}(\{\overline{P}_{s,a}\}_{a\in \mathcal{A}}) \}$ be the set of all possible $1-$step transition probability under $\overline{P}$, where $\mathrm{Conv}(\mathcal{X})$ denote the convex hull of a set $\mathcal{X}$.
Let $\mathcal{P}(\tilde{d})= \{ \Pi_{i=1}^{\tilde{d}} p_i | p_i \in \mathcal{P}(1),\forall i \}$, which is the set of $l$-th step transition probability with respect to $\overline{P}$.

By definition, we have that
\begin{align}
 & \max_{\pi}X^{\pi}_{(S+2)\tilde{d}}(\mathcal{O},P,\mu_1) = \max_{\{p^{(i)}\in \mathcal{P}(l)\}_{i=1}^{S+1}} \mu_{1}^{\top}\Pi_{i=1}^{S+1}p^{(i)} \textbf{1}_{z} \label{eq:ff0}
 \\ &  \max_{\pi}X^{\pi}_{(S+1)\tilde{d}}(\mathcal{O},P,\mu_1) = \max_{\{p^{(i)}\in \mathcal{P}(l)\}_{i=1}^{S+1}} \mu_{1}^{\top}\Pi_{i=1}^{S}p^{(i)} \textbf{1}_{z}.\label{eq:ff1}
\end{align}

Let $\mathcal{T}=\{  \{\tilde{s}_i\in \mathcal{S}\cup\{z\}\}_{i=1}^{S+1}    \}$ be the set of all possible trajectories with length $S+1$. Then we have that for any $\{p^{(i)}\in\mathcal{P}(l)\}_{i=1}^{S+1}$
\begin{align}
 \mu_{1}^{\top}\Pi_{i=1}^{S+1}p^{(i)} \textbf{1}_{z} = \sum_{ \{\tilde{s}_i\}_{i=1}^{S+1}\in\mathcal{T}}\mu_{1}(\tilde{s}_1)\Pi_{i=1}^{S}p^{(i)}_{\tilde{s}_i,\tilde{s}_{i+1}}\cdot p^{(S+1)}_{\tilde{s}_{S+1},z} .\label{eq:22251}
\end{align}

 For any trajectory $ \{\tilde{s}_i\}_{i=1}^{S+1}\in \mathcal{T}$, by the pigeon hole principle, it either holds that $\tilde{s}_{S+1}=z$ or $\exists i_1,i_2$ such that $\tilde{s}_{i_1}=\tilde{s}_{i_2}$. In the first case, we define that $\mathcal{T}'= \{      \{\tilde{s}_i\}_{i=1}^{S+1}\in \mathcal{T}, :\tilde{s}_{S+1}=z     \}$. Then we  have that for any $ \{\tilde{s}_i\}_{i=1}^{S+1}\in \mathcal{T}'$,
\begin{align}
\mu_{1}(\tilde{s}_1)\Pi_{i=1}^{S}p^{(i)}_{\tilde{s}_i,\tilde{s}_{i+1}}\cdot p^{(S+1)}_{\tilde{s}_{S+1},z}= \mu_{1}(\tilde{s}_1)\Pi_{i=1}^{S-1}p^{(i)}_{\tilde{s}_i,\tilde{s}_{i+1}}\cdot p^{(S)}_{\tilde{s}_{S},z}.\label{eq:w1}
\end{align}

Taking sum, we have that
\begin{align}
\sum_{\{\tilde{s}_i\}_{i=1}^{S+1}\in \mathcal{T}'}\mu_{1}(\tilde{s}_1)\Pi_{i=1}^{S}p^{(i)}_{\tilde{s}_i,\tilde{s}_{i+1}}\cdot p^{(S+1)}_{\tilde{s}_{S+1},z}  = \sum_{\{\tilde{s}_i\}_{i=1}^{S}\in \mathcal{T}}\mu_{1}(\tilde{s}_1)\Pi_{i=1}^{S-1}p^{(i)}_{\tilde{s}_i,\tilde{s}_{i+1}}\cdot p^{(S)}_{\tilde{s}_{S},z} = \mu_1^{\top}\Pi_{i=1}^{S} p^{(i)}\textbf{1}_{z}.\label{eq:w6_lem15}
\end{align}

In the second case, for a fixed $(i_1,i_2)$, we define $\mathcal{T}(i_1,i_2) = \{    \{\tilde{s}_{i}  \}_{i=1}^{S+1}\in \mathcal{T}: \tilde{s}_{S+1}\neq z, \tilde{s}_{i_1}=\tilde{s}_{i_2}    \}$ for $1\leq i_1<i_2\leq S+1$.

Then we have that 
\begin{align}
& \sum_{\{\tilde{s}_{i}  \}_{i=1}^{S+1}\in \mathcal{T}(i_1,i_2)}\mu_{1}(\tilde{s}_1)\Pi_{i=1}^{S}p^{(i)}_{\tilde{s}_i,\tilde{s}_{i+1}}\cdot p^{(S+1)}_{\tilde{s}_{S+1},z}
\nonumber 
\\ & =\sum_{\{\tilde{s}_{i}  \}_{i=1}^{S+1}\in \mathcal{T}(i_1,i_2)} \mu_{1}(\tilde{s}_1)\Pi_{i=1}^{i_1-1}p^{(i)}_{\tilde{s}_i,\tilde{s}_{i+1}}\cdot \Pi_{i=i_1}^{i_2-1} p^{(i)}_{\tilde{s}_i,\tilde{s}_{i+1}}\cdot  \Pi_{i=i_2}^{S}p^{(i)}_{\tilde{s}_i,\tilde{s}_{i+1}}\cdot p^{(S+1)}_{\tilde{s}_{S+1},z}.\nonumber 
\\ & \leq \sum_{  \{\tilde{s}_i \}_{i=1}^{i_1}, \{\tilde{s}  \}_{i=i_2+1}^{S+1}       }\mu_{1}(\tilde{s}_1)\Pi_{i=1}^{i_1-1}p^{(i)}_{\tilde{s}_i,\tilde{s}_{i+1}} \cdot \Pi_{i=i_2}^{S}p^{(i)}_{\tilde{s}_i,\tilde{s}_{i+1}}\cdot p^{(S+1)}_{\tilde{s}_{S+1},z} \cdot \sum_{    \{\tilde{s}_i\}_{i=i_1+1}^{i_2-1}   } \Pi_{i=i_1}^{i_2-1} p^{(i)}_{\tilde{s}_i,\tilde{s}_{i+1}}\nonumber 
\\ & \leq \sum_{  \{\tilde{s}_i \}_{i=1}^{i_1}, \{\tilde{s}  \}_{i=i_2+1}^{S+1}       }\mu_{1}(\tilde{s}_1)\Pi_{i=1}^{i_1-1}p^{(i)}_{\tilde{s}_i,\tilde{s}_{i+1}} \cdot \Pi_{i=i_2}^{S}p^{(i)}_{\tilde{s}_i,\tilde{s}_{i+1}}\cdot p^{(S+1)}_{\tilde{s}_{S+1},z} \label{eq:w3_lem15}
\\ & = \mu_1^{\top} \Pi_{i=1}^{i_1-1}p^{(i)}\cdot \Pi_{i=i_2}^{S+1}p^{(i)}\textbf{1}_{z}.\label{eq:w5}
\end{align}

Here \eqref{eq:w3_lem15} holds by the fact that $\sum_{    \{\tilde{s}_i\}_{i=i_1+1}^{i_2-1}   } \Pi_{i=i_1}^{i_2-1} p^{(i)}_{\tilde{s}_i,\tilde{s}_{i+1}}$ is the probability of transiting to $\tilde{s}_{i_1}$ from $\tilde{s}_{i_1}$ using $i_2-i_1$ steps under transition $\{p^{(i)}\}_{i=i_1}^{i_2-1}$, which is bounded by $1$.

By \eqref{eq:w5} and \eqref{eq:w6_lem15}, we obtain that
\begin{align}
& \mu_{1}^{\top}\Pi_{i=1}^{S+1}p^{(i)} \textbf{1}_{z} \nonumber\\
&  = \sum_{ \{\tilde{s}\}_{i=1}^{S+1}\in\mathcal{T}}\mu_{1}(\tilde{s}_1)\Pi_{i=1}^{S}p^{(i)}_{\tilde{s}_i,\tilde{s}_{i+1}}\cdot p^{(S+1)}_{\tilde{s}_{S+1},z}\nonumber 
\\ & \leq \sum_{\{\tilde{s}\}_{i=1}^{S+1}\in \mathcal{T}'}\mu_{1}(\tilde{s}_1)\Pi_{i=1}^{S}p^{(i)}_{\tilde{s}_i,\tilde{s}_{i+1}}\cdot p^{(S+1)}_{\tilde{s}_{S+1},z} + \sum_{1\leq i_1<i_2\leq S+1}\sum_{\{\tilde{s}\}_{i=1}^{S+1}\in \mathcal{T}(i_1,i_2)}\mu_{1}(\tilde{s}_1)\Pi_{i=1}^{S}p^{(i)}_{\tilde{s}_i,\tilde{s}_{i+1}}\cdot p^{(S+1)}_{\tilde{s}_{S+1},z} \nonumber 
\\ & \leq \mu_1^{\top}\Pi_{i=1}^{S} p^{(i)}\textbf{1}_{z}+ \sum_{1\leq i_1<i_2\leq S+1}\mu_1^{\top} \Pi_{i=1}^{i_1-1}p^{(i)}\cdot \Pi_{i=i_2}^{S+1}p^{(i)}\textbf{1}_{z}\nonumber 
\\ & \leq S^2 \max_{\{p^{(i)}\in \mathcal{P}(l)\}_{i=1}^{S}} \mu_{1}^{\top}\Pi_{i=1}^{S}p^{(i)} \textbf{1}_{z}.\label{eq:w7_lem15}
\end{align}
Noting that \eqref{eq:w7_lem15} holds for any $\{p^{(i)}\in \mathcal{P}(l)\}_{i=1}^{S+1}$, we conclude that
\begin{align}
 \max_{\{p^{(i)}\in \mathcal{P}(l)\}_{i=1}^{S+1}} \mu_{1}^{\top}\Pi_{i=1}^{S+1}p^{(i)} \textbf{1}_{z}\leq S^2 \max_{\{p^{(i)}\in \mathcal{P}(l)\}_{i=1}^{S}} \mu_{1}^{\top}\Pi_{i=1}^{S}p^{(i)} \textbf{1}_{z}.\nonumber
\end{align}
The proof is completed  by \eqref{eq:ff0} and \eqref{eq:ff1}.
\end{proof}

\subsection{Proof of Lemma~\ref{lemma:add0}}

The following lemma guarantees that all state-action pairs in the known set (i.e., $\notin \mathcal{O}^k$), we have collected enough data.

\begin{lemma}\label{lemma:da1} Recall the definition of $U(s,a) = \max_\pi W^\pi_H(\textbf{1}_{s,a}, P, \mu_1)$.
With probability $1-10SAK\delta$,
for each $1\leq k \leq K_1$ and each $(s,a)\notin \mathcal{O}^{k}$, we have that
\begin{align}
N^k(s,a) \geq \frac{C}{71S(S+1)\log(S)}U(s,a),
\end{align}
where $C$ is an universal constant.
\end{lemma}

\begin{proof}[Proof of Lemma~\ref{lemma:da1}]
By Lemma~\ref{lemma:sr}, with probability $1-SAK\delta$, it holds that $N^k(s,a)\geq C\max_{\pi\in \Pi_{\mathrm{sta}}}W^{\pi}_{\frac{H}{2S(S+1)\log(S)}}(\textbf{1}_{s,a},P,\textbf{1}_{s})$ for some constant $C$.

Now we can lower bound
\begin{align*}
N^k(s,a) & \geq C\max_{\pi\in \Pi_{\mathrm{sta}}}W^{\pi}_{\frac{H}{2S(S+1)\log(S)}}(\textbf{1}_{s,a},P,\textbf{1}_{s})  \\
& \geq \frac{C}{12S(S+1)\log(S)}\max_{\pi}W^{\pi}_{H}(\textbf{1}_{s,a},P,\textbf{1}_{s})\\
 & \geq \frac{C}{12S(S+1)\log(S)}\max_{\pi}W^{\pi}_{H}(\textbf{1}_{s,a},P,\mu_1)\nonumber \\
&  =  \frac{C}{12S(S+1)\log(S)}U(s,a)
\end{align*}
where the first inequality we used Lemma~\ref{lemma:al1}, and the second inequality we used that $\textbf{1}_s$ is the optimal initial distribution for reward $\mathbf{1}_{s,a}$.

\end{proof}

\begin{proof}[Proof of Lemma~\ref{lemma:add0}]
Now we proceed to prove  Lemma~\ref{lemma:add0}. We first make some definitions.

Define $\bar{P}^k = \mathrm{Clip}(P,\mathcal{O}^k)$. Note that $\mathcal{O}^k$ never appears under $\bar{P}^k$. We then define the state-action space of $\bar{P}^k$ as $\Phi^k$. Below we continue the analysis for the $k$-th episode in the first stage under the context of $\bar{P}^k$. 
 Let $\bar{r}^k=\textbf{1}_{z}$. Note that $\mathcal{O}^k$ varies in $k$, then the definition of $z$ varies in different episodes. As a result, $\bar{P}^k$ and $\bar{r}^k$ also vary in $k$.
 
Define $\tilde{R}=\sum_{k=1}^{K_1} \max_{\pi}W^{\pi}(\textbf{1}_{z},\bar{P}^k,\mu_1)-\sum_{k=1}^{K_1}\sum_{h=1}^H \bar{r}^k(s_h^k,a_h^k)$, which can be viewed as the regret.
Note this is different from the regret in standard MDP because the reward is not fixed ($z$ depends on $k$).


Recall that $(\pi^k,\tilde{P}^k)=\arg\max_{\pi,p\in \mathcal{P}^k}X_{d}^{\pi}(\mathcal{O}^k,p,\mu_1)$. 
Let $\tilde{p}^k=\mathrm{Clip}(\tilde{P}^k,\mathcal{O}^k)$, $\hat{p}^k = \mathrm{Clip}(\hat{P}^k,\mathcal{O}^k)$ and $\{\tilde{V}^k_h(s)\}_{(h,s)\in [H]\times \mathcal{S}^k}$ be the value function with reward $\textbf{1}_{z}$ and transition $\tilde{p}^k$.

Define $\check{\mathcal{J}}:=\{(k,h):k\in [K_1], \exists h'\leq h, (s,a)\in \mathcal{S}\times\mathcal{A}, N_{h'}^k(s,a)>2N^k(s,a)+1 \}$ and $i_h^k=\mathbb{I}[(k,h)\notin \check{\mathcal{J}}]$.

Following the regret analysis in \cite{zhang2020reinforcement}, by the optimality of $\tilde{P}^k$, we have that
\begin{align}
&\sum_{k=1}^{K_1} \max_{\pi}W^{\pi}_{d}(\textbf{1}_{z},\bar{P}^k,\mu_1)-\sum_{k=1}^{K_1}\sum_{h=1}^d \bar{r}^k(s_h^k,a_h^k) \nonumber 
\\ & \leq \sum_{k=1}^{K_1} \tilde{V}_1^k(s_1^k)i_1^k-\sum_{k=1}^{K_1}\sum_{h=1}^d \bar{r}^k(s_h^k,a_h^k) \nonumber 
\\ & = \sum_{k=1}^{K_1} \sum_{h=1}^d \left(\tilde{V}_h^k(s_h^k)i_h^k - \tilde{V}_{h+1}^k(s_{h+1}^k)i_{h+1}^k\right) -  \sum_{k=1}^{K_1}\sum_{h=1}^d \bar{r}^k(s_h^k,a_h^k) \nonumber
\\ & = \sum_{k=1}^{K_1} \sum_{h=1}^d \left(\tilde{p}_{s_h^k,a_h^k}^k \tilde{V}_{h+1}^k i_h^k + \bar{r}^k(s_h^k,a_h^k)i_h^k- \textbf{1}_{s_{h+1}^k}\tilde{V}_{h+1}^ki_{h+1}^k\right) -  \sum_{k=1}^{K_1}\sum_{h=1}^d \bar{r}^k(s_h^k,a_h^k) \nonumber
\\ & \leq \sum_{k=1}^{K_1} \sum_{h=1}^d \left(\tilde{p}_{s_h^k,a_h^k}^k \tilde{V}_{h+1}^k - \textbf{1}_{s_{h+1}^k}\tilde{V}_{h+1}^k\right)i_{h+1}^k +  \sum_{k=1}^{K_1}\sum_{h=1}^d \tilde{p}^k_{s_{h}^k,a_h^k}\tilde{V}_{h+1}^k (i_h^k-i_{h+1}^k) \nonumber
\\ & = \sum_{k=1}^{K_1} \left(\sum_{h=1}^d (\tilde{p}^k_{s_h^k,a_h^k}- \bar{P}^k_{s_h^k,a_h^k}) \tilde{V}_{h+1}^ki_{h+1}^k +(\bar{P}^k_{s_h^k,a_h^k}- \textbf{1}_{s_{h+1}^k})\tilde{V}_{h+1}^ki_{h+1}^k\right) + \sum_{k=1}^{K_1}\sum_{h=1}^d \tilde{p}^k_{s_{h}^k,a_h^k}\tilde{V}_{h+1}^k (i_h^k-i_{h+1}^k) \nonumber
\\ & =  \sum_{k=1}^{K_1} \left(     \sum_{h=1}^d \left(  (\tilde{p}^k_{s_h^k,a_h^k}- \bar{P}^k_{s_h^k,a_h^k}) (\tilde{V}_{h+1}^k  -\bar{P}_{s_h^k,a_h^k}^k \tilde{V}_{h+1}^k \cdot \textbf{1})i_{h+1}^k + (\bar{P}^k_{s_h^k,a_h^k}-\textbf{1}_{s_{h+1}^k})\tilde{V}_{h+1}^k i_{h+1}^k    \right)      \right) \nonumber
\\ & \quad \quad \quad \quad \quad \quad \quad \quad \quad \quad\quad \quad \quad \quad \quad \quad \quad \quad \quad \quad\quad \quad \quad \quad \quad\quad + \sum_{k=1}^{K_1}\sum_{h=1}^d \tilde{p}^k_{s_{h}^k,a_h^k}\tilde{V}_{h+1}^k (i_h^k-i_{h+1}^k) \label{eq:addsub}
\\ & \leq 20\sum_{k=1}^{K_1} \sum_{h=1}^d \sum_{s'\neq z}    \left(     \sqrt{\frac{\bar{P}^k_{s_h^k,a_h^k,s'}\iota}{N^k(s_h^k,a_h^k)}} + \frac{\iota}{N^k(s_h^k,a_h^k)}        \right) |\tilde{V}_{h+1}^{k}(s')-\bar{P}^k_{s_h^k,a_h^k}\tilde{V}_{h+1}^k|i_{h+1}^k  \nonumber 
\\ & \quad \quad \quad \quad \quad \quad \quad \quad \quad \quad \quad  + \sum_{k=1}^{K_1}\sum_{h=1}^d (\bar{P}^k_{s_h^k,a_h^k}-\textbf{1}_{s_{h+1}^k})\tilde{V}_{h+1}^ki_{h+1}^k + \sum_{k=1}^{K_1}\mathbb{I}[\exists h\in [d], (k,h)\in \bar{\mathcal{J}}] \label{eq:cr1}
\\ & \leq 20\sqrt{\sum_{k=1}^{K_1}\sum_{h=1}^d \frac{Si_{h+1}^k\iota}{N^k(s_h^k,a_h^k)}}\cdot \sqrt{\sum_{k=1}^{K_1}\sum_{h=1}^d \mathbb{V}(\bar{P}^k_{s_h^k,a_h^k}, \tilde{V}_{h+1}^{k})i_{h+1}^k  }  + 20\sum_{k=1}^{K_1}\sum_{h=1}^d \frac{Si_{h+1}^k\iota}{N^k(s_h^k,a_h^k)}\nonumber
\\ & \quad \quad \quad \quad \quad \quad \quad \quad \quad \quad \quad  + \sum_{k=1}^{K_1}\sum_{h=1}^d (\bar{P}^k_{s_h^k,a_h^k}-\textbf{1}_{s_{h+1}^k})\tilde{V}_{h+1}^ki_{h+1}^k + \sum_{k=1}^{K_1}\mathbb{I}[\exists h\in [d], (k,h)\in \bar{\mathcal{J}}].\label{eq:al10}
\end{align}
In the first equality we used the fact that $i_1^k = 1$.
In \eqref{eq:addsub}, we used the fact that $\bar{P}_{s_h^k,a_h^k}^k \tilde{V}_{h+1}^k \cdot \textbf{1}$ is a constant factor and $\|\tilde{p}^k_{s_h^k,a_h^k}\|_1 = \|\bar{P}^k_{s_h^k,a_h^k}\|_1 = 1$. 
In \eqref{eq:cr1}, we used the fact that $$|\tilde{p}^k_{s_h^k,a_h^k,s'}-\bar{P}_{s_h^k,a_h^k,s'}^k|\leq    \left(     \sqrt{\frac{4\hat{P}^k_{s_h^k,a_h^k,s'}\iota}{N^k(s_h^k,a_h^k)}} + \frac{5\iota}{N^k(s_h^k,a_h^k)}        \right)\leq 20 \left(     \sqrt{\frac{\bar{P}^k_{s_h^k,a_h^k,s'}\iota}{N^k(s_h^k,a_h^k)}} + \frac{\iota}{N^k(s_h^k,a_h^k)}        \right)$$ for $s'\neq z$. Lastly, \eqref{eq:al10} holds by Cauchy's inequality.

Define 
\begin{align}
& T_1= \sum_{k=1}^{K_1}\sum_{h=1}^d \frac{i_{h+1}^k\iota}{N^k(s_h^k,a_h^k)}\nonumber 
\\ & T_2 = \sum_{k=1}^{K_1}\sum_{h=1}^d \mathbb{V}(\bar{P}^k_{s_h^k,a_h^k}, \tilde{V}_{h+1}^{k})i_{h+1}^k \nonumber 
\\ & T_3 =  \sum_{k=1}^{K_1}\sum_{h=1}^d (\bar{P}^k_{s_h^k,a_h^k}-\textbf{1}_{s_{h+1}^k})\tilde{V}_{h+1}^ki_{h+1}^k \nonumber 
\\ & T_4 = \sum_{k=1}^{K_1}\mathbb{I}[\exists h\in [d], (k,h)\in \bar{\mathcal{J}}].\nonumber
\end{align}
The following lemma bounds these four terms.

\begin{lemma}\label{lemma:bdt}
With probability $1-4\delta$, $T_1,T_4\leq SAB$ with $B=O(\iota \mathrm{polylog}(SAK))$, 
\begin{align*}
T_2\leq & O\left(\mathrm{polylog}(SAK)\left(\sum_{k=1}^{K_1}\sum_{h=1}^d \bar{r}^k(s_h^k,a_h^k)+S^2AB\iota\right)\right)\\\leq 
&O\left(\mathrm{polylog}(SAK)(S^7A^3\iota +S^2A\iota^2)\right),\end{align*} and $T_3\leq O\left(\sqrt{S^7A^3\iota^2+S^2A\iota^3}\mathrm{polylog}(SAK)\right)$.
\end{lemma}

By Lemma~\ref{lemma:bdt}, with probability $1-4\delta$, we have that
\begin{align}
& \sum_{k=1}^{K_1} \max_{\pi}W^{\pi}(\textbf{1}_{z},\bar{P}^k,\mu_1)-\sum_{k=1}^{K_1}\sum_{h=1}^d \bar{r}^k(s_h^k,a_h^k)\nonumber
\\ & \leq O\left(\mathrm{polylog}(SAK)\left(\sqrt{S^2AB\iota\sum_{k=1}^{K_1}\sum_{h=1}^d \bar{r}^k(s_h^k,a_h^k) } +S^2AB\iota\right) \right)\nonumber 
\\ & \leq O\left(\mathrm{polylog}(SAK)\left(\sum_{k=1}^{K_1}\sum_{h=1}^d \bar{r}^k(s_h^k,a_h^k)+   S^2AB\iota\right)\right),\nonumber
\end{align}
where it follows that
\begin{align}
\sum_{k=1}^{K_1} \max_{\pi}W^{\pi}(\textbf{1}_{z},\bar{P}^k,\mu_1)\leq O\left(\mathrm{polylog}(SAK)\left(\sum_{k=1}^{K_1}\sum_{h=1}^d \bar{r}^k(s_h^k,a_h^k)+   S^2AB\iota\right)\right).\label{eq:al13}
\end{align}

Let $u = \max_{\pi}W_{d}^{\pi}(\textbf{1}_{z},\bar{P}^{K_1+1},\mu_1)$ be the maximal possible probability of visiting $\mathcal{O}^{k+1}$. Noting that $\mathcal{O}^k$ is non-increasing in $k$, the probability of visiting $\mathcal{O}^k$ is also non-increasing in $k$. By \eqref{eq:al13} we have that
\begin{align}
K_1u \leq & O\left(\mathrm{polylog}(SAK)\left(\sum_{k=1}^{K_1}\sum_{h=1}^d \bar{r}^k(s_h^k,a_h^k)+   S^2AB\iota\right)\right) \\
= &O((S^7A^3\iota+S^2A\iota^2)\mathrm{polylog}(SAK)),
\end{align}
which implies that $u=O\left(\frac{S^7A^3\iota+S^2A\iota}{K_1}\mathrm{polylog}(SAK)\right)$.

The proof is completed.

\end{proof}

It remains to prove Lemma~\ref{lemma:bdt}.
\begin{proof}[Proof of Lemma~\ref{lemma:bdt}]
We start with bounding $T_1$ and $T_4$. Define $\mathcal{L}(s,a) = \{k\in [K_1]: N^{k+1}(s,a)-N^k(s,a)\geq K^2U(s,a) \}$. By Lemma~\ref{lemma:ratiocon}, $|\mathcal{L}(s,a)|\leq O(1/K^2+\iota)$ with probability $1-\delta$. By Lemma~\ref{lemma:da1}, $N^k(s,a)\geq \frac{C}{12S(S+1)\log(S)}U(s,a)$ for any $(s,a)\notin \mathcal{O}^k$. By definition
\begin{align}
    T_1 &  \leq  2\sum_{k=1}^{K_1}\sum_{(s,a)\notin \mathcal{O}^k} \min\left\{ \log(\frac{N^{k+1}(s,a)}{N^k(s,a)}) ,1\right\}.\nonumber
\end{align}
Fix $(s,a)\notin \mathcal{O}^{K_1+1}$. Noting that $\mathcal{O}^k$ is non-increasing in $k$, there exists some $k'$, such that $(s,a)\notin \mathcal{O}^{k}$ for $k\geq k'$ and $(s,a)\in \mathcal{O}^{k'-1}$.

 Suppose $\mathcal{L}_2(s,a)\cap \{k: k'\leq k\leq K_1\}=\{k_1,k_2,...,\}$. Let $k_0=k'-1$.  We have that
\begin{align}
& \sum_{k=k':k\notin \mathcal{L}(s,a)}^{K_1}\min\{\log(N^{k+1}(s,a)/N^k(s,a)),1 \}\nonumber
\\ & =\sum_{i\geq 0}\sum_{k=k_i+1}^{k_{i+1}-1}\log(N^{k+1}(s,a)/N^k(s,a))\nonumber
\\ & \leq \sum_{i\geq 0}\log\left(\frac{K^3U(s,a)+N^{k_{i}+1}}{N^{k_i+1}}\right) \nonumber
\\& \leq |\mathcal{L}(s,a)|\log(K^3U(s,a)/N^{k'}(s,a))\nonumber
\\ & \leq O(\mathrm{polylog}(SAK)\iota).\nonumber
 \end{align}
It then holds that
\begin{align}
    & \sum_{k=k'}^{K_1} \min\left\{ \log(\frac{N^{k+1}(s,a)}{N^k(s,a)}) ,1\right\}\nonumber 
    \\ & \leq |\mathcal{L}(s,a)| + |\mathcal{L}(s,a)|\log(K^3U(s,a)/N^{k'}(s,a))\nonumber 
    \\ & \leq O(\mathrm{polylog}(SAK)\iota).
\end{align}
Taking sum over $(s,a)$, we obtain that $T_1\leq O(\mathrm{polylog}(SAK)SA\iota)$

In a similar way we have that
\begin{align}
    T_4 \leq \sum_{k=1}^{K_1}\sum_{(s,a)\notin \mathcal{O}^k} \min\left\{ \log(\frac{N^{k+1}(s,a)}{N^k(s,a)}) ,1\right\}\leq O(\mathrm{polylog}(SAK)SA\iota).
\end{align}

To bound $T_2$, following regret analysis in \cite{zhang2020reinforcement}, we have that
\begin{align}
T_2 &  =  \sum_{k=1}^{K_1}\sum_{h=1}^d \mathbb{V}(\bar{P}^k_{s_h^k,a_h^k}, \tilde{V}_{h+1}^{k})i_{h+1}^k \nonumber 
\\  & = \sum_{k=1}^{K_1}\sum_{h=1}^d  \left( \bar{P}^k_{s_h^k,a_h^k}(\tilde{V}_{h+1}^k)^2 - (\bar{P}^k_{s_h^k,a_h^k}\tilde{V}_{h+1}^k)^2 \right)i_{h+1}^k \nonumber 
\\ & =\sum_{k=1}^{K_1}\sum_{h=1}^d  \left( \bar{P}^k_{s_h^k,a_h^k}(\tilde{V}_{h+1}^k)^2i_{h+1}^k  - (\tilde{V}_{h}^k(s_h^k))^2i_h^k  \right)+\sum_{k=1}^{K_1}\sum_{h=1}^d  \left((\tilde{V}_{h}^k(s_h^k))^2i_h^k -(\bar{P}^k_{s_h^k,a_h^k}\tilde{V}_{h+1}^k)^2i_{h+1}^k \right) \nonumber 
\\ & \leq \sum_{k=1}^{K_1}\sum_{h=1}^d  \left( \bar{P}^k_{s_h^k,a_h^k}(\tilde{V}_{h+1}^k)^2i_{h+1}^k  - (\tilde{V}_{h}^k(s_h^k))^2i_h^k  \right)+2 \sum_{k=1}^{K_1}\sum_{h=1}^d \min\{  \tilde{V}_{h}^k(s_h^k) i_h^k-\bar{P}^k_{s_h^k,a_h^k}\tilde{V}_{h+1}^ki_{h+1}^k   ,0\}\nonumber 
\\ & \leq \sum_{k=1}^{K_1}\sum_{h=1}^d  \left( \bar{P}^k_{s_h^k,a_h^k}(\tilde{V}_{h+1}^k)^2i_{h+1}^k  - (\tilde{V}_{h+1}^k(s_{h+1}^k))^2i_{h+1}^k  \right) +2\sum_{k=1}^{K_1}\sum_{h=1}^d \bar{r}^k(s_h^k,a_h^k)+8\sqrt{S^2AB\iota T_2}+2T_4 \nonumber 
\\ & \leq \sum_{k=1}^{K_1}\sum_{h=1}^d  \left( \bar{P}^k_{s_h^k,a_h^k}(\tilde{V}_{h+1}^k)^2i_{h+1}^k  - (\tilde{V}_{h+1}^k(s_{h+1}^k))^2i_{h+1}^k  \right) +2\sum_{k=1}^{K_1}\sum_{h=1}^d \bar{r}^k(s_h^k,a_h^k)+\frac{1}{2}T_2 + 8S^2AB\iota + 2B\nonumber
\\ & \leq  2\sum_{k=1}^{K_1}\sum_{h=1}^d  \left( \bar{P}^k_{s_h^k,a_h^k}(\tilde{V}_{h+1}^k)^2i_{h+1}^k  - (\tilde{V}_{h+1}^k(s_{h+1}^k))^2i_{h+1}^k  \right) +4\sum_{k=1}^{K_1}\sum_{h=1}^d \bar{r}^k(s_h^k,a_h^k)+ 20S^2AB\iota.\label{eq:al11}
\end{align}
where the last step is by solving $T_2$.

Define \[T_5=\sum_{k=1}^{K_1}\sum_{h=1}^d(\bar{P}^k_{s_h^k,a_h^k}(\tilde{V}_{h+1}^k)^2i_{h+1}^k-\tilde{V}_{h+1}^k(s_{h+1}^k)^2i_{h+1}^k )\] and \[T_6 =\sum_{k=1}^{K_1}\sum_{h=1}^d \mathbb{V}(\bar{P}^k_{s_h^k,a_h^k},(\tilde{V}_{h+1}^k)^2i_{h+1}^k ).\] By Lemma~\ref{lemma:varb}, we have that
\begin{align}
T_6\leq 4\sum_{k,h} \mathbb{V}(\bar{P}^k_{s_h^k,a_h^k},\tilde{V}_{h+1}^ki_{h+1}^k) =4T_2.
\end{align}

We note that here we cannot directly use the Freedman's inequality because we do not know a tight upper bound of variance and using a naive upper bound will lead to a dependency on $H$.
Instead, we resort to Lemma~\ref{lemma:self-norm}.

By Lemma~\ref{lemma:self-norm}, we have that
\begin{align}
\mathbb{P}\left[ \exists i, T_5 \geq 10\cdot 2^i i\iota, T_6\leq 2^{2i}i\iota             \right]\leq \delta, \label{eqn:use_lemma13}
\end{align}
which implies that
\begin{align}
\mathbb{P}\left[ \exists i, T_5 \geq 10\cdot 2^i i\iota, T_2\leq 2^{2i-2}i\iota             \right]\leq \delta.
\end{align}
Therefore, with probability $1-\delta$, for any $i\geq 1$, it either holds $T_5<10\cdot 2^i i\iota$ or $T_2>2^{2i-2}i\iota$.
Then we have that
\begin{align}
T_2 \leq T_5 + 4K + 20\sqrt{S^2AKL\iota}+60S^2AB\iota \leq T_5 + 8K + 60S^2AL\iota.\nonumber
\end{align}
Suppose $T_5 \geq C \geq 8K + 60S^2AB\iota$, then we have that
\begin{align}
T_2 \geq \frac{T^2_{5}}{800\iota \log_2(C)}\geq \frac{C^2}{800\iota \log_2(C)}\geq 3C.
\end{align}
Then we have that $T_5\geq T_2 -( 8K + 60S^2AB\iota)\geq T_2-C\geq 2C$. In this case, $T_5 $ is infinite, which leads to a contradiction. Therefore, with probability $1-\delta$, $T_5 <8K + 60S^2AB\iota $, and it follows that  
$T_2 \leq 16K+ 480S^2AB\iota$. As a result, $T_6\leq 64K + 480S^2AB\iota$ and $T_5 \leq \sqrt{800\iota (\log_2(K)+10)T_6}=O(\sqrt{\iota T_2})\polylog(S,A,K,1/\delta)$. Recall that  $\mathcal{J}==\{ k\in [K_1]| \exists h\in [d],a , (s_{h+1}^k,a)\in \mathcal{O}^k\}$. By Lemma~\ref{lemma:numc}, we have that $|\mathcal{J}|\leq  O(S^7A^3\iota\mathrm{polylog}(SAK))$. 
Therefore, we have 
\begin{align}
T_2 &\leq 4\sum_{k=1}^{K_1}\sum_{h=1}^d \bar{r}^k(s_h^k,a_h^k)+ 20S^2AB\iota+O(\sqrt{\iota T_2}\polylog(S,A,K,1/\delta)) \nonumber 
\\ & \leq O\left(  \mathrm{polylog}(SAK)\left(\sum_{k=1}^{K_1}\sum_{h=1}^d \bar{r}^k(s_h^k,a_h^k)+S^2AB\iota\right)\right)\nonumber 
\\ & \leq O\left(  \mathrm{polylog}(SAK)\left(S^7A^3\iota+S^2A\iota^2\right)\right).\label{eq:al12}
\end{align}
Here \eqref{eq:al12} is by the fact that $\sum_{k=1}^{K_1}\sum_{h=1}^d \bar{r}^k(s_h^k,a_h^k)\leq |\mathcal{J}|\leq  O(S^7A^3\iota\mathrm{polylog}(SAK))$.

Using Lemma~\ref{lemma:self-norm} again, and noting that with probability $1-\delta$, $T_2\leq O(\mathrm{polylog}(SAK)(S^7A^3\iota+S^2A\iota^2))$, we learn that  with probability $1-\delta$
\begin{align}
T_3\leq O(\sqrt{T_2\iota}+\iota)\leq  O\left(\sqrt{S^7A^3\iota^2+S^2A\iota^3}\mathrm{polylog}(SAK)\right).\nonumber
\end{align}
The proof is finished.

\end{proof}

\section{Proof of Lemma~\ref{lemma:stage2}}\label{app:stage2}

\paragraph{Notations}
Since the proof is independent of our main proof, we will re-use some notations for simplicity. We re-define  $N^k(s,a,s')$ be the count of $(s,a,s')$ before the $k$-th episode in the second stage. Let $N^k_{h}(s,a,s')$ be the count of $(s,a,s')$ before the $h$-th step  in the $k$-th episode in the second stage. We also define $N^k(s,a)=\max\{\sum_{s'}N^k(s,a,s'),1\}$ and $N^k_h(s,a)=\max\{\sum_{s'}N^k_h(s,a,s'),1\}$. 

Define $(k,h)\leq (k',h')$ when $k'>k$ or $k'=k,h'\geq  h$.
Similarly we define $(k,h) < (k',h')$ when $k'>k$ or $k'=k,h' >  h$.
Let $\mathcal{F}_h^k= \sigma(\{s_{h'}^{k'}\}_{(k',h')<(k,h)})$

Now we bound the regret. Conditioned on $\mathcal{G}$, we have that $P\in \mathcal{P}^k$ for any $1\leq k \leq K$. By induction on $h$, we have that $\max_{\pi}W^{\pi}_{H}(r,P,\textbf{1}_{s_1^k})\leq V_{1}^k(s_1^k)$ for any $k$.

Define $\mathcal{J}:=\{(k,h):\exists h'\leq h, (s,a)\in \mathcal{S}\times\mathcal{A}, N_{h'}^k(s,a)>2N^k(s,a)+1 \}$ and $I_h^k=\mathbb{I}[(k,h)\notin \mathcal{J}]$. Let $\check{V}_{h}^k=V_h^k\cdot I_h^k$ . Let $h,s$ be fixed and $a=\pi^k_h(s)$. 
Using a similar argument in the proof of Lemma~\ref{lemma:add0}, . we have that
\begin{align}
&\check{V}_h^k(s)-P_{s,a}\check{V}_{h+1}^k\nonumber\\ & \leq r^k(s,a) I_h^k +\max_{p\in \mathcal{P}^k_{s,a}}pV_{h+1}^k \cdot ( I_h^k-I_{h+1}^k)+\max_{p\in \mathcal{P}^k_{s,a}}(p-P_{s,a})\check{V}_{h+1}^k 
\\ &\leq r^k(s,a) I_h^k + ( I_h^k-I_{h+1}^k)+ \sum_{s'}\left(\sqrt{\frac{2P_{s,a,s'}\iota }{ N^k(s,a)}} +\frac{\iota}{3N^k(s,a)}\right)|\check{V}_{h+1}^k(s')-P_{s,a}\check{V}_{h+1}^k|.\label{eq:exp0}
\end{align}

 By definition of $\pi^k$, the regret is bounded by 
\begin{align}
&\sum_{k=1}^{K} \left(\max_{\pi}W^{\pi}_{H}(r,P,\mu_1)-\sum_{h=1}^H r(s_h^k,a_h^k)\right)\nonumber
\\ & \leq \sum_{k=1}^{K} \left(\check{V}_1^k(s_1^k)-\sum_{h=1}^H r(s_h^k,a_h^k)\right)\nonumber
\\ &=\sum_{k=1}^K\left( \sum_{h=1}^H \left(r^k(s_h^k,a_h^k)I_h^k-r(s_h^k,a_h^k)+  (\max_{p\in \mathcal{P}^k_{s_h^k,a_h^k}}p-P_{s_h^k,a_h^k})\check{V}_{h+1}^k  \right)  + (P_{s_{h}^k,a_h^k}-\textbf{1}_{s_{h+1}^k})\check{V}_{h+1}^k\right)\nonumber 
\\& \quad \quad \quad \quad \quad \quad \quad \quad\quad \quad \quad \quad\quad \quad \quad \quad\quad \quad \quad \quad\quad \quad \quad \quad +\sum_{k=1}^K\sum_{h=1}^H \max_{p\in \mathcal{P}^k_{s,a}}p V_{h+1}^k(I_h^k-I_{h+1}^k)\nonumber 
\\ & \leq 10\sum_{k=1}^K\left( \sum_{h=1}^H \left( \sum_{s'}\left(\sqrt{\frac{P_{s_h^k,a_h^k,s'}\iota}{N^k(s_h^k,a_h^k)}}+\frac{\iota}{N^k(s_h^k,a_h^k)}\right)|\check{V}_{h+1}^k(s')-P_{s_h^k,a_h^k}\check{V}_{h+1}^k| \right) \right)  +\nonumber
\\ & \quad \quad \quad \quad \quad + \sum_{k=1}^K\sum_{h=1}^H (P_{s_{h}^k,a_h^k}-\textbf{1}_{s_{h+1}^k})\check{V}_{h+1}^k +\sum_{k=1}^K \mathbb{I}[\exists h, (k,h)\in \mathcal{J}].\label{eq:exp1}
\end{align}
Here \eqref{eq:exp1} is by the definition of $I_h^k$ and Lemma~\ref{lemma:con4}.

Let 
\begin{align}
& M_1 = 10\sum_{k=1}^K\ \sum_{h=1}^H        \left(\sum_{s'}\left(\sqrt{\frac{P_{s_h^k,a_h^k,s'}\iota}{N^k(s_h^k,a_h^k)}}+\frac{\iota}{N^k(s_h^k,a_h^k)}\right)|\check{V}_{h+1}^k(s')-P_{s_h^k,a_h^k}\check{V}_{h+1}^k|\right) \nonumber          
\\ & M_2 = \sum_{k=1}^K\sum_{h=1}^H (P_{s_{h}^k,a_h^k}-\textbf{1}_{s_{h+1}^k})\check{V}_{h+1}^k\nonumber 
\\ & M_3 = \sum_{k=1}^K \mathbb{I}[\exists h, (k,h)\in \mathcal{J}].\nonumber
\end{align}

The following lemma is crucial in bounding $M_1$ and $M_3$ and it shows the usefulness of stage 1. 
\begin{lemma}\label{lemma:bdm4} 
Define $L = \max_{(s,a)\notin \mathcal{O}^{k+1}}\sum_{k}\min\{\log(N^{K_1+1}(s,a)/N^{k}(s,a)),1\}$.
With probability  $1-2SA\delta$,
$M_3\leq  U:=SAL+O(S^8A^3K\iota/K_1)\leq O\left(\frac{S^8A^3K\iota}{K_1}\mathrm{polylog}(SAK)\right) $ and $L\leq O(\iota\mathrm{polylog}(SAK))$.
\end{lemma}
\begin{proof}[Proof of Lemma~\ref{lemma:bdm4}]
	Define $\mathcal{B}_1= \{k\in [K]: \exists h, (s_h^k,a_{h}^k)\in \mathcal{O}^{K_1+1}  \}$ and $\mathcal{B}_2(s,a)= \{k\in [K]: (s,a): N^{k+1}(s,a)-N^k(s,a)\geq K^2U(s,a)\}$. By Lemma~\ref{lemma:stage1} and \ref{lemma:ratiocon}, with probability $1-\delta$, we have that $|\mathcal{B}_1|\leq O\left( \frac{S^8A^3K\iota}{K_1}+\iota\right)$. By definition of $U(s,a)$ and Lemma~\ref{lemma:ratiocon}, with probability $1-SA\delta$, $|\mathcal{B}_2(s,a)|\leq (1/K^2+\iota)$ for any $(s,a)$. 
	By definition, we have that
	\begin{align}
		M_3  & =\sum_{k=1}^K \mathbb{I}[\exists h, (k,h)\in \mathcal{J}]\\& \leq |\mathcal{B}_1|+\sum_{k=1}^K\max_{(s,a)\notin \mathcal{O}^{K_1+1}}\min\{\log(N^{k+1}(s,a)/N^k(s,a)),1 \}.\nonumber
	\end{align}
	
	Let $(s,a)\notin \mathcal{O}^{K_1+1}$ be fixed.
	Suppose $\mathcal{B}_2(s,a)=\{k_1,k_2,...,\}$. Let $k_0=0$.  Then  
	\begin{align}
		& \sum_{k\notin \mathcal{B}_2(s,a)}\min\{\log(N^{k+1}(s,a)/N^k(s,a)),1 \}\nonumber
		\\ & =\sum_{i\geq 0}\sum_{k=k_i+1}^{k_{i+1}-1}\log(N^{k+1}(s,a)/N^k(s,a))\nonumber
		\\ & \leq \sum_{i\geq 0}\log\left(\frac{K^3U(s,a)+N^{k_{i}+1}}{N^{k_i+1}}\right) \nonumber
		\\& \leq |\mathcal{B}_2(s,a)|\log(K^3U(s,a)/N^1(s,a))\nonumber
		\\ & \leq O(\mathrm{polylog}(SAK)\iota).\nonumber
	\end{align}
	The proof is completed by taking sum over $SA$.
\end{proof}

Now we use this lemma to bound $M_1$.
We have that
\begin{align}
M_1 &  \leq  10\sum_{k=1}^K\ \sum_{h=1}^H        \left(\sqrt{\frac{S\sum_{s'}P_{s_h^k,a_h^k,s'}|\check{V}_{h+1}^k(s')-P_{s_h^k,a_h^k}\check{V}_{h+1}^k|^2\iota}{N^k(s_h^k,a_h^k)}}+\frac{SI_{h+1}^k\iota}{N^k(s_h^k,a_h^k)}\right) \nonumber
\\& \leq 10\sqrt{\sum_{k,h}\frac{I_{h+1}^k\iota}{N^k(s_h^k,a_h^k)}}\cdot\sqrt{S\sum_{k,h}\sum_{s'}P_{s_h^k,a_h^k,s'}|\check{V}_{h+1}^k(s')-P_{s_h^k,a_h^k}\check{V}_{h+1}^k|^2}+ 10S\sum_{k,h}\frac{I_{h+1}^k\iota}{N^k(s_h^k,a_h^k)}\nonumber 
\\ & \leq 10\sqrt{S^2AL\iota}\cdot \sqrt{\sum_{k,h}\sum_{s'}P_{s_h^k,a_h^k,s'}|\check{V}_{h+1}^k(s')-P_{s_h^k,a_h^k}\check{V}_{h+1}^k|^2}+40S^2AL\iota,\label{eq:exp2}
\end{align}
where the last line is by the fact that $I_{h+1}^k\leq I_h^k$ and $\sum_{k,h}\frac{1}{N^k(s_h^k,a_h^k)} \le SAL$.

Let 
\begin{align}
M_4 = \sum_{k,h}\sum_{s'}P_{s_h^k,a_h^k,s'}|\check{V}_{h+1}^k(s')-P_{s_h^k,a_h^k}\check{V}_{h+1}^k|^2.
\end{align}

By \eqref{eq:exp2} we have that 
\begin{align}
M_1\leq 10\sqrt{S^2ALM_4\iota}+ 40 S^2AL\iota\leq \frac{1}{4}M_4 + 140S^2AL\iota .\label{eq:exp3}
\end{align}

We continue with bounding $M_4$.

\begin{align}
&M_4 \nonumber \\ 
& = \sum_{k,h}\sum_{s'}P_{s_h^k,a_h^k,s'}|\check{V}_{h+1}^k(s')-P_{s_h^k,a_h^k}\check{V}_{h+1}^k|^2 \nonumber 
\\ &=\sum_{k,h}\left(P_{s_h^k,a_h^k}(\check{V}_{h+1}^k)^2-(P_{s_h^k,a_h^k}\check{V}_{h+1}^k)^2\right)\nonumber 
\\ &= \sum_{k,h}(P_{s_h^k,a_h^k}(\check{V}_{h+1}^k)^2-\check{V}_h^k(s_h^k)^2 )
+\sum_{k,h} \left( (\check{V}_h^k(s_h^k))^2-(P_{s_h^k,a_h^k}\check{V}_{h+1}^k)^2 \right) \nonumber 
\\ &\leq \sum_{k,h}(P_{s_h^k,a_h^k}(\check{V}_{h+1}^k)^2-\check{V}_h^k(s_h^k)^2 )+2\sum_{k,h}\max\{\check{V}_h^k(s_h^k)-P_{s_h^k,a_h^k}\check{V}_{h+1}^k, 0\}\nonumber
\\ &\leq \sum_{k,h}(P_{s_h^k,a_h^k}(\check{V}_{h+1}^k)^2-\check{V}_{h+1}^k(s_{h+1}^k)^2 )\nonumber
\\&\quad \quad  +2\sum_{k,h}\left(r^k(s_h^k,a_h^k) I_h^k + ( I_h^k-I_{h+1}^k)+ 10\sum_{s'}\left(\sqrt{\frac{P_{s_h^k,a_h^k,s'}\iota }{ N^k(s_h^k,a_h^k)}} +\frac{\iota}{N^k(s_h^k,a_h^k)}\right)|\check{V}_{h+1}^k(s')-P_{s_h^k,a_h^k}\check{V}_{h+1}^k|\right)\nonumber
\\ &\leq \sum_{k,h}(P_{s_h^k,a_h^k}(\check{V}_{h+1}^k)^2-\check{V}_{h+1}^k(s_{h+1}^k)^2 )+2\sum_{k,h} ( I_h^k-I_{h+1}^k)+ 2\sum_{k,h}r(s_h^k,a_h^k)+ 2M_1\nonumber
\\&\leq \sum_{k,h}(P_{s_h^k,a_h^k}(\check{V}_{h+1}^k)^2-\check{V}_{h+1}^k(s_{h+1}^k)^2 )+2\sum_{k,h}r(s_h^k,a_h^k)+2(M_1+M_3)\nonumber
\\  & \leq 2\sum_{k,h}(P_{s_h^k,a_h^k}(\check{V}_{h+1}^k)^2-\check{V}_{h+1}^k(s_{h+1}^k)^2 )+4K+ 20\sqrt{S^2AKL\iota}+60S^2AL\iota + 4U ,\label{eq:exp4}
\end{align}	
where $U=O\left(\frac{S^8A^3K\iota}{K_1}\mathrm{polylog}(SAK)\right)$ is an upper bound of $M_3$. 
Here \eqref{eq:exp4} is by \eqref{eq:exp3} and rearrangement.

Let $M_5=\sum_{k,h}(P_{s_h^k,a_h^k}(\check{V}_{h+1}^k)^2-\check{V}_{h+1}^k(s_{h+1}^k)^2 )$ and $M_6 =\sum_{k,h} \mathbb{V}(P_{s_h^k,a_h^k},(\check{V}_{h+1}^k)^2 ) $. By Lemma~\ref{lemma:varb}, we have that
\begin{align}
M_6 \leq 4\sum_{k,h} \mathbb{V}(P_{s_h^k,a_h^k},\check{V}_{h+1}^k) =4M_4.
\end{align}

By Lemma~\ref{lemma:self-norm}, we have that
\begin{align}
\mathbb{P}\left[ \exists i, M_5 \geq 10\cdot 2^i i\iota, M_6\leq 2^{2i}i\iota             \right]\leq \delta, \label{eq:lemma13_2}
\end{align}
which implies that
\begin{align}
\mathbb{P}\left[ \exists i, M_5 \geq 10\cdot 2^i i\iota, M_4\leq 2^{2i-2}i\iota             \right]\leq \delta.
\end{align}
Therefore, with probability $1-\delta$, for any $i\geq 1$, it either holds $M_5<10\cdot 2^i i\iota$ or $M_4>2^{2i-2}i\iota$.
By \eqref{eq:exp4}, we have that
\begin{align}
M_4 \leq M_5 + 4K + 40\sqrt{S^2AKL\iota}+60S^2AL\iota +4U\leq M_5 + 8K + 100S^2AL\iota+4U.\label{eq:exp5}
\end{align}
Suppose $M_5 \geq C \geq 8K + 100S^2AL\iota+4U$, then we have that
\begin{align}
M_4 \geq \frac{M^2_{6}}{800\iota \log_2(C)}\geq \frac{C^2}{800\iota \log_2(C)}\geq 3C.
\end{align}
By \eqref{eq:exp5}, we have that $M_5\geq M_4 -( 8K + 100S^2AL\iota+4U)\geq M_4-C\geq 2C$. In this way, $M_5 $ is infinite, which leads to contradiction. Therefore, with probability $1-\delta$, $M_5 <8K + 100S^2AL\iota +4U$, and it follows that 
\begin{align}
M_4 \leq 16K + 100S^2AL\iota+4U= O\left(K+\frac{S^8A^3K\iota}{K_1}\mathrm{polylog}(SAK)\right).\label{eq:exp7}
\end{align}

Next, we bound $M_2$.
Using Lemma~\ref{freedman}, we have that
\begin{align}
 & \mathbb{P}\left[   M_2\geq 10\sqrt{16K + 200S^2AL\iota+4U},          \right]  \nonumber 
\\ & \leq  \mathbb{P}\left[   M_2\geq 10\sqrt{16K + 200S^2AL\iota+4U},  M_4\leq    16K + 200S^2AL\iota+4U    \right] +\delta\nonumber 
\\ & \leq  2\delta.\nonumber
\end{align}

Finally, putting all together, with probability $1-6SA\delta$, $$\sum_{k=1}^{K} \left(\max_{\pi}W^{\pi}_{H}(r,P,\mu_1)-\sum_{h=1}^H r(s_h^k,a_h^k)\right)\leq O\left(\mathrm{polylog}(SAK)\left(\sqrt{S^2A\iota^2}+\frac{S^8A^3K\iota}{K_1}\right)\right).$$
The proof is completed.

\section{ Missing Proofs about Collecting Initial Samples }\label{app:pfthsc} 

\subsection{Approximated Reference Model}

Define $\iota = \log(1/\delta)$. Define $\bar\{\mathcal{S}\}=\mathcal{S}\cup\{z,z'\}$.
Define $\mathcal{G}_1$ be the event where
\begin{align}
    |P_{s,a,s'}-\frac{N^k(s,a,s')}{N^k(s,a)}|\leq \sqrt{\frac{2P_{s,a,s'}\iota}{N^k(s,a)}}+\frac{\iota}{3N^k(s,a)}
\end{align}
holds for any proper $k\in [K],(s,a)\in \mathcal{S}\times\mathcal{A}$. By Bennet's inequality (see Lemma~\ref{lemma:bennet}), we have that $\mathbb{P}[\mathcal{G}_1]\geq 1-S^2AK\delta$. We continue the analysis assuming $\mathcal{G}_1$ holds.

\begin{lemma}\label{lemma:con1}  Let $k$ be fixed.
	With probability $1-S^2A\delta$, $e^{-1/S}P_{s,a,s'}^{\mathrm{ref},k}\leq \bar{P}_{s,a,s'}^{\mathrm{cut},k}\leq e^{1/S}P_{s,a,s'}^{\mathrm{ref},k}$ for any $(s,a,s')\in \bar{\mathcal{S}}\times \mathcal{A}\times \bar{\mathcal{S}}$.
\end{lemma}
\begin{proof}
For each $(s,a)\in \mathcal{U}^k$, $P^{\mathrm{ref},k}_{s,a,z}= \bar{P}^{\mathrm{cut},k}_{s,a,z}=1$. For any $(s,a,s')\notin \mathcal{K}^k$, $P^{\mathrm{ref},k}_{s,a,s'}=\bar{P}^{\mathrm{cut},k}_{s,a,s'}=0$.

For $s=z,z'$, we have that $P^{\mathrm{ref},k}_{z,a}=\bar{P}^{\mathrm{cut},k}_{z,a}=^{\mathrm{ref},k}_{z',a}=\bar{P}^{\mathrm{cut},k}_{z',a}=\textbf{1}_{z'}$ for any $a$.
 
For $(s,a,s')\in \mathcal{K}^k$,  by the definition of $\mathcal{G}$, and noting that $n^k(s,a,s')\geq 256S^2\iota$, $ 1/S\leq 1/2$, we have that
\begin{align}
  &  |N^k(s,a)P_{s,a,s'}-N^k(s,a,s')| \nonumber
  \\ & \leq \sqrt{2P_{s,a,s'}N^k(s,a)\iota}+\frac{\iota}{3}\nonumber 
  \\ & \leq \frac{1}{4S}N^k(s,a)P_{s,a,s'}+ 8S\iota+\frac{\iota}{3}\nonumber 
  \\ & \leq \frac{1}{4S}N^k(s,a)P_{s,a,s'}+\frac{1}{4S}N^k(s,a,s'),\nonumber
\end{align}
which implies that 
\begin{align}
    N^k(s,a)e^{-\frac{1}{2S}}P_{s,a,s'}\leq N^k(s,a,s')\leq N^k(s,a)e^{\frac{1}{2S}}P_{s,a,s'}.\nonumber
\end{align}
Taking sum over $s'$ such that $(s,a,s')\in \mathcal{K}$, we have that
\begin{align}
     N^k(s,a)e^{-\frac{1}{2S}}\sum_{s':(s,a,s')\in \mathcal{K}}P_{s,a,s'}\leq \sum_{(s,a,s')\in \mathcal{K}}n^k(s,a,s')\leq N^k(s,a)e^{\frac{1}{2S}}\sum_{(s,a,s')\in \mathcal{K}}P_{s,a,s'}.\nonumber
\end{align}
Therefore, it holds that
\begin{align}
   \frac{ N^k(s,a)e^{-\frac{1}{2S}}P_{s,a,s'}}{N^k(s,a)e^{\frac{1}{2S}}\sum_{(s,a,s')\in \mathcal{K}}P_{s,a,s'}}\leq \frac{N^k(s,a,s')}{\sum_{(s,a,s')\in \mathcal{K}}N^k(s,a,s')} \leq  \frac{ N^k(s,a)e^{\frac{1}{2S}}P_{s,a,s'}}{N^k(s,a)e^{-\frac{1}{2S}}\sum_{(s,a,s')\in \mathcal{K}}P_{s,a,s'}}.\nonumber
\end{align}
The proof is completed.
\end{proof}

\subsubsection{Proof of Lemma~\ref{lemma:sr}}\label{app:pfsr}

\textbf{lemma}[Restatement of Lemma~\ref{lemma:sr}]\emph{
With  probability $1-10S^3A^2K\delta$, it holds that}
\begin{align}
N^{\tilde{k}+1}(\tilde{s},\tilde{a})\geq 2\max_{\pi\in \Pi_{\mathrm{sta}} }W_{d_2}^{\pi}(\textbf{1}_{\tilde{s},\tilde{a}},P,\textbf{1}_{\tilde{s}})\log(1/\delta)
\end{align}
\emph{
for any $(\tilde{s},\tilde{a})\in \mathcal{O}^{\tilde{k}+1}/\mathcal{O}^{\tilde{k}}$ and any $1\leq \tilde{k}\leq K_1$.
}

\begin{proof} 

Let $\mathrm{Trigger}^k$ denote the value of $\mathrm{Trigger}$ in the end of the $k$-th round for $k\in [K_1]$. 
Fix $\tilde{k}$ and $(\tilde{s},\tilde{a})\in \mathcal{O}^{\tilde{k}+1}/\mathcal{O}^{\tilde{k}}$.

Recall that $\mathcal{J}: = \{k \in [K_1]|  \exists (h,a), (s_{h+1}^k,a)\in \mathcal{O}^k \}$. 
Define $\mathcal{C}:=\{k\in \mathcal{J}, k \leq \tilde{k} | \mathrm{Trigger}^k=\mathrm{FALSE}, (s_1^{*,k} ,a_{1}^{*,k})=(\tilde{s},\tilde{a}) \}$.

By Algorithm~\ref{alg:main}, we have that $|\mathcal{C}|\geq 400\log(1/\delta)$.
Let $k\in \mathcal{C}$ be fixed. 
    We first show that
    \begin{align}
        \max_{\pi\in \Pi_{\mathrm{sta}},\pi(\tilde{s})=\tilde{a}}W^{\pi}_{d_2}(\textbf{1}_{z},\bar{P}^k,\textbf{1}_{\tilde{s}}) \leq \frac{1}{10}.\label{eqv4_1}
    \end{align}
    Recall $\gamma = 1-\frac{1}{d_2}$. By Lemma~\ref{lemma:eff}
    \begin{align}
& u^k(s)=\max_{\pi\in \Pi_{\mathrm{sta}}}X_{\gamma}^{\pi}(\{s\}, P^{\mathrm{ref},k},\textbf{1}_{\tilde{s}})\geq \frac{1}{3} \max_{\pi\in\Pi_{\mathrm{sta},\pi(\tilde{a})=\tilde{a}}} X^{\pi}_{d_2}(\{s\},P^{\mathrm{ref},k}, \textbf{1}_{\tilde{s}})\label{eq:w1x}
    \\ & v^k(s,a)=\max_{\pi\in \Pi_{\mathrm{sta}}}W_{\gamma}^{\pi}(\textbf{1}_{s,a},P^{\mathrm{ref},k},\textbf{1}_{s})\geq \frac{1}{3} \max_{\pi\in\Pi_{\mathrm{sta},\pi(\tilde{s})=\tilde{a}}} W^{\pi}_{d_2}(\textbf{1}_{s,a},P^{\mathrm{ref},k}, \textbf{1}_{s}).\label{eq:w2}
    \end{align}
    
    For each $s\in \mathcal{S}$, if there exists $a$ and $s'$ such that $(s,a,s')\notin \mathcal{K}^k$, then 
    we either have $u^k(s)<\frac{1}{1200S}$ or $N^k(s,a) > 100S^2AN_0 u^k(s) v^k(s,a)$.
    Denote $\mathcal{S}_1^k:=\{s:u^k(s)\leq \frac{1}{1200S}\}$.
    In the first case, we have that
    \begin{align}  \max_{\pi\in \Pi_{\mathrm{sta}},\pi(\tilde{s})=\tilde{a}} X^{\pi}_{d_2}(\{s\},\bar{P}^k, \textbf{1}_{\tilde{s}})  & \leq \max_{\pi\in \Pi_{\mathrm{sta}},\pi(\tilde{s})=\tilde{a}} X^{\pi}_{d_2}(\{s\},\bar{P}^{\mathrm{cut},k}, \textbf{1}_{\tilde{s}}) \label{eq:ddl1}
     \\ & \leq 3 \max_{\pi\in \Pi_{\mathrm{sta}},\pi(\tilde{s})=\tilde{a}} X^{\pi}_{d_2}(\{s\},P^{\mathrm{ref},k}, \textbf{1}_{\tilde{s}}) \label{eq:ddl2}
     \\&  \leq  12u^k(s)\label{eq:ddl3}
     \\ & \leq \frac{1}{100S}.\nonumber
    \end{align}
    Here \eqref{eq:ddl1} holds by Lemma~\ref{lemma:v2} and the fact that $\bar{P}^{\mathrm{cut},k}=\mathrm{cut}(\bar{P}^k)$, \eqref{eq:ddl2} holds by  Lemma~\ref{lemma:con1} and  Lemma~\ref{lemma:approx}.
  
    In the second case,  by definition of $\mathcal{G}$, we have that
    \begin{align}\bar{P}^{k}_{s,a,z} = \sum_{s':(s,a,s')\notin \mathcal{K}}P_{s,a,s'}\leq \frac{2SN_0}{N^k(s,a)} \leq \frac{1}{810SA u^k(s,a)v^k(s,a)}.\label{eq:w3}
    \end{align}
    For any stationary policy $\pi$ such that $\pi(\tilde{s})=\tilde{a}$, noting that $z$ is a transient state, the probability of reaching $z$ from some state $s$ is bounded by the probability of reaching $s$. That is, for any state $s$, it holds that $\sum_{a}W_{d_2}^{\pi}(\textbf{1}_{s,a},\bar{P}^k,\textbf{1}_{\tilde{s}})\bar{P}_{s,a,z}^k\leq X_{d_2}^{\pi}(\{s\},\bar{P}^k,\textbf{1}_{\tilde{s}}) $. As a result, we have that
    \begin{align}
      W_{d_2}^{\pi}(\textbf{1}_{z},\bar{P}^k,\textbf{1}_{\tilde{s}})
      & =\sum_{s,a}W_{d_2}^{\pi}(\textbf{1}_{s,a},\bar{P}^k,\textbf{1}_{\tilde{s}})\bar{P}_{s,a,z}^k
      \nonumber
      \\ & = \sum_{s\in \mathcal{S}_1^k}\sum_{a}W_{d_2}^{\pi}(\textbf{1}_{s,a},\bar{P}^k,\textbf{1}_{\tilde{s}})\bar{P}_{s,a,z}^k  +\sum_{s\notin \mathcal{S}_1^k}\sum_{a}W_{d_2}^{\pi}(\textbf{1}_{s,a},\bar{P}^k,\textbf{1}_{\tilde{s}})\bar{P}_{s,a,z}^k  \nonumber \\ & \leq \sum_{s\in \mathcal{S}_1^k}X_{d_2}^{\pi}(\{s\},\bar{P}^k,\textbf{1}_{\tilde{s}}) +\sum_{s\notin \mathcal{S}_1^k}\sum_{a}W_{d_2}^{\pi}(\textbf{1}_{s,a},\bar{P}^k,\textbf{1}_{\tilde{s}})\bar{P}_{s,a,z}^k .\label{eq:cr2}
    \end{align}
    Continuing the computation, we obtain that
    \begin{align}
       &  W_{d_2}^{\pi}(\textbf{1}_{z},\bar{P}^k,\textbf{1}_{\tilde{s}}) \nonumber
       \\ & \leq  3\sum_{s\in \mathcal{S}_1^k}u^k(s)+ \sum_{s\notin \mathcal{S}_1^k}\sum_{a}W_{d_2}^{\pi}(\textbf{1}_{s,a},\check{P}^k,\textbf{1}_{\tilde{s}})\bar{P}_{s,a,z}^k \label{eq:cr2.1}
       \\ & \leq \frac{1}{400}+ \sum_{s\notin \mathcal{S}_1^k}\sum_{a}W_{d_2}^{\pi}(\textbf{1}_{s,a},\bar{P}^{k},\textbf{1}_{\tilde{s}})\bar{P}_{s,a,z}^k \nonumber
       \\ & \leq \frac{1}{400}+\sum_{s,a}9 u^k(s)\cdot v^k(s,a)\cdot \frac{1}{810SAu^k(s)v^k(s,a)}\label{eq:ddl7}
       \\ & \leq \frac{1}{10}.\nonumber
    \end{align}

Here  \eqref{eq:cr2.1} is by \eqref{eq:cr2} and \eqref{eq:w2}, and 
\eqref{eq:ddl7} holds because \eqref{eq:w3} and the fact below
\begin{align}
& W_{d_2}^{\pi}(\textbf{1}_{s,a},\bar{P}^{k},\textbf{1}_{\tilde{s}})\nonumber \\&\leq \max_{\pi\in\Pi_{\mathrm{sta}},,\pi(\tilde{s})=\tilde{a}} X^{\pi}_{d_2}(\{s\},\bar{P}^k, \textbf{1}_{\tilde{s}}) \cdot \max_{\pi\in\Pi_{\mathrm{sta}},\pi(\tilde{s})=\tilde{a}} W^{\pi}_{d_2}(\textbf{1}_{s,a},\bar{P}^k, \textbf{1}_{s})\nonumber 
\\ & \leq 9u^k(s)\cdot 9v^k(s,a).\nonumber
\end{align}

Define $E^k_1$ to be the event $z$ is visited under $\bar{P}^k$. So $E_1^k$ is corresponding to visiting $(\mathcal{K}^k)^C$ under $P$. Let be $X^k$ be the count of $(\tilde{s},\tilde{a})$ in the first $d_2$ steps, i.e., $X^k = \sum_{i=1}^d \mathbb{I}[(s_i,a_i)=(\tilde{s},\tilde{a})]$.
Then we have that for any stationary policy $\pi$ such that $\pi(\tilde{s})=\tilde{a}$,
\begin{align}
   &  W_{d_2}^{\pi}(\textbf{1}_{\tilde{s},\tilde{a}},\bar{P}^k,\textbf{1}_{\tilde{s}}) = \mathbb{E}_{\bar{P}^k,\pi}[X^k|\mu_1 = \textbf{1}_{\tilde{s}}]\geq \mathbb{E}_{\bar{P}^k,\pi}[X^k\mathbb{I}[(E_{1}^k)^{C}]|\mu_1 = \textbf{1}_{\tilde{s}}].
\end{align}
By definition of $\bar{P}^k$ and $E_1^k$, we obtain that 
\begin{align}
    \mathbb{E}_{\bar{P}^k,\pi}[X^k\mathbb{I}[(E_{1}^k)^{C}]|\mu_1 = \textbf{1}_{\tilde{s}}]= \mathbb{E}_{P,\pi}[X^k\mathbb{I}[(E_{1}^k)^{C}]|\mu_1 = \textbf{1}_{\tilde{s}}],\label{eq:cr2222}
\end{align}
which implies  that  $W_{d_2}^{\pi}(\textbf{1}_{\tilde{s},\tilde{a}},\bar{P}^k,\textbf{1}_{\tilde{s}})  \geq \mathbb{E}_{P,\pi}[X^k\mathbb{I}[(E_{1}^k)^{C}]|\mu_1 = \textbf{1}_{\tilde{s}}]$.
 
On the other hand, by Lemma~\ref{lemma:li4}, we have that
\begin{align}
    \mathrm{Pr}\left[X^k\geq \frac{1}{2}\mathbb{E}_{P,\pi}[X^k|\mu_1 = \textbf{1}_{\tilde{s}}] \text{ and } (E_{1}^k)^{C}\right]\geq \frac{1}{2}-\frac{1}{10}.\label{eq:adddd1}
\end{align}
As a result, it holds that
\begin{align}
   & W_{d_2}^{\pi}(\textbf{1}_{\tilde{s},\tilde{a}},\bar{P}^k,\textbf{1}_{\tilde{s}}) \nonumber\\
   \geq &\mathbb{E}_{\pi,P}[X^k\mathbb{I}[(E_{1}^k)^{C}]|\mu_1 = \textbf{1}_{\tilde{s}}] \nonumber \\
   \geq & \frac{1}{2}(1-\frac{1}{10})\mathbb{E}_{\pi,P}[X^k|\mu_1 = \textbf{1}_{\tilde{s}}] \nonumber\\
   = & \frac{9}{20} W_{d_2}^{\pi}(\textbf{1}_{\tilde{s},\tilde{a}},P,\textbf{1}_{\tilde{s}}) \label{eq:220}
\end{align}

By \eqref{eqv4_1} and Lemma~\ref{lemma:v1}, we have that for any stationary policy $\pi$ such that $\pi(\tilde{s})=\tilde{a}$,  
\begin{align}
    W_{d_2}^{\pi}(\textbf{1}_{\tilde{s},\tilde{a}},\bar{P}^k,\textbf{1}_{\tilde{s}})\leq W^{\pi}_{d_2}(\textbf{1}_{\tilde{s},\tilde{a}},\bar{P}^{\mathrm{cut},k},\textbf{1}_{\tilde{s}} )\leq 2 W_{d_2}^{\pi}(\textbf{1}_{\tilde{s},\tilde{a}},\bar{P}^k,\textbf{1}_{\tilde{s}}).\label{eq:221}
\end{align}
    By Lemma~\ref{lemma:con1} and \ref{lemma:approx}, we further have that
    \begin{align}
         W_{d_2}^{\pi}(\textbf{1}_{\tilde{s},\tilde{a}},\bar{P}^k,\textbf{1}_{\tilde{s}})\leq 3W^{\pi}_{d_2}(\textbf{1}_{\tilde{s},\tilde{a}},P^{\mathrm{ref},k},\textbf{1}_{\tilde{s}} )\leq 18 W_{d_2}^{\pi}(\textbf{1}_{\tilde{s},\tilde{a}},\bar{P}^k,\textbf{1}_{\tilde{s}}).\label{eq:222}
    \end{align}

Combining \eqref{eq:220} with \eqref{eq:222}, we learn that
\begin{align}
    W_{d_2}^{\pi}(\textbf{1}_{\tilde{s},\tilde{a}},P,\textbf{1}_{\tilde{s}})\leq 9W^{\pi}_{d_2}(\textbf{1}_{\tilde{s},\tilde{a}},P^{\mathrm{ref},k},\textbf{1}_{\tilde{s}} )\leq 54W_{d_2}^{\pi}(\textbf{1}_{\tilde{s},\tilde{a}},P,\textbf{1}_{\tilde{s} })\label{eq:223}
\end{align}
for any stationary policy $\pi$ such that $\pi(\tilde{s})=\tilde{a}$.

Recall that $\gamma =1-\frac{1}{d_2} $.
By running the policy \[\pi_2^k := \arg\max_{\pi\in \Pi_{\mathrm{sta}} , \pi(\tilde{s})=\tilde{a}}W^{\pi}_{\gamma}(\textbf{1}_{\tilde{s},\tilde{a}},P^{\mathrm{ref},k},\textbf{1}_{\tilde{s}} ),\]
using \eqref{eq:adddd1} and \eqref{eq:223}, with probability $1/2$, 
\begin{align}
X^k &\geq  \frac{1}{4}W^{\pi^k_2}_{d_2}(\textbf{1}_{\tilde{s},\tilde{a}},P^{\mathrm{ref},k},\textbf{1}_{\tilde{s}} )\nonumber 
\\ & \geq \frac{1}{12}W^{\pi_2}_{\gamma}(\textbf{1}_{\tilde{s},\tilde{a}},P^{\mathrm{ref},k},\textbf{1}_{\tilde{s} } )\label{eq:ddl4}
\\ & =  \frac{1}{12}\max_{\pi\in \Pi_{\mathrm{sta}},\pi(\tilde{s})=\tilde{a} }\frac{1}{12}W^{\pi}_{\gamma}(\textbf{1}_{\tilde{s},\tilde{a}},P^{\mathrm{ref},k},\textbf{1}_{\tilde{s} } )\nonumber 
\\ & \geq \frac{1}{36}\max_{\pi\in \Pi_{\mathrm{sta}},\pi(\tilde{s})=\tilde{a} }\frac{1}{12}W^{\pi}_{\gamma}(\textbf{1}_{\tilde{s},\tilde{a} },P^{\mathrm{ref},k},\textbf{1}_{\tilde{s} } )\label{eq:ddl5}
\\ & \geq 
\frac{1}{108}\max_{\pi\in \Pi_{\mathrm{sta}} , \pi(\tilde{s})=\tilde{a}}  W_{d_2}^{\pi}(\textbf{1}_{\tilde{s},\tilde{a}},P,\textbf{1}_{\tilde{s} })\label{eq:ddl6}
\end{align}
samples of $(\tilde{s},\tilde{a})$ in the $k$-th episode. Here \eqref{eq:ddl4} and \eqref{eq:ddl5} are by Lemma~\ref{lemma:eff}, and \eqref{eq:ddl6} holds by Lemma~\ref{lemma:con1} and Lemma~\ref{lemma:approx}.

By Lemma~\ref{lemma:ratiocon}, with probability $1-\delta$ it holds that
\begin{align}
N^{\tilde{k}+1}(\tilde{s},\tilde{a}) & \geq  \sum_{k\in \mathcal{C}}X^k \nonumber
\\ & \geq 2\max_{\pi\in \Pi_{\mathrm{sta}} , \pi(s)=a} W_{d_2}^{\pi}(\textbf{1}_{s,a},P,\textbf{1}_{s})\log(1/\delta)\nonumber
\\ & =2\max_{\pi\in \Pi_{\mathrm{sta}}} W_{d_2}^{\pi}(\textbf{1}_{s,a},P,\textbf{1}_{s})\log(1/\delta).\end{align}
The  proof is completed.

\end{proof}

\subsubsection{Statement and Proof of Lemma~\ref{lemma:numc}}

\begin{lemma}\label{lemma:numc}
With probability $1-3S^3A^2K\delta$, it holds that $|\mathcal{J}|\leq O(\mathrm{polylog}(SAK)S^7A^3\log(1/\delta))$.
\end{lemma}

\begin{proof}
For fixed $(s,a)$, we define $\mathrm{True}(s,a)=\{ k\in \mathcal{J}|\mathrm{Trigger}^k \text{ is set to be } \mathrm{True} \text{ with respect to } (s,a) \}$.

Now we analyze the size of $\mathrm{True}(s,a)$. Let $\mathcal{I} = \{k :\mathcal{K}^{k}\neq \mathcal{K}^{k-1}\}$. Since $\mathcal{K}^{1}\subset \mathcal{K}^{2} \subset\ldots\mathcal{K}^{k}\subset \ldots$, and the $|\mathcal{K}^k|\leq S^2A$ for any $k$, we have that $|\mathcal{I}|\leq S^2A$. Suppose $\mathcal{I} = \{k_1,k_2,\ldots,k_{|\mathcal{I}|}\}$.

 Then we have that
\begin{align}
   | \mathrm{True}(s,a)| = \sum_{i=1}^{|\mathcal{I}|} |\mathrm{True}(s,a)\cap[k_{i-1},k_{i}-1 ] |, \label{eq:v44}
\end{align}
where $k_{0}$ is defined as $1$.   

Then we have the lemma to bound  $| \mathrm{True}(s,a)|$.

\begin{lemma}\label{lemma:v3}
    For any $1\leq i \leq |\mathcal{I}|$, with probability $1-3S^3A^2\delta$, it holds that  $$\sum_{(s,a)}\sum_{1\leq i \leq |\mathcal{I}|}|\mathrm{True}(s,a)\cap[k_{i-1},k_{i}-1 ] |\leq  480S(9600S^4A^3N_0 +5S^3A^2\log(1/\delta)+2S^2AN_0).$$
\end{lemma}

By Lemma~\ref{lemma:v3}, we obtain that  $$|\mathcal{J}|\leq \sum_{1\leq i \leq |\mathcal{I}|}\sum_{s,a}|\mathrm{True}(s,a)\cap[k_{i-1},k_{i}-1 ] | + \sum_{k\in [\mathcal{J}]}\mathbb{I}\left[\mathrm{Trigger}^k =\mathrm{FALSE}\right]  = O(\mathrm{polylog}(SAK)S^7A^3\log(1/\delta)).$$
The proof is completed.

\end{proof}

\begin{proof}[Proof of Lemma~\ref{lemma:v3}]
    Let $i$  and $(s,a)$ be fixed. Let $\Lambda_{1,i}(s,a) =\mathrm{True}(s,a)\cap[k_{i-1},k_{i}-1 ]  $.  Let $\mathcal{K} =\mathcal{K}^k$ for some $k\in \Lambda_{1,i}(s,a)$. The definition is proper since $\mathcal{K}^k$ is the same for any $k\in  \Lambda_{1,i}(s,a)$. 
    
    Recall that $\gamma = 1-\frac{1}{d_2}$ and 
 \begin{align}
& u^k(s)=\max_{\pi\in \Pi_{\mathrm{sta}}}X_{\gamma}^{\pi}(\{s\}, P^{\mathrm{ref},k},\textbf{1}_{\tilde{s}})\geq \frac{1}{3} \max_{\pi\in\Pi_{\mathrm{sta},\pi(s_{1}^k{*,k})=a_1^{*,k}}} X^{\pi}_{d_2}(\{s\},P^{\mathrm{ref},k}, \textbf{1}_{s_{1}^{*,k}})\label{eq:w6}
    \\ & v^k(s,a)=\max_{\pi\in \Pi_{\mathrm{sta}}}W_{\gamma}^{\pi}(\textbf{1}_{s,a},P^{\mathrm{ref},k},\textbf{1}_{s})\geq \frac{1}{3} \max_{\pi\in\Pi_{\mathrm{sta},\pi(s_{1}^{*,k})=a_1^{*,k}}} W^{\pi}_{d_2}(\textbf{1}_{s,a},P^{\mathrm{ref},k}, \textbf{1}_{s}).\label{eq:w7}
    \end{align}
       Recall that $\pi_1^k,\pi_2^k\in \Pi_{\mathrm{sta}}$ are such that $\pi_1^k(s_{1}^{*,k})=\pi_2^k(s_1^{*,k})=a_1^{*,k}$ and
       \begin{align}
        &   X^{\pi_1^k}_{\gamma}(\{s\},P^{\mathrm{ref},k},\textbf{1}_{s_{1}^{*,k}}) =u^k(s)\nonumber 
        \\ & W^{\pi_2^k}_{\gamma}(\textbf{1}_{s,a},P^{\mathrm{ref},k}, \textbf{1}_{s})=v^k(s,a).\label{eq:xxxx1}
       \end{align}
    By Lemma~\ref{lemma:eff}, and recalling $d_1 = d-d_2\geq 10S\log(S)d_2$, we also have that
    \begin{align}
    &       X^{\pi_1^k}_{d_1}(\{s\},P^{\mathrm{ref},k},\textbf{1}_{s_{1}^{*,k}})\geq \frac{1}{10}X^{\pi_1^k}_{\gamma}(\{s\},P^{\mathrm{ref},k},\textbf{1}_{s_{1}^{*,k}})= \frac{1}{10}u^k(s)\label{eq:ddl8}
         \\ & W^{\pi_2^k}_{d_2}(\textbf{1}_{s,a},P^{\mathrm{ref},k}, \textbf{1}_{s})\geq \frac{1}{3}v^k(s,a).\label{eq:ddl9}
    \end{align}
   By definition, for any $k\in \Lambda_{1,i}(s,a)$, we have that $u^k(s) \geq \frac{1}{1200S}$ and $N^k(s,a)\leq 300SAN_0u^k(s)v^k(s,a).$
    By Lemma~\ref{lemma:approx} 
    we have that
    \begin{align}
        X^{\pi_1^k}_{d_1}(\{s\},\bar{P}^{\mathrm{cut},k},\textbf{1}_{s_{1}^{*,k}})  \geq \frac{1}{3} X^{\pi_1^k}_{d_1}(\{s\},P^{\mathrm{ref},k},\textbf{1}_{s_{1}^{*,k}})\geq \frac{1}{30}u^k(s)
    \end{align}
    By Lemma~\ref{lemma:v2}, Lemma~\ref{lemma:con1} and Lemma~\ref{lemma:approx}, we further have that
    \begin{align}
        &    X^{\pi_1^k}_{d_1}(\{s\},\bar{P}^{k},\textbf{1}_{s_{1}^{*,k}}) \nonumber
        \\ & \geq  X^{\pi_1^k}_{d_1}(\{s\},\bar{P}^{\mathrm{cut},k},\textbf{1}_{s_{1}^{*,k}}) -W^{\pi_1^k}_{d_1}(\textbf{1}_{z},\bar{P}^{k},\textbf{1}_{s_{1}^{*,k}})  \nonumber
        \\ & \geq \frac{1}{30}u^k(s) - W^{\pi_1^k}_{d_1}(\textbf{1}_{z},\bar{P}^{k},\textbf{1}_{s_{1}^{*,k}}) .\nonumber
    \end{align}
    By rearranging the inequality, we have that
    \begin{align}
       X^{\pi_1^k}_{d_1}(\{s\},\bar{P}^{k},\textbf{1}_{s_{1}^{*,k}}) +  W^{\pi_1^k}_{d_1}(\textbf{1}_{z},\bar{P}^{k},\textbf{1}_{s_{1}^{*,k}}) \geq X^{\pi_1^k}_{d_1}(\{s\},\bar{P}^{\mathrm{cut},k},\textbf{1}_{s_{1}^{*,k}}) \geq \frac{1}{30}u^k(s)\geq \frac{1}{300S}.\label{eq:key5}
    \end{align}
    Let $E_{2}^k$ be reaching $s$ without visiting $\mathcal{K}^{C}$ in the $k$-th episode and $E_3^k$ be reaching $\mathcal{K}^{C}$ in the first $d$ steps in the $k$-th episode. Then we have that
    \begin{align}
        \mathrm{Pr}_{P,\pi_1^k}[E_2^k\cup E_3^k] &=     X^{\pi_1^k}_{d-1}(\{s\},\bar{P}^{k},\textbf{1}_{s_{1}^{*,k}}) +  W^{\pi_1^k}_{d_1}(\textbf{1}_{z},\bar{P}^{k},\textbf{1}_{s_{1}^{*,k}}) \\
    &    \geq X^{\pi_1^k}_{d_1}(\{s\},\bar{P}^{\mathrm{cut},k},\textbf{1}_{s_{1}^{*,k}}) \geq \frac{1}{30}u^k(s)\geq \frac{1}{3600S}.
    \end{align}
Therefore, if $|\Lambda_{1,i}(s,a) |\geq 28800S\log(1/\delta)$, by Lemma~\ref{lemma:ratiocon}, with probability $1-\delta$, it holds that
\begin{align}
     \sum_{k\in \Lambda_{1,i}(s,a) }\mathbb{I}[E_2^k]+\mathbb{I}[E_3^k]  \geq \frac{1}{14400S} |\Lambda_{1,i}(s,a)|-\log(1/\delta).  \label{eq:120}
  \end{align}

Define $\Lambda_{2,i}(s,a) = \{k\in \Lambda_{1,i}(s,a): \mathbb{I}[E_{2}^k] = 1 \}$. Let $k\in \Lambda_{2,i}(s,a)$ be fixed.
Also recall that $\pi_2^k(s) = a$ and $W_{\gamma}^{\pi_2^k}(\textbf{1}_{s,a},P^{\mathrm{ref},k},\textbf{1}_{s}) = v^k(s,a)$. By Lemma~\ref{lemma:approx} and \eqref{eq:ddl9}, we have that
\begin{align}
    W_{d_2}^{\pi_2^k}(\textbf{1}_{s,a},\bar{P}^{\mathrm{cut},k},\textbf{1}_{s}) \geq \frac{1}{9}v^k(s,a).\nonumber
\end{align}
 
 By Lemma~\ref{lemma:v1}, we have that
 \begin{align}
     W_{d_2}^{\pi_2^k}(\textbf{1}_{s,a},\bar{P}^k,\textbf{1}_{s}) \geq & (1-  W^{\pi_2^k}_{d_2}(\textbf{1}_z,\bar{P}^k,\textbf{1}_{s})) \cdot  W_{d_2}^{\pi_2^k}(\textbf{1}_{s,a},\bar{P}^{\mathrm{cut},k},,\textbf{1}_{s})\\
     \geq& \frac{1}{3}(1-  W^{\pi_2^k}_{d_2}(\textbf{1}_z,\bar{P}^k,\textbf{1}_{s}))\cdot v^k(s,a).
 \end{align}
  
  Let $Z^k$ be the number of samples of $(s,a)$ collected in the $d_2$ steps following $\pi_2^k$. Noting that $\mathbb{E}[Z_k]= W^{\pi_2^k}_{d_2}(\textbf{1}_{s,a},P,\textbf{1}_{s}) \geq  W^{\pi_2^k}_{d}(\textbf{1}_{s,a},\bar{P}^k,\textbf{1}_{s})$,  
  by Lemma~\ref{lemma:li4}, we have that
  \begin{align}
      \mathrm{Pr}\left[Z^k \geq 
     \frac{1}{12}(1-  W^{\pi_2^k}_{d_2}(\textbf{1}_z,\bar{P}^k,\textbf{1}_{s}))\cdot v^k(s,a)\right]  \geq \frac{1}{2}.\label{eq:445}
  \end{align}
Note that for any $k\in \Lambda_{1,i}(s,a)$, $\bar{P}^{\mathrm{cut},k}$ does not vary in $k$. By Lemma~\ref{lemma:con1} and Lemma~\ref{lemma:approx}, we learn that $\min_{k\in \Lambda_{1,i}(s,a)}v^k(s,a)\geq \frac{1}{9}\max_{k\in \Lambda_{1,i}(s,a)}v^k(s,a)$. As a result, by \eqref{eq:445}, we have that
\begin{align}
      \mathrm{Pr}\left[Z^k \geq 
     \frac{1}{96}(1-  W^{\pi_2^k}_{d_2}(\textbf{1}_z,\bar{P}^k,\textbf{1}_{s}))\cdot \max_{k'\in \Lambda_{1,i}(s,a)} v^{k'}(s,a)\right]  \geq \frac{1}{2}.\label{eq:446}
\end{align}
 Let $k_{\mathrm{max}}=\max_{k'\in \Lambda_{2,i}(s,a)}k' $.
By Lemma~\ref{lemma:ratiocon},  with probability $1-\delta$ it holds that
\begin{align}
   & \sum_{k\in \Lambda_{2,i}(s,a),k<k_{\mathrm{max}}} \mathbb{I}\left[Z^k \geq 
     \frac{1}{96}(1-  W^{\pi_2^k}_{d_2}(\textbf{1}_z,\bar{P}^k,\textbf{1}_{s}))\cdot \max_{k'\in \Lambda_{1,i}(s,a)} v^{k'}(s,a)\right] 
     \geq & \frac{|\Lambda_{2,i}(s,a)-1|}{4}-\log(1/\delta),\nonumber
\end{align}
where it follows that
\begin{align}
    & \sum_{k\in \Lambda_{2,i}(s,a),k<k_{\mathrm{max}}}Z^k 
    \geq& \left(\frac{|\Lambda_{2,i}(s,a)-1|}{392} - \frac{1}{96}\log(1/\delta)- \frac{1}{96}\sum_{k\in \Lambda_{2,i}(s,a)}W_{d_2}^{\pi_2^k}(\textbf{1}_{z},\bar{P}_k,\textbf{1}_{s}) \right)\max_{k\in \Lambda_{1,i}(s,a)}v^k(s,a) .\label{eq:121}
\end{align}

 Let $E_4^k$ be the event $E_{2}^k$ occurs, and then the agent reaches $\mathcal{K}^C$ in the following $d$ steps under $\pi_2^k$.
Note that $W_{d_2}^{\pi_2^k}(\textbf{1}_z,\bar{P}_k,\textbf{1}_{s})=\mathbb{E}[E_4^k]$.
By Lemma~\ref{lemma:ratiocon}, with probability $1-\delta$ it holds that 
\begin{align}
    \sum_{k\in \Lambda_{2,i}(s,a),k<k_{\mathrm{max}}}W_{d_2}^{\pi_2^k}(\textbf{1}_{z},\bar{P}_k,\textbf{1}_{s})\leq 2\sum_{k\in \Lambda_{2,i}(s,a),k<k_{\mathrm{max}}}\mathbb{I}[E_4^k]+4\log(1/\delta).\label{eq:122}
    \end{align}
 By \eqref{eq:120}, \eqref{eq:121} and \eqref{eq:122}, with probability $1-3\delta$, it holds that
 \begin{align}
    & \sum_{k\in \Lambda_{2,i}(s,a),k<k_{\mathrm{max}}}Z^k \nonumber \\ &  \geq \frac{1}{96}\left(\frac{|\Lambda_{2,i}(s,a)-1|}{4} -5\log(1/\delta)-2\sum_{k\in \Lambda_{2,i}(s,a)}\mathbb{I}[E_4^k]\right)\max_{k\in \Lambda_{1,i}(s,a)}v^k(s,a) \nonumber
     \\ & \geq \frac{1}{96}\left(\frac{1}{4}\left(\frac{1}{14400S}|\Lambda_{1,i}(s,a)|-\log(1/\delta)-\sum_{k\in \Lambda_{1,i}(s,a)}\mathbb{I}[E_3^k]\right) -6\log(1/\delta)-2\sum_{k\in \Lambda_{2,i}(s,a)}\mathbb{I}[E_4^k]\right) \nonumber \\
     & ~~~~~~\cdot \max_{k\in \Lambda_{1,i}(s,a)}v^k(s,a)\nonumber
     \\ & \geq \frac{1}{96}\left(\frac{1}{57600S}|\Lambda_{1,i}(s,a)|-6\log(1/\delta)-2\sum_{k\in \Lambda_{1,i}(s,a)}\mathbb{I}[E_3^k \cup E_4^k]\right)\max_{k\in \Lambda_{1,i}(s,a)}v^k(s,a).\nonumber
 \end{align}

  Also note that by definition,  $$\sum_{k'\in \Lambda_{2,i}(s,a),k'<k_{\mathrm{max}}}Z^{k'} \leq N^{k_{\mathrm{max}}}(s,a) \leq 810SAN_0u^k(s)v^k(s,a)\leq 810SAN_0 \max_{k\in \Lambda_{1,i}(s,a)}v^k(s,a),$$ we have that
  \begin{align}
      \left(\frac{1}{57600S}|\Lambda_{1,i}(s,a)|-6\log(1/\delta)-2\sum_{k\in \Lambda_{1,i}(s,a)}\mathbb{I}[E_3^k \cup E_4^k]\right)\leq 810000SAN_0.\label{eq:124}
  \end{align}

Taking sum over $i$, and noting that $\sum_{s,a}\sum_{i,k\in \Lambda_{1,i}(s,a)}\mathbb{I}[E_3^k \cup E_4^k]\leq S^2AN_0$, we learn that
\begin{align}
    \sum_{s,a}\sum_{1\leq i \leq |\mathcal{I}|} |\Lambda_{1,i}(s,a)|\leq   57600S(810000S^4A^3N_0 +6S^3A^2\log(1/\delta)+2S^2AN_0)
\end{align}
The proof is completed.
\end{proof}

\subsubsection{Putting All Together}
Combining Lemma~\ref{lemma:sr} and \ref{lemma:v3}, the proof is completed.

\subsection{Other Missing Proofs}

\textbf{Lemma~\ref{lemma:approx} (restated)}
\emph{
Let the initial distribution $\mu_1$ be fixed.
	For two transition model $P'$ and $P''$ such that $P'$ is $\epsilon$-closed to $P''$, it holds that
	\begin{align}
	   e^{-4S\epsilon}    W^{\pi}_{d}(r,P',\mu_1) \leq W^{\pi}_{d}(r,P'',\mu_1)\leq  e^{4S\epsilon}W^{\pi}_{d}(r,P'',\mu_1)
	\end{align}
	for any stationary policy $\pi$, horizon $d\geq 1$ and any non-negative reward $r$. 
	}

\begin{proof} Let the policy $\pi$ be fixed.
First, we assume $P'$ only differs with $P''$ at $(s^*,a^*)$, where $a^*=\pi(s^*)$\footnote{Here we deal with a deterministic policy $\pi$. The proof also works for non-deterministic policies.}. Moreover, we assume that there are only two possible next states of $(s^*,a^*)$, which are denoted as $s_{l}$ and $s_{r}$. 

 Let $p'_{l}= P'_{s^*,a^*,s_{l}}$, $p'_{r}=1-p'_{l}$, $p''_{l}=P''_{s^*,a^*,s_{l}}$ and $p''_{r}=1-p''_{l}$. We assume that the agent starts at $(s^*,a^*)$, since the transitions of the two models are exactly the same before visiting $(s^*,a^*)$. For each $(s,a,s')$, we define 
\begin{align}
 & 	w'_{d}(s,a,s') =\mathbb{E}_{\pi,P'}\left[\sum_{i=1}^{d}\mathbb{I}[(s_i,a_i,s_{i+1})=(s,a,s')] \mid (s_1,a_1)=(s^*,a^*)\right];\nonumber
 \\  & w''_{d}(s,a,s')=\mathbb{E}_{\pi,P''}\left[\sum_{i=1}^{d}\mathbb{I}[(s_i,a_i,s_{i+1})=(s,a,s')] \mid (s_1,a_1)=(s^*,a^*)\right].\nonumber
\end{align}

 For fixed $h$, we define $\kappa'_{h}(s,a,s'):= \min_{1\leq d \leq h} \frac{w'_{d}(s,a,s')}{w''_{d}(s,a,s')}$ .
 
 Assuming $p'_{r}\geq p''_{r}$, by policy difference lemma (Lemma~\ref{lemma:pd}) we then have  that for any $d$, $w'_{d}(s^*,a^*,s_{r})\geq w''_{d}(s^*,a^*,s_{r})$ and $w'_{d}(s^*,a^*,s_{l})\leq w''_{d}(s^*,a^*,s_{l})$.
 
 By definition it follows that
 \begin{align}
 & \kappa'_{h}(s^*,a^*,s_{r})\geq 1\nonumber
\\ & \kappa'_{h}(s^*,a^*,s_{l})\leq 1\nonumber
  \end{align}

Let $(s,a,s')$ be fixed.
 Let $x_{l,d_1,d_2} (x_{r,d_1,d_2})$ be the probability of visiting $(s,a,s')$ at the $d_2$-th step starting  from $s_{l} (s_{r})$ at the $d_1$-th step without visiting $(s^*,a^*)$ between the $d_1$-th and $d_2$-th step.  By definition, $x_{l,d_1,d_2}$ only depends on $(d_2-d_1)$ and we can rewrite $x_{l,d_1,d_2}$ as $x_{l}(d_2-d_1)$. Similarly we define $x_{r}(d_2-d_1)=x_{r,d_1,d_2}$. Note that $x_{l}(d')$($x_{r}(d')$) do not depends on $p'_{r}$. 
 Since the initial state-action pair is $(s^*,a^*)$, for any $h'\in [h]$  we have that
 \begin{align}
     w'_{h'}(s,a,s') &= \sum_{d_1=1}^{h'}\sum_{d_2=d_1}^{h'} (\mathbb{P}_{P'}[s_{d_1}=s_{l}] x_{l,d_1,d_2}+\mathbb{P}_{P'}[s_{d_1}=s_{r}]x_{r,d_1,d_2} ) \nonumber 
     \\ &  = \sum_{d_1=1}^{h'}\left(\mathbb{P}_{P'}[s_{d_1}=s_{l}]\sum_{d_2=d_1}^{h'} x_{l,d_1,d_2}+\mathbb{P}_{P'}[s_{d_1}=s_{r}]\sum_{d_2=d_1}^{h'}x_{r,d_1,d_2} \right) \nonumber 
     \\ & =\sum_{d_1=1}^{h'}\left(\mathbb{P}_{P'}[s_{d_1}=s_{l}]\sum_{d_2=1}^{h'-d_1+1} x_{l}(d_2)+\mathbb{P}_{P'}[s_{d_1}=s_{r}]\sum_{d_2=1}^{h'-d_1+1}x_{r}(d_2) \right)\nonumber 
     \\ & =\sum_{d_2=1}^{h'} x_{l,1,d_2} \left(\sum_{d_1=1}^{h'-d_2+1} \mathbb{P}_{P'}[s_{d_1}=s_{l}]\right) + \sum_{d_2=1}^{h'} x_{r}(d_2) \left(\sum_{d_1=1}^{h'-d_2+1} \mathbb{P}_{P'}[s_{d_1}=s_{r}]\right)\nonumber 
     \\ & =\sum_{d_2=1}^{h'}x_{l}(d_2)w'_{h'-d_2+1}(s^*,a^*,s_{l}) +\sum_{d_2=1}^{h'}x_{r}(d_2)w'_{h'-d_2+1}(s^*,a^*,s_{r}) \nonumber 
     \\ & \geq  \kappa_{h}(s^*,a^*,s_{l})\sum_{d_2=1}^{h'}x_{l}(d_2)w''_{h'-d_2+1}(s^*,a^*,s_{l}) +\sum_{d_2=1}^{h'}x_{r}(d_2)w''_{h'-d_2+1}(s^*,a^*,s_{r}) \nonumber 
     \\ & =\kappa_{h'}(s^*,a^*,s_{l}) w''_{h'}(s,a,s').\nonumber
 \end{align}
By definition, we then have that
\begin{align}
      \kappa'_h(s^*,a^*,s_{l}) \leq      \kappa'_h(s,a,s')\label{eq:key1}
\end{align}
 
 Note that for any $d\geq 1$,
 \begin{align}
&  w'_{d}(s^*,a^*,s_{r})/p'_{r}= w'_{d}(s^*,a^*,s_{l})/p'_{l};\nonumber
\\ & w''_{d}(s^*,a^*,s_{r})/p''_{r}= w''_{d}(s^*,a^*,s_{l})/p''_{l}.\nonumber
 \end{align}
 We then have that
 \begin{align}
 1\leq \kappa'_{h}(s^*,a^*,s_{r})=\min_{1\leq d \leq h }\frac{w'_{d}(s^*,a^*,s_{r})}{w''_{d}(s^*,a^*,s_{r})} =  \frac{p'_rp''_{l}}{p'_lp''_{r}}\min_{1\leq d\leq h} \frac{w'_{d}(s^*,a^*,s_{l})}{w''_{d}(s^*,a^*,s_{l})} \leq e^{2\epsilon} \kappa'_{h}(s^*,a^*,s_{l})\leq e^{2\epsilon}.\label{eq:key2}
 \end{align}
By \eqref{eq:key1} and \eqref{eq:key2}, we have that $\kappa_{h}'(s,a,s')\geq \kappa_{h}'(s^*,a^*,s_{r})\geq  e^{-2\epsilon}$ for any $h\geq 1$ and any $(s,a,s')$, which implies that
\begin{align}
w'_{d}(s,a,s')\geq e^{-2\epsilon}w''_{d}(s,a,s')\nonumber
\end{align}
for any $d\geq 1$ and any $(s,a,s')$. By reversing $s_{l}$ and $s_{r}$, we can obtain that
\begin{align}
w''_{d}(s,a,s')\geq e^{-2\epsilon}w'_{d}(s,a,s')\nonumber
\end{align}
for any $d\geq 1$ and any $(s,a,s')$.  The proof is completed by noting that any reward $r$ is a positive linear combination of $\{\textbf{1}_{s,a,s'}\}_{(s,a,s')\in \mathcal{S}\times\mathcal{A}\times\mathcal{S}}$ .

As for the general case, we also the case where $P'$ only differs with $P''$ at $(s^*,a^*)$. Let $p' = P'_{s^*,a^*}$ and $p''=P''_{s^*,a^*}$. We claim there exists $p_0 = p',p_1,p_2,\ldots,p_{S}=p''$ satisfying that: $\exists \mathcal{S}^i_1,\mathcal{S}^i_2,\mathcal{S}^i_3$ be a partition of $\mathcal{S}$ such that
\begin{align}
    p_{i,s}=p_{i+1,s}, \forall s\in \mathcal{S}^i_1, \quad p_{i,s} = e^{\epsilon'_i}p_{i+1,s}, \forall s\in \mathcal{S}^i_2,\quad p_{i,s}=e^{-\epsilon''_i}p_{i+1,s}, \forall s\in \mathcal{S}^i_{3}
\end{align}
for $1\leq i \leq S-1$ with $\epsilon'_i,\epsilon''_i\ge 0$ and $\sum_{i=1}^{S-1}\max\{\epsilon'_i,\epsilon''_i\}\leq 2\epsilon$. If this claim holds, then the conclusion holds by iteratively using the proof for the case where there are only two possible next states.  Now we construct  $\{p_i\}$. Give $p_1=p'$, we define $\mathcal{S}'=\{s: p'_{s}>p''_{s}\}$. For $\lambda\in [0,1]$, define $p_{1,\lambda}$ by setting $p_{1,\lambda,s}=\lambda p_{1,s}$ for $s\in \mathcal{S}'$ and $p_{1,\lambda,s}=\lambda'p_{1,s}$ for $s\in \mathcal{S}'$, where $
\lambda'$ is the unique real such that $\sum_{s'}p_{1,\lambda,s'}=1$. Let $\lambda_1$ be the largest real in $[0,1]$ such that  $\exists s',p_{1,\lambda_1,s'}=p''_{s'}$. We then choose $p_2= p_{1,\lambda_1}$. For $i\geq 2$, we define $p_{i,\lambda}$ by setting $p_{i,\lambda,s}=p_{i,s},\forall s $ such that $p_{i,s}=p''_{s}$, $p_{i,\lambda,s}=\lambda p_{i,s},\forall s$ such that $p_{i,s}>p''_{s}$ and $p_{i,\lambda,s}=\lambda'p_{i,s}, \forall s$ such that $p_{i,s}<p''_{s}$, where $\lambda'$ is  the unique real such that $\sum_{s'}p_{i,\lambda,s'}=1$. Let $\lambda_i$ be the largest real such that $\exists, s'$ such that $p_{i,s'}\neq p''_{s'}$ and $p_{i,\lambda,s'}=p''_{s'}$ and $\lambda'_{i}$ be the corresponding $\lambda'$. Then we set $p_{i+1}=p_{i,\lambda_i}$. It is easy to note that for $s'=\argmax_{s}\frac{p'_{s}}{p''_{s}}$, $p_{i,s'}$ is increasing in $i$, therefore, we have that $\Pi_{i=1}^{S-1}\frac{1}{\lambda_i}\leq e^{\epsilon}$.   In a similar way, $\Pi_{i=1}^{S-1}\lambda'_i\leq e^{\epsilon}$. The proof is completed.
\end{proof}

\begin{lemma}\label{lemma:v1}
For any stationary policy $\pi$ and any transition model $p$ and any $(s,a)\in \mathcal{S}\times \mathcal{A}$ such that $\pi(s)=a$ and $p_{z,a}=\textbf{1}_{z}, \forall a$, it holds that
\begin{align}
    (1- W_d^{\pi}(\textbf{1}_{z},p,\textbf{1}_{s}))W_d^{\pi}(\textbf{1}_{s,a},\mathrm{cut}(p),\textbf{1}_{s}) \leq W_d^{\pi}(\textbf{1}_{s,a},p,\textbf{1}_{s})\leq  W_d^{\pi}(\textbf{1}_{s,a},\mathrm{cut}(p),\textbf{1}_{s}).\nonumber
\end{align}
\end{lemma}
\begin{proof}
Let $p' = \mathrm{cut}(p)$.
    By policy difference lemma (Lemma~\ref{lemma:pd}), we have that
    \begin{align}
 &         W^{\pi}_{d}( \textbf{1}_{s,a},p,\textbf{1}_{s}) - W^{\pi}_{d}(\textbf{1}_{s,a},p',\textbf{1}_{s}) \nonumber
 \\ & = \mathbb{E}_{p,\pi}\left[\sum_{h=1}^{d-1} \sum_{s'}(p_{s_h,a_h,s'}-p'_{s_h,a_h,s'}) W_{d-h}^{\pi}(\textbf{1}_{s,a},p',\textbf{1}_{s'}) \right] \nonumber
 \\ & \geq  -\mathbb{E}_{p,\pi}\left[\sum_{h=1}^{d-1} p_{s_h,a_h,z} \mathbb{I}[s_h\neq z]\right]W_{d}^{\pi}(\textbf{1}_{s,a},p',\textbf{1}_{s})\label{eq:lemma1_0}
 \\ & =  -W_d^{\pi}(\textbf{1}_{z},p,\textbf{1}_{s})\cdot W_{d}^{\pi}(\textbf{1}_{s,a},p',\textbf{1}_{s}).\label{eq:lemma1_1}
    \end{align}
    Here \eqref{eq:lemma1_0} uses the fact that $$W_{d-h}^{\pi}(\textbf{1}_{s,a},p,\textbf{1}_{s'})\leq  W_{d}^{\pi}(\textbf{1}_{s,a},p,\textbf{1}_{s'})\leq W_{d}^{\pi}(\textbf{1}_{s,a},p,\textbf{1}_{s}).$$
    The left part is proven by rearranging \eqref{eq:lemma1_1}. As for the right side, it suffices to note that for any $(s,a)$ such that $p_{s,a,z}<1$,
    \begin{align}
        (p_{s_h,a_h}-p'_{s_h,a_h}) W_{d-h}^{\pi}(\textbf{1}_{s,a},p',\textbf{1}_{s_h,a_h}) = -p_{s_h,a_h,z}W_{d-h}^{\pi}(\textbf{1}_{s,a},p',\textbf{1}_{s_h,a_h})\leq 0 ,\nonumber
    \end{align}
    and
    for $(s,a)$ such that $p_{s,a,z}=1$, $
        (p_{s_h,a_h}-p'_{s_h,a_h}) W_{d-h}^{\pi}(\textbf{1}_{s,a},p',\textbf{1}_{s_h,a_h}) =0$.
    The proof is completed. 
\end{proof}
\begin{lemma}\label{lemma:v2} Let the initial distribution $\mu_1$ and $s\in \mathcal{S}$ be fixed. Let $\pi$ be a policy (which is possibly non-stationary). Suppose $p_{z,a}=\textbf{1}_{z}$ for any $a$.
  Then we have that
    \begin{align}
        X_{d}^{\pi}(\{s\},p,\mu_1) \geq X_{d}^{\pi}(\{s\},\mathrm{cut}(p),\mu_1) - W_{d}^{\pi}(\textbf{1}_{z},p,\mu_1).\nonumber
    \end{align}
\end{lemma}
\begin{proof}[Proof of Lemma~\ref{lemma:v2}]
Let $\bar{p}$ be defined as $\bar{p}_{s',a'}=p_{s',a'}$ for $s'\neq s$ and any $a$, and $\bar{p}_{s,a} = \textbf{1}_{z_1}$ for any $a$, where $p_{z_1,a}=\textbf{1}_{z_2}$ and $p_{z_2,a}=\textbf{1}_{z_2}$ for any $a$. We also define $\bar{p}'$ by $\bar{p}' = p'_{s,a}$ for $s'\neq s$ and any $a$, and $\bar{p}'_{s,a}=\textbf{1}_{z_1}, \bar{p}'_{z_1,a}=\textbf{1}_{z_1}$ and $\bar{p}'_{z_2,a}=\textbf{1}_{z_2}$ for any $a$. Then we have that
    \begin{align}
      &    X_{d}^{\pi}(\{s\},p,\mu_1) =W^{\pi}_{d}(\textbf{1}_{z_1},\bar{p},\mu_1) ;\nonumber 
      \\ & X_{d}^{\pi}(\{s\},p',\mu_1) = W^{\pi}_{d}(\textbf{1}_{z_1},\bar{p}',\mu_1).
    \end{align}
    Using policy difference lemma policy difference lemma (Lemma~\ref{lemma:pd}), and noting that  $0\leq W_{d'}^{\pi}(\textbf{1}_{z_1},p',\mu_1)\leq 1$ for any $0\leq d'\leq d+1$, we obtain that
    \begin{align}
        W^{\pi}_{d}(\textbf{1}_{z_1},\bar{p},\mu_1) & \geq W^{\pi}_{d}(\textbf{1}_{z_1},\bar{p}',\mu_1) - \mathbb{E}_{\bar{p},\pi}\left[\sum_{h=1}^{d-1}p_{s_h,a_h,z}\mathbb{I}[s_h\neq z]\right]\nonumber 
        \\ & \geq W^{\pi}_{d}(\textbf{1}_{z_1},\bar{p}',\mu_1) -W_{d}^{\pi}(\textbf{1}_{z},\bar{p},\mu_1)\nonumber 
        \\ & \geq W^{\pi}_{d}(\textbf{1}_{z_1},\bar{p}',\mu_1) -W_{d}^{\pi}(\textbf{1}_{z},p,\mu_1),
    \end{align}
    where the last line is by the fact that
    \begin{align}
        W_{d}^{\pi}(\textbf{1}_z,p,\mu_1) - W^{\pi}_d(\textbf{1}_z,\bar{p},\mu_1)= \mathbb{E}_{p,\pi}\left[ \sum_{h=1}^{d-1}(p_{s_h,a_h}-\bar{p}_{s_h,a_h})W_{d}^{\pi}(\textbf{1}_z,\bar{p},\textbf{1}_{s_h,a_h})\mathbb{I}[s_h\neq z]\right]\geq 0.
    \end{align}
\end{proof}

\begin{lemma}\label{lemma:pd}
Let $p$ and $p'$ be two different transition model. Let the reward $r$, policy $\pi$, horizon $d$ and initial distribution $\mu_1$ be fixed. It then holds that
\begin{align}
    &W_{d}^{\pi}(r,p,\mu_1)- W_d^{\pi}(r,p',\mu_1)\nonumber 
    \\ & = \sum_{h=1}^d\sum_{s,a}\mathbb{P}_{\pi,p}[(s_h,a_h)=(s,a)|s_1\sim \mu_1] \sum_{s'} (p_{s,a,s'}-p'_{s,a,s'})W_{d-h}^{\pi}(r,p',\textbf{1}_{s'}).
\end{align}
\end{lemma}
\begin{proof} 
Let $\tilde{\mu}_i$ be the distribution of $s_{i+1}$ under $p$ for $i\geq 0$. Define $w_i = W_{i}^{\pi}(r,p,\mu_1)+ W_{d-i}^{\pi}(r,p',\tilde{\mu}_i) $. Then we have that $w_d = W_{d}^{\pi}(r,p,\mu_1)$ and $w_{0}=W_{d}^{\pi}(r,p',\mu_1)$.
Then the proof is completed by noting that
\begin{align}
&w_{h+1}-w_{h}\nonumber 
\\ & =   W_{h+1}^{\pi}(r,p,\mu_1)+ W_{d-h+1}^{\pi}(r,p',\tilde{\mu}_{h+1})  - W_{h}^{\pi}(r,p,\mu_1)+ W_{d-h}^{\pi}(r,p',\tilde{\mu}_h)  \nonumber \\
 & = \sum_{s,a}\tilde{\mu}_{h}(s)\pi_{h+1}(a|s) r_{h+1}(s,a) + W_{d-h-1}^{\pi}(r,p',\tilde{\mu}_{h+1}) -W_{d-h}^{\pi}(r,p',\tilde{\mu}_h) \nonumber 
 \\ & =  \sum_{s,a}\tilde{\mu}_{h}(s)\pi_{h+1}(a|s) r_{h+1}(s,a) + W_{d-h-1}^{\pi}(r,p',\tilde{\mu}_{h+1}) \nonumber 
 \\ & \quad \quad \quad \quad \quad \quad \quad \quad \quad \quad\quad \quad \quad - \sum_{s,a}\tilde{\mu}_h(s)\pi_{h+1}(a|s)(r_{h+1}(s,a) + \sum_{s'}p'_{s,a,s'}W_{d-h-1}^{\pi}(r,p',\textbf{1}_{s
 }))\nonumber 
 \\ & =\sum_{s,a}\tilde{\mu}_{h}(s)\pi_{h+1}(a|s) \sum_{s'}p_{s,a,s'}W_{d-h-1}^{\pi}(r,p',\textbf{1}_{s
 }) -\sum_{s,a}\tilde{\mu}_{h}(s)\pi_{h+1}(a|s) \sum_{s'}p'_{s,a,s'}W_{d-h-1}^{\pi}(r,p',\textbf{1}_{s
 }) .\nonumber
\end{align}
 Here the second last inequality is by the fact that $\tilde{\mu}_{h+1}(s')=\sum_{s,a}\tilde{\mu}_{h}(s)\pi_{h+1}(a|s)p_{s,a,s'}$
\end{proof}

\end{document}